\documentclass[twoside,11pt]{article}

\usepackage{amsmath}
\usepackage{amsfonts}
\usepackage{nicefrac}
\usepackage{mathtools}
\usepackage{mathrsfs}
\usepackage{bm}
\usepackage{enumerate}
\usepackage{xcolor}
\usepackage{rotating}
\usepackage{hyperref}
\usepackage[english]{babel}
\usepackage{bbm}
\usepackage{xspace}
\usepackage{booktabs}
\usepackage{multirow}
\usepackage{caption}
\usepackage{subcaption}
\usepackage{dsfont}
\usepackage{float}
\usepackage{makecell}
\usepackage{tablefootnote}
\usepackage{soul}
\usepackage{hyperref}
\usepackage{graphicx}
\usepackage{amssymb}
\usepackage{adjustbox}
\usepackage[abbrvbib, preprint, nohyperref]{jmlr2e}

\allowdisplaybreaks

\newcommand{\noisetar}{\epsilon^{\tagtar}}
\newcommand{\noisesrc}{\epsilon^{\tagm{m}}}

\newcommand{\vm}[2]{{v^{\tagm{#1}}_{#2}}}
\newcommand{\vtar}[1]{{v^{\tagtar}_{#1}}}
\newcommand{\Pcip}{{P_y}}
\newcommand{\probsrc}[1]{{p^{\tagm{m}}_{#1}}}

\newcommand{\probtar}[1]{{p^{\tagtar}_{#1}}}

\newcommand{\dataset}{\mathcal{D}}
\newcommand{\distri}{\mathcal{P}}

\newcommand{\distriP}{\mathcal{P}}
\newcommand{\distriPhat}{\widehat{\distriP}}

\newcommand{\discrepancy}[2]{\mathfrak{D}\left({#1},{#2}\right)}
\newcommand {\probdiv}[2]{\mathfrak{D}\left({#1},{#2}\right)}

\newcommand{\horacle}{h_\textnormal{oracle}}

\newcommand{\tagm}[1]{{(#1)}}
\newcommand{\tagtar}{{({\mathfrak{T}})}}

\newcommand{\ntar}{n^\tagtar}
\newcommand{\nsrc}{n^\tagm{m}}
\newcommand{\nsrcm}[1]{n^\tagm{#1}}
\newcommand{\nsrcone}{n^{(1)}}

\newcommand{\nttar}{\ntar}
\newcommand{\ntsrc}{\nsrc}

\newcommand{\Dtar}{\dataset^\tagtar}
\newcommand{\Dsrc}{\dataset^\tagm{m}}

\newcommand{\Ptar}{\distri^{\tagtar}}

\newcommand{\Psrc}{\distri^{(m)}}

\newcommand{\Psrcone}{\distri^{(1)}}

\newcommand{\ones}[1]{\mathbf{1}_{{#1}}}
\newcommand{\real}{\mathbb{R}}

\makeatletter
\newcommand*{\rom}[1]{\expandafter\@slowromancap\romannumeral #1@}
\makeatother

\newtheorem{assumption}{Assumption}

\newtheorem{subtheorem}{Theorem}

\newcommand{\assumpref}[1]{Assumption~\ref{assump:#1}}

\newcommand{\figref}[1]{Figure~\ref{fig:#1}}

\newcommand{\secref}[1]{Section~\ref{sec:#1}}
\newcommand{\appref}[1]{Appendix~\ref{app:#1}}

\newcommand{\defref}[1]{Definition~\ref{def:#1}}

\newcommand{\lemref}[1]{Lemma~\ref{lem:#1}}

\newcommand{\propref}[1]{Proposition~\ref{prop:#1}}

\newcommand{\thmref}[1]{Theorem~\ref{thm:#1}}

\newcommand{\corref}[1]{Corollary~\ref{cor:#1}}
\newcommand{\tabref}[1]{Table~\ref{tab:#1}}

\newcommand{\eqnref}[1]{\eqref{eqn:#1}}

\DeclareMathOperator{\rank}{rank}

\DeclareMathOperator{\trace}{trace}

\DeclareMathOperator{\var}{Var}

\DeclareMathOperator{\spn}{span}

\DeclareMathOperator*{\argmin}{arg\,min}

\newcommand{\defn}{\coloneqq}

\newcommand{\iid}[0]{i.i.d.\xspace}
\newcommand{\one}[1]{{\mathbbm{1}}_{{#1}}}

\newcommand{\norm}[1]{\lVert{#1}\rVert}

\newcommand{\EE}[1]{\mathbb{E}\left[{#1}\right]} 
\newcommand{\Ep}[2]{\mathbb{E}_{#1}\left[{#2}\right]}
\newcommand{\EEst}[2]{\mathbb{E}\left[{#1}\ \middle| \ {#2}\right]} 

\newcommand{\VVst}[2]{\var\left[{#1}\ \middle|\ {#2}\right]}
\renewcommand{\O}[1]{\mathcal{O}\left({#1}\right)}
\def\R{\mathbb{R}}

\newcommand{\ident}{\mathbf{I}}
\newcommand{\zeros}{\mathbf{0}}

\newcommand{\iidsim}{\stackrel{\mathrm{i.i.d.}}{\sim}}

\newcommand{\ignore}[1]{}

\newcommand{\eps}{\epsilon}

\newcommand{\betastar}{\beta^\star}

\newcommand{\mutar}[1]{\mu^\tagtar_{#1}}
\newcommand{\erf}[1]{{\rm erf}\left(#1\right)}

\newcommand{\phicip}{\phi_{\textnormal{CIP}}}
\newcommand{\phiinv}{\phi_{\textnormal{inv}}}
\newcommand{\hinv}{h_{\textnormal{inv}}}
\newcommand{\ginv}{g_{\textnormal{inv}}}

\newcommand{\gcip}{g_{\textnormal{CIP}}}
\newcommand{\hcip}{h_{\textnormal{CIP}}}
\newcommand{\phihcip}{\widehat{\phi}_{\textnormal{CIP}}}
\newcommand{\phiiwcip}{\phi_{\textnormal{IW-CIP}}}
\newcommand{\phihiwcip}{\widehat{\phi}_{\textnormal{IW-CIP}}}
\newcommand{\ghcip}{\widehat{g}_{\textnormal{CIP}}}
\newcommand{\hhcip}{\widehat{h}_{\textnormal{CIP}}}
\newcommand{\giwcip}{g_{\textnormal{IW-CIP}}}
\newcommand{\hiwcip}{h_{\textnormal{IW-CIP}}}
\newcommand{\ghiwcip}{\widehat{g}_{\textnormal{IW-CIP}}}
\newcommand{\hhiwcip}{\widehat{h}_{\textnormal{IW-CIP}}}

\newcommand{\gstar}{g^\star}
\newcommand{\hstar}{h^\star}
\newcommand{\phistar}{\phi^\star}

\newcommand{\penaltyiwcip}[1]{\widehat{\Lambda}_{#1}}
\newcommand{\Pssi}[1]{\Psi_{\Gset,#1}}

\newcommand{\hdip}{h_{\textnormal{DIP}}}
\newcommand{\gdip}{g_{\textnormal{DIP}}}
\newcommand{\phidip}{\phi_{\textnormal{DIP}}}
\newcommand{\hhdip}{\widehat{h}_{\textnormal{DIP}}}
\newcommand{\ghdip}{\widehat{g}_{\textnormal{DIP}}}
\newcommand{\phihdip}{\widehat{\phi}_{\textnormal{DIP}}}

\newcommand{\hjdip}{h_{\textnormal{j-DIP}}}
\newcommand{\gjdip}{g_{\textnormal{j-DIP}}}
\newcommand{\phijdip}{\phi_{\textnormal{j-DIP}}}
\newcommand{\hhjdip}{\widehat{h}_{\textnormal{j-DIP}}}
\newcommand{\ghjdip}{\widehat{g}_{\textnormal{j-DIP}}}
\newcommand{\phihjdip}{\widehat{\phi}_{\textnormal{j-DIP}}}

\newcommand{\fsrc}{f^{(m)}}
\newcommand{\fsrcm}[1]{f^{(#1)}}

\newcommand{\fsrcone}{f^{(1)}}

\newcommand{\ftar}{f^{\tagtar}}
\newcommand{\fc}{f_{\textnormal {inv}}}

\newcommand{\epssrc}{\epsilon^{(m)}}
\newcommand{\epssrcone}{\epsilon^{(1)}}
\newcommand{\epssrctwo}{\epsilon^{(2)}}
\newcommand{\epstar}{\epsilon^{\tagtar}}
\newcommand{\Peps}{\mathcal{P}_{\epsilon}}

\newcommand{\kk}{\zeta}

\newcommand{\lamcip}{\lambda_{\textnormal{CIP}}}
\newcommand{\lamiwcip}{\lambda_{\textnormal{IW-CIP}}}
\newcommand{\lamdip}{\lambda_{\textnormal{DIP}}}
\newcommand{\lamjdip}{\lambda_{\textnormal{j-DIP}}}
\newcommand{\lammax}{\lambda_{\rm max}}
\newcommand{\lammin}{\lambda_{\rm min}}

\newcommand{\Gdiv}[2]{D_{\mathcal{G}}\left({#1},{#2}\right)}
\newcommand{\Gdivsq}[2]{D_{\mathcal{G}}^2\left({#1},{#2}\right)}

\newcommand{\tvdistsq}[2]{{\rm TV}^2\left({#1},{#2}\right)}
\newcommand{\KLdiv}[2]{{\rm KL}\left({#1},{#2}\right)}

\newcommand{\prob}[1]{\mathbb{P}\left\{{#1}\right\}}
\newcommand{\probc}[2]{\mathbb{P}\left\{{#1}\ \middle| \ {#2}\right\}}
\newcommand{\probsubset}[2]{\mathbb{P}_{{#1}}\left\{{#2}\right\}}

\newcommand{\risksrc}[1]{\mathcal{R}^{(m)}({#1})}
\newcommand{\risksrcone}[1]{\mathcal{R}^{(1)}({#1})}
\newcommand{\risksrcw}[2]{\mathcal{R}^{(m)}({#1};{#2})}

\newcommand{\risktar}[1]{\mathcal{R}^{\tagtar}({#1})}
\newcommand{\riskhsrc}[1]{\widehat{\mathcal{R}}^{(m)}({#1})}
\newcommand{\riskhsrcone}[1]{\widehat{\mathcal{R}}^{(1)}({#1})}
\newcommand{\riskhsrcw}[2]{\widehat{\mathcal{R}}^{(m)}({#1};{#2})}
\newcommand{\riskhsrconew}[2]{\widehat{\mathcal{R}}^{(1)}({#1};{#2})}
\newcommand{\riskhtar}[1]{\widehat{\mathcal{R}}^{\tagtar}({#1})}

\newcommand{\riskbar}[2]{\overline{\mathcal{R}}({#1};{#2})}
\newcommand{\riskhbar}[2]{\widehat{\overline{\mathcal{R}}}({#1};{#2})}

\newcommand{\risktarsubset}[2]{\mathcal{R}_{{#1}}^{\tagtar}({#2})}
\newcommand{\risksrconesubset}[2]{\mathcal{R}_{{#1}}^{(1)}({#2})}

\newcommand{\penaltycipm}[2]{{{\Delta}^{(#1)}_{\phi}(#2)}}
\newcommand{\varcip}{{\Sigma_\phi^{(m)}}(y)}
\newcommand{\meanhat}[2]{\widehat{\mu^{(#1)}_\phi(#2)}}
\newcommand{\Pihatphi}[1]{{\Pi_\phi(#1)}}

\newcommand{\wm}{w^{(m)}}
\newcommand{\wmi}[1]{w^{(#1)}}
\newcommand{\wmy}{\coord{w}{y}^{(m)}}
\newcommand{\wmcoord}[1]{\coord{w}{#1}^{(m)}}
\newcommand{\whmcoord}[1]{\coord{\widehat{w}}{#1}^{(m)}}
\newcommand{\wone}{w^{(1)}}
\newcommand{\wM}{w^{(M)}}
\newcommand{\whm}{\widehat{w}^{(m)}}
\newcommand{\whmi}[1]{\widehat{w}^{(#1)}}
\newcommand{\whone}{\widehat{w}^{(1)}}

\newcommand{\w}{w}
\newcommand{\wh}{\widehat{w}}
\newcommand{\ww}{{\rm w}}
\newcommand{\wwm}{{\rm w}^{(m)}}
\newcommand{\wwmi}[1]{{\rm w}^{(#1)}}

\newcommand{\muh}{\widehat{\mu}}
\newcommand{\condi}{\kappa}

\newcommand{\Xtar}{X^{\tagtar}}
\newcommand{\Ytar}{Y^{\tagtar}}
\newcommand{\Xsrc}{X^{(m)}}
\newcommand{\Ysrc}{Y^{(m)}}
\newcommand{\Xsrcone}{X^{(1)}}
\newcommand{\Ysrcone}{Y^{(1)}}
\newcommand{\Xsrctwo}{X^{(2)}}
\newcommand{\Ysrctwo}{Y^{(2)}}

\newcommand{\Xt}{\widetilde{X}}

\newcommand{\Xsrcm}[1]{X^{(#1)}}
\newcommand{\Ysrcm}[1]{Y^{(#1)}}

\newcommand{\noisesrcm}[1]{\epsilon^{(#1)}}

\newcommand{\csrcone}{c^{(1)}}
\newcommand{\ctar}{c^{\tagtar}}

\newcommand{\Asrc}{A^{(m)}}
\newcommand{\Atar}{A^\tagtar}
\newcommand{\Asrcm}[1]{A^{(#1)}}

\newcommand{\Phsrc}{\widehat{\distri}^{(m)}}
\newcommand{\Phtar}{\widehat{\distri}^{\tagtar}}

\newcommand{\Pmartar}[1]{\distri^{\tagtar}_{{#1}}}
\newcommand{\Pcondtar}[2]{\distri^{\tagtar}_{{#1} \vert {#2}}}
\newcommand{\Pmarsrc}[1]{\distri^{(m)}_{{#1}}}
\newcommand{\Pmarsrcone}[1]{\distri^{(1)}_{{#1}}}
\newcommand{\Pcondsrc}[2]{\distri^{(m)}_{{#1} \vert {#2}}}
\newcommand{\Pcondsrcone}[2]{\distri^{(1)}_{{#1} \vert {#2}}}
\newcommand{\Pcondsrcmp}[2]{\distri^{(m')}_{{#1} \vert {#2}}}

\newcommand{\Phmartar}[1]{\widehat{\distri}^{\tagtar}_{{#1}}}
\newcommand{\Phcondtar}[2]{\widehat{\distri}^{\tagtar}_{{#1} \vert {#2}}}
\newcommand{\Phmarsrc}[1]{\widehat{\distri}^{(m)}_{{#1}}}
\newcommand{\Phmarsrcone}[1]{\widehat{\distri}^{(1)}_{{#1}}}
\newcommand{\Phcondsrc}[2]{\widehat{\distri}^{(m)}_{{#1} \vert {#2}}}

\newcommand{\Phcondsrcmp}[2]{\widehat{\distri}^{(m')}_{{#1} \vert {#2}}}

\newcommand{\coord}[2]{{#1}_{[#2]}}
\newcommand{\coordw}[2]{#1_{[#2]}}

\newcommand{\name}[1]{\textbf{\textup{#1}}}

\newcommand{\Aset}{\mathcal{A}}

\newcommand{\Hset}{\mathcal{H}}
\newcommand{\Hsetm}{\mathcal{H}^{(m)}}

\newcommand{\Gset}{\mathcal{G}}
\newcommand{\Gsety}{\mathcal{G}_y}
\newcommand{\Gsetyp}{\mathcal{G}_{y'}}

\newcommand{\Xset}{\mathcal{X}}
\newcommand{\Yset}{\mathcal{Y}}
\newcommand{\Zset}{\mathcal{Z}}
\newcommand{\Normal}{\mathcal{N}}
\newcommand{\Ind}{\ensuremath{\mathbb{I}}}

\newcommand{\rad}[2]{\mathfrak{R}_{#1}\parenth{{#2}}}

\newcommand{\row}[1]{\textnormal{row}({#1})}

\newcommand{\wt}{\widetilde{w}}
\newcommand{\ut}{\widetilde{u}}

\newcommand{\sampcomp}{\gamma}

\newcommand{\vecnorm}[2]{\left\| #1\right\|_{#2}}

\long\def\comment#1{}
\definecolor{battleshipgrey}{rgb}{0.52, 0.52, 0.51}
\definecolor{darkgray}{rgb}{0.66, 0.66, 0.66}
\definecolor{darkgreen}{rgb}{0.0, 0.2, 0.13}
\definecolor{darkspringgreen}{rgb}{0.09, 0.45, 0.27}
\definecolor{dukeblue}{rgb}{0.0, 0.0, 0.61}
\definecolor{olivedrab7}{rgb}{0.24, 0.2, 0.12}
\definecolor{darkblue}{rgb}{0.0, 0.0, 0.55}
\definecolor{darkscarlet}{rgb}{0.34, 0.01, 0.1}
\definecolor{candyapplered}{rgb}{1.0, 0.03, 0.0}
\definecolor{ao(english)}{rgb}{0.0, 0.5, 0.0}
\definecolor{applegreen}{rgb}{0.55, 0.71, 0.0}

\newcommand{\brackets}[1]{\left[ #1 \right]}
\newcommand{\parenth}[1]{\left( #1 \right)}

\newcommand{\braces}[1]{\left\{ #1 \right \}}
\newcommand{\abss}[1]{\left| #1 \right |}

\newcommand{\pmat}[1]{\begin{pmatrix} #1 \end{pmatrix}}

\usepackage{lastpage}
\jmlrheading{26}{2025}{1-\pageref{LastPage}}{9/23; Revised 4/25}{5/25}{23-1234}{Keru Wu, Yuansi Chen, Wooseok Ha, and Bin Yu}
\ShortHeadings{Prominent roles of CICs in DA}{Wu, Chen, Ha, and Yu}

\firstpageno{1}

\begin{document}
	
\title{Prominent Roles of Conditionally Invariant Components\\ in Domain Adaptation: Theory and Algorithms}

\author{\name Keru~Wu\thanks{Equal contribution}\textsuperscript{ \hspace{0.1cm}} \email keru.wu@duke.edu \\ 
	\addr Department of Statistical Science \\ Duke University \\Durham, North Carolina, USA
	\AND
	\name Yuansi~Chen\footnotemark[1]\textsuperscript{ \hspace{0.1cm}} \email yuansi.chen@duke.edu \\ 
	\addr Department of Statistical Science \\ Duke University \\Durham, North Carolina, USA
	\AND	
	\name Wooseok~Ha\footnotemark[1]\textsuperscript{ \hspace{0.1cm}} \email haywse@kaist.ac.kr \\ 
	\addr Department of Mathematical Sciences \\ Korea Advanced Institute of Science \& Technology \\ Daejeon, Korea
	\AND
	\name Bin~Yu \email binyu@berkeley.edu \\ 
	\addr Department of Statistics and EECS \\ University of California, Berkeley \\ Berkeley, CA, USA
}

\editor{Mladen Kolar}

\maketitle

\begin{abstract}%
Domain adaptation (DA) is a statistical learning problem that arises when the distribution of the source data used to train a model differs from that of the target data used to evaluate the model. While many DA algorithms have demonstrated considerable empirical success, blindly applying these algorithms can often lead to worse performance on new datasets. To address this, it is crucial to clarify the assumptions under which a DA algorithm has good target performance. In this work, we focus on the assumption of the presence of conditionally invariant components (CICs), which are relevant for prediction and remain conditionally invariant across the source and target data. We demonstrate that CICs, which can be estimated through conditional invariant penalty (CIP), play three prominent roles in providing target risk guarantees in DA.  First, we propose a new algorithm based on CICs, importance-weighted conditional invariant penalty (IW-CIP), which has target risk guarantees beyond simple settings such as covariate shift and label shift. Second, we show that CICs help identify large discrepancies between source and target risks of other DA algorithms. Finally, we demonstrate that incorporating CICs into the domain invariant projection (DIP) algorithm can address its failure scenario caused by label-flipping features. We support our new algorithms and theoretical findings via numerical experiments on synthetic data, MNIST, CelebA, Camelyon17, and DomainNet datasets.
\end{abstract}


\section{Introduction}

The classical statistical learning problem assumes that the data used for training and those used for testing are drawn from the same data distribution.
While this assumption is often valid, distribution shifts are prevalent in real-world {data problems}. Distribution shifts happen when the distribution of training (or source) data differs from that of test (or target) data~\citep{koh2021wilds}. For example, when a machine learning model is trained on labeled source data from a few hospitals and then deployed in a new hospital, often there are distributional shifts because the data collection and pre-processing in different hospitals can be different~\citep{veta2016mitosis, komura2018machine, zech2018variable}. The statistical learning problem that tackles distributional shifts with labeled source data and unlabeled target data is called a \textit{domain adaptation} (DA) problem. A solution to DA is desired especially in situations where obtaining labeled data from target domain is difficult and expensive while unlabeled data is easily available. In this case, without collecting new labeled target data, one may attempt to pre-train the model on related large labeled datasets such as ImageNet~\citep{deng2009imagenet} and adapt the model onto the unlabeled target data such as CT scan images~\citep{cadrin2022moving}. However, due to distributional shifts between the large labeled datasets and the target dataset, performance improvement is not always guaranteed~\citep{he2019rethinking}. Without careful consideration, the presence of distributional shifts can result in a decrease in performance of many classical statistical learning algorithms.


While DA is an important learning problem, a generic cure is hopeless if there is no useful relation between the source and target data that can aid in prediction. In particular, the DA problem is ill-posed in general because for any given algorithm and source data, there will always be some arbitrarily chosen target data such that the algorithm trained on the source data will not perform well. Establishing reasonable assumptions relating the source and target data is critical, and depending on these assumptions, many ways to formulate a DA problem exist.

One common way to formulate a DA problem is to assume that the conditional distribution of label given the covariate, $Y\mid X$, remains the same in both source and target data. When $Y\mid X$ is invariant, it is implied that the covariate distribution changes. This formulation is known as \textit{covariate shift}~\citep{shimodaira2000improving,quinonero2008dataset}. Successful approaches to tackle the covariate shift assumption include estimating the likelihood ratio between source and target covariates to correct {for} this shift~\citep{shimodaira2000improving,sugiyama2007covariate,sugiyama2012machine}. A related but different way of relating the source and target distributions is to assume that the conditional distribution of covariates, given the label, $X\mid Y$, is invariant. In this case, the marginal distributions of the label can differ. For this reason, this {formulation} is named \textit{label shift}~\citep{lipton2018detecting}. {DA} solutions typically involve correcting the likelihood ratio of labels using the conditional invariance of $X\mid Y$~\citep{azizzadenesheli2019regularized,tachet2020domain,garg2020unified}. Although covariate shift and label shift assumptions have been widely studied and successfully applied in some cases~\citep{sugiyama2007covariate,wu2021online}, their applicability is often limited in more practical scenarios.

Moving beyond the scope of covariate shift and label shift, another popular way of formulating DA is to assume the presence of invariant feature mappings {(i.e., transformations of covariates $X$)}. Then the DA problem is reduced to identifying features that are important for the underlying task and are invariant across the source and target domains. Domain Invariant Projection (DIP)~\citep{ baktashmotlagh2013unsupervised} was proposed as an attempt to identify these invariant features through projecting the source and target covariates into a common subspace. 
Subsequent works~\citep{ganin2016domain, tzeng2017adversarial, hoffman2018cycada} have advanced the common subspace approach by incorporating neural network implementation, demonstrating empirical success across many datasets.
Despite its empirical success, however, recent work by~\cite{johansson2019support, zhao2019learning} revealed that DIP may have a target risk much larger than its source risk, caused by the so-called label-flipping issue. Specifically, in the absence of target labels, in general there is no guarantee that DIP will find the true invariant representations. If DIP fails to do so, its target performance may deviate significantly from its source performance. 
What is worse is that currently there are no practical ways to check whether DIP has found the true invariant features, which presents a challenge for its practical use. 


In this work, we make the assumption on the existence of conditionally invariant components (CICs)~\citep{gong2016domain, heinze2021conditional}{---feature representations which are useful for prediction and whose distribution is invariant \textit{conditioned} on the labels across source and target domains (see~\defref{cic} for the formal definition of CICs)}. {With access to multiple source domain data that are related to the target domain, it becomes practically plausible to estimate CICs.} In this setting, the existence of CICs can be well-justified because any features that are invariant, conditioned on the labels across these heterogeneous source domains, are likely to remain conditionally invariant in the target domain. The idea that takes advantage of the heterogeneity in multiple datasets has its origins in causality, robust statistics~\citep{peters2016causal,buhlmann2020invariance} as well as stability-driven statistical analysis in the PCS framework~\citep{yu2020veridical}. Moreover, in anticausal learning scenarios~\citep{scholkopf2012causal} or when datasets are generated through structural causal models~\citep{pearl2009causality,chen2021domain}, CICs naturally emerge when unperturbed covariates are descendents of the labels.


Under the assumption on the existence of CICs, Conditional Invariant Penalty (CIP) is a widely used algorithm to identify CICs~\citep{gong2016domain, heinze2021conditional}, through enforcing the invariance of conditional feature distributions across multiple source domains. CIP is shown to achieve good target performance under several theoretical settings~\citep{heinze2021conditional, chen2021domain} and empirically~\citep{li2018domain,li2018deep,jiang2020implicit}.

Despite the rapid development on CIP to identify CICs, the understanding of DA algorithms based on CICs beyond simple structural equation models is still limited, and their ability to handle DA problems involving both covariate and label shifts is unclear in the previous literature. Additionally, while the generalized label shift has been proposed in~\cite{tachet2020domain}, it is not known how to reliably identify CICs via their proposed algorithms. Note that in the DA setting, target labels are unavailable, making it impossible to estimate target performance through validation or cross-validation. Therefore, it is crucial to rigorously quantify target risk guarantees of DA algorithms under assumptions made.

Additionally, while DA algorithms based on CICs exhibit target performance comparable to that on the source data, often they are not the top performers~\citep{heinze2021conditional, chen2021domain}. This is mainly due to the fact that these algorithms only use source data, and the requirement for CICs to maintain invariance across multiple source domains may discard many features useful for target prediction. 
In contrast, DIP leverages a single source data and target covariates, which can lead to better target performance when successful, making it a preferred choice for many practitioners. However, in some cases, DIP can perform significantly worse than the Empirical Risk Minimization (ERM)~\citep{wu2019domain,zhao2019learning,chen2021domain}. Given the potential advantanges and risks of DIP, it is natural to ask whether there is a way to detect potential failure of DIP algorithms and improve their target performance with theoretical guarantees.

\subsection{Our contributions}

In this work, we show that learning CICs enables significant progress in overcoming the aforementioned challenges in DA with theoretical guarantees. Specifically, under the assumption on the existence of CICs and the availability of multiple source datasets, we highlight three prominent roles that CICs can play in enhancing DA algorithms.


First, we show that CICs can enable label shift correction in the presence of both covariate and label distribution shifts. Based on this, we introduce the importance-weighted conditional invariant penalty (IW-CIP) algorithm and analyze its theoretical performance. Under structural equation models, we show that CICs can be {correctly} identified via both CIP and IW-CIP using labeled data from multiple source domains. Consequently, it is only the finite-sample error gap that accounts for the difference between the target risk of the IW-CIP classifier and that of the optimal conditionally invariant classifier.


Second, we demonstrate how CICs can be used to provide target risk lower bounds for other DA algorithms without requiring access to target labels. Provided that CICs are accurately identified, this lower bound allows for assessing the target performance of any other DA algorithms, including DIP. Consequently, this approach makes it possible to detect algorithm failures using only source data and unlabeled target data.


Lastly, we introduce JointDIP, a new DA algorithm that extends the domain invariant projection (DIP) by incorporating CICs. Under structural equation models, we prove that JointDIP reduces the possibility of label-flipping and enhances the target prediction performance of CICs-based algorithms. Therefore, this algorithm addresses both the conservative nature of CIP and the potential risks associated with DIP. Our findings are supported by numerical experiments on synthetic and real datasets, including MNIST, CelebA, Camelyon17, and DomainNet.

The rest of the paper is organized as follows. \secref{problem_setup} provides the necessary technical background and formally sets up the domain adaptation problem, along with a comprehensive review of related work in literature than in the introduction. In \secref{CIP}, we present the first role of CICs by introducing the IW-CIP algorithm and establish finite-sample target risk bounds to characterize its target risk performance (cf.~\thmref{IW-CIP_gen_bound}, \ref{thm:IW-CIP_gen_bound2}). \secref{DIP} describes the other two roles of CICs in DA, demonstrating how they can be used to detect the failure of other DA algorithms (cf.~\thmref{conditional_invariant_proxy_y}), and introducing the JointDIP algorithm to address the label-flipping issues of DIP (cf.~\thmref{SEM_guarantee}). In \secref{exp}, we complement our theoretical arguments with extensive numerical experiments on synthetic and real datasets, emphasizing the importance of learning CICs as an essential part of domain adaptation pipeline. 
Finally, we conclude in~\secref{discussion} with potential directions for future work.


\section{Background and problem setup}\label{sec:problem_setup}
In this section, we begin by defining the domain adaptation problem and introducing the concept of conditionally invariant components (CICs). We then outline two baseline DA algorithms: the conditional invariant penalty (CIP) algorithm, which finds conditionally invariant representation across multiple source domains, and the domain invariant projection (DIP) algorithm, which works with a single source domain and takes advantage of additional unlabeled target data. Subsequently, we review related work to better position our contributions within the existing body of literature.

\subsection{Domain adaptation problem setup}
We consider the domain adaptation problem with $M\ (M \geq 1)$ labeled source environments and one unlabeled target environment. By an \textit{environment} or a \textit{domain}, we mean a dataset $\dataset$ with \iid samples drawn from a common distribution $\distri$. Specifically, for $m \in \braces{1, \ldots, M}$, in the $m$-th source environment, we observe $\nsrc$ \iid samples
\[\Dsrc = \{(X_k^\tagm{m}, Y_k^\tagm{m})\}_{k=1}^{\nsrc},\]
drawn from the $m$-th source data distribution $\Psrc$. Independently of the source data, there are $\ntar$ \iid samples
\[\Dtar =  \{(X_k^\tagtar, Y_k^\tagtar)\}_{k=1}^{\ntar},\]
drawn from the target distribution $\Ptar$. We denote general random variables drawn from $\Psrc$ and $\Ptar$ as $(\Xsrc, \Ysrc)$ and $(\Xtar, \Ytar)$, respectively. In the domain adaptation setting, all the source data are observed, while only the target covariates $\Dtar_X = \{X_k^\tagtar\}_{k=1}^{\ntar}$ are observed in the target domain. For simplicity, throughout the paper, we assume that each covariate lies in a $p$-dimensional Euclidean space $\real^p$, and the labels belong to the set $\Yset=\braces{1, \ldots, L}$ where $L$ represents the total number of classes. 

To measure the performance of {a DA algorithm}, we define the \textit{target population risk} of a classifier $h: \real^p \to \braces{1, \ldots, L}$, mapping covariates to labels, via the $0$-$1$ loss as
\begin{align}
  \label{eqn:target_risk}
	\risktar{h} = \EE{\ones{h(\Xtar)\neq \Ytar}} = \prob{h(\Xtar)\neq \Ytar}.
\end{align}
Consequently, $1-\risktar{h}$ is the target population classification accuracy. Similarly, we define the $m$-th source population risk as
\begin{align}
  \label{eqn:source_risk}
	\risksrc{h} = \EE{\ones{h(\Xsrc)\neq \Ysrc}} = \prob{h(\Xsrc)\neq \Ysrc}.
\end{align}

The main goal of the domain adaptation problem is to use source and unlabeled target data to estimate a classifier $h: \real^p \to \braces{1, \ldots, L}$, from a set of functions called the hypothesis class $\Hset$, such that the target population risk is small. To quantify this discrepancy, we compare the target population risk with the \textit{oracle target population risk} $\risktar{\horacle}$ that we may aspire to achieve, where
\begin{align}\label{eqn:oracle_risk}
  \horacle = \argmin_{h \in \Hset} \risktar{h}.
\end{align}
Without specifying any relationship between the source distribution $\Psrc$ and the target distribution $\Ptar$, there is no hope that the target population risk of a classifier learned from the source and unlabeled target data is close to the oracle target population risk. Throughout the paper, we focus on DA problems where conditionally invariant components (CICs) across all source and target environments are present and they are correlated with the labels. The existence of CICs was first assumed in~\cite{gong2016domain} and~\cite{heinze2021conditional}. Under assumptions of arbitrary large interventions and infinite data, \cite{heinze2021conditional} established that their classifier built on CICs achieves distributional robustness. In this paper, instead of discussing distributional robustness of a classifier, we construct classifiers that have target population risks close to the oracle target risk. Before that, we introduce CICs and the best possible classifier built upon CICs.

\begin{definition}\label{def:cic}
	\sloppy \name{(Conditionally invariant components (CICs)~\cite{gong2016domain,heinze2021conditional})} Suppose that there exist $M$ source distributions $\{\Psrc\}_{1\leq m\leq M}$ and a target distribution $\Ptar$ on $\real^p\times\{1,2,\ldots,L\}$. We say that a function $\phi:\real^p\rightarrow\real^q$ is a \textit{conditionally invariant feature mapping}, if
	\begin{align}\label{eqn:cond_inv_defn}
		\begin{split}
			\Psrc_{\phi(X)\mid Y=y} = \Ptar_{\phi(X)\mid Y=y}, \ \ \forall m\in\{1,\ldots,M\},\ y\in\{1,\ldots,L\}.
		\end{split}
	\end{align}
	The corresponding feature representation $\phi(X)$ is called a \textit{conditionally invariant component (CIC)} if it has a single dimension ($q=1$), and \textit{CICs} if it is multidimensional ($q>1$). When $\phi$ maps $\real^p$ to $\{1,2,\ldots,L\}$ and satisfies~Eq.~\eqnref{cond_inv_defn}, we refer to it as a \textit{conditionally invariant classifier}.
\end{definition}
We will use the term ``CICs across source distributions'' instead if the feature representations is conditionally invariant on all $\Psrc$ but not necessarily on $\Ptar$. {With this definition, our assumption on the ``existence of CICs'' can be stated as: there exists a conditionally invariant mapping $\phi$ such that $\phi(X)$ is CIC(s) across $M$ source distributions $\{\Psrc\}_{1\leq m\leq M}$ and a target distribution $\Ptar$.} In~\defref{cic}, constant representations can be viewed as a trivial case of CICs. However, only CICs that are useful for prediction are beneficial for DA problems. Thus, our assumption on the existence of CICs refers to the existence of CICs useful for prediction. The best possible classifier built upon such CICs is the following optimal classifier.

\begin{definition}\label{def:optimal_cie}
\sloppy \name{(Optimal conditionally invariant classifier)} Let $\Phi$ and $\Gset$ are classes of functions where each function $\phi\in\Phi$ maps $\real^{p}$ to $\real^q$, and each function $g\in\Gset$ maps $\real^{q}$ to $\{1,2,\ldots,L\}$. Under the assumption on the existence of CICs, we define the \textit{optimal conditionally invariant classifier $\hstar$} as
\begin{align}\label{eqn:optimal_cond_inv_estimator}
    \begin{split}
        \hstar &= \gstar \circ \phistar, \\
        \gstar,\phistar &= \argmin_{g\in\Gset,\phi\in\Phi} \ \ \risktar{g\circ\phi}\\
        &\hspace{0.2in} \textnormal{subject to }\Psrc_{\phi(X)\mid Y=y} = \Ptar_{\phi(X)\mid Y=y}, \ \ \forall m\in\{1,\ldots,M\},\ y\in\{1,\ldots,L\}.
    \end{split}
\end{align}
\end{definition}

The conditionally invariant classifier above is optimal in the sense that it minimizes the target population risk while the learned representation is conditionally invariant across $\Psrc$ ($1\leq m\leq M$) and $\Ptar$. When evaluating the target performance of a CICs-based classifier $h=g\circ\phi$, instead of directly comparing it with $\horacle$, we consider comparing it with $\hstar$ first, and then relating $\hstar$ to $\horacle$. Intuitively, the target risk difference between $\hstar$ and $\horacle$ will not be significant when the dimension of CICs is sufficiently large (c.f. \propref{linear_sem_example2}). In this case, to build a CICs-based classifier with a guaranteed target risk bound compared to $\horacle$, it suffices to find a classifier which achieves a low target risk compared to $\hstar$. 

Another widely-used class of DA algorithms known as domain invariant projection (DIP) (cf.~\secref{DIP_defn}) seeks to find a feature mapping $\phi$ which matches the source and target marginal distribution of $\phi(X)$. Although it has been shown to be successful in some practical scenarios~\citep{ganin2016domain, mao2017least, hoffman2018cycada, peng2019moment}, in general there is no guarantee for the low target risk. Both \cite{johansson2019support} and \cite{zhao2019learning} provide simple examples where DIP can even perform worse than a random guess, as if features learned by DIP ``flip'' the labels. We formulate the rationale behind their examples by defining label-flipping features as follows. 


\begin{definition}\label{def:label-flipping}
	\sloppy \name{(Label-flipping feature)} Without loss of generality, consider the first source distribution $\Psrcone$ and the target distribution $\Ptar$ on $\real^p\times\{1,2,\ldots,L\}$. We say that a function $f:\real^p\rightarrow \real$ is a \textit{label-flipping} feature mapping, if there exists $y\in\{1,2,\ldots,L\}$ such that\footnote{ When $Y$ is binary, the definition is equivalent to  $\rho\parenth{f(\Xsrcone),\Ysrcone}\cdot \rho\parenth{f(\Xtar),\Ytar} < 0$.}
	\begin{align}\label{eqn:label_flipping_feat_defn}
		\rho\parenth{f(\Xsrcone),\ones{\Ysrcone=y}}\cdot \rho\parenth{f(\Xtar),\ones{\Ytar=y}} < 0,
	\end{align}
	where $\rho(\cdot,\cdot)$ denotes the correlation between random variables. The corresponding feature $f(X)$ is called a label-flipping feature.
\end{definition}

If the label-flipping features exist between a source distribution and the target distribution, they can be inadvertently learned by DIP as part of its learning algorithm for domain invariant representation. This can lead to degraded prediction performance on the target domain, as the sign of the correlation between these features and the labels changes under source and target distributions. We refer to it as the label-flipping issue {of DIP}. 

Next we define an anticausal data generation model {that} serves as a concrete working example for validating our assumptions and establishing new results. While not all of our theoretical results depend on this model, it nevertheless aids in illustrating how our methods work. 
\begin{definition}\label{def:model1}
\name{(General anticausal model)} We say that the data generation model is a general anticausal model, if the source and target distributions are specified as follows. Under the $m$-th source distribution $\Psrc$, source covariates and label are generated by
\begin{align*}
&\Ysrc \sim \textnormal{ Categorical}\parenth{\probsrc{1},\probsrc{2},\ldots,\probsrc{L}},\\
&\Xsrc = \fsrc(\Ysrc) +  \epssrc, \text{ $\epssrc\perp\Ysrc$,}
\end{align*}
where $\probsrc{y}\in(0,1), \sum_{y=1}^L\probsrc{y}=1$, and $\fsrc:\Yset=\{1,\ldots,L\}\to\R^p$ is a deterministic function defining the mechanism between the $m$-th source covariates $\Xsrc$ and label $\Ysrc$. Under the target distribution $\Ptar$, target covariates and label are generated independently of the source data by
\begin{align*}
&\Ytar \sim \textnormal{ Categorical}\parenth{\probtar{1},\probtar{2},\ldots,\probtar{L}},\\
&\Xtar = \ftar(\Ytar) +  \epstar, \text{ $\epstar\perp\Ytar$,}
\end{align*}
where $\probtar{y}\in(0,1), \sum_{y=1}^L\probtar{y}=1$, and $\ftar:\Yset=\{1,\ldots,L\}\to\R^p$ is a deterministic function defining the mechanism between the target covariates $\Xtar$ and label $\Ytar$. The noise terms $\epssrc, \epstar\in\real^p$, are generated \iid from a zero-mean distribution $\Peps$.
\end{definition}
Under the general anticausal model, conditioned on the labels, the mechanism functions $\fsrc$, $\ftar$, $m=1,\ldots,M,$ determine the difference between the source and target conditional distributions $X \mid Y$ because the noise terms share the same distribution $\Peps$. This generative model generalizes various perturbations that can occur in an anticausal model~\citep{pearl2009causality}. For example, label shift might occur if the marginal distributions of $Y$ differ. Covariate shift, conditioned on the labels, can occur when the deterministic functions $\fsrc,\ftar$ vary. We present an explicit example under this model, including the presence of both CICs and label-flipping features in~\appref{anticausal_example}.

\paragraph{Notation} To distinguish subscripts from coordinates, we represent the $j$-th coordinate of a constant vector $x$ as $\coord{x}{j}$, and similarly, for a random vector $X$, we use $\coord{X}{j}$. Next, we introduce several notations for the empirical equivalents of population quantities. For the $m$-th source and target datasets $\Dsrc$, $\Dtar$, $m=1,\ldots,M$, we use $\Phsrc$ and $\Phtar$ to denote the empirical data distributions, respectively. We define the $m$-th source and target empirical risk of a classifier $h\in\Hset$ as 
\begin{align*}
	\riskhsrc{h} &=  \Ep{(X, Y)\sim\Phsrc}{\ones{h(X)\neq Y}} = \frac{1}{\nsrc}\sum_{k=1}^{\nsrc}\ones{h(\Xsrc_k)\neq \Ysrc_k}, \;\;\text{ and }\\
	\riskhtar{h} &=   \Ep{(X, Y)\sim\Phtar}{\ones{h(X)\neq Y}}=  \frac{1}{\ntar}\sum_{k=1}^{\ntar}\ones{h(\Xtar_k)\neq \Ytar_k}.
\end{align*}
For any mapping $\phi$ defined on $\R^p$, we use $\Pmarsrc{\phi(X)}$ and $\Pcondsrc{\phi(X)}{Y=y}$ to denote the $m$-th source marginal distribution of $\phi(X)$ and the $m$-th source conditional distribution of $\phi(X)$ given its label $Y=y$. Similarly, we use $\Pmartar{\phi(X)}$ and $\Pcondtar{\phi(X)}{Y=y}$ to denote the target marginal distribution of $\phi(X)$ and the target conditional distribution of $\phi(X)$ given its label $Y=y$. The corresponding empirical quantities are denoted by $\Phmarsrc{\phi(X)}$, $\Phcondsrc{\phi(X)}{Y=y}$, $\Phmartar{\phi(X)}$, and $\Phcondtar{\phi(X)}{Y=y}$, respectively.

Letting $\distriP$ be a distribution on $\real^q$ and $\Gset$ be a function class where each function $g\in\Gset$ maps $\real^q$ to $\real$, we recall the Rademacher complexity as
\begin{align}
    \rad{n,\distriP}{\Gset}&\defn\Ep{Z_k\iidsim\distriP, \sigma_k\iidsim\sigma}{\sup_{g\in\Gset}\abss{\frac 1 n\sum_{k=1}^n\sigma_kg(Z_k)}},
\end{align}
\sloppy where $\sigma_k$'s are random variables drawn independently from the Rademacher distribution, i.e., $\prob{\sigma_k=1}=\prob{\sigma_k=-1}=1/2$. Additionally, for any two distributions $\mathcal{P}$ and $\mathcal{Q}$ on $\real^q$ and a class of classifiers $\Gset$ where each function $g\in\Gset$ maps $\real^q$ to $\Yset=\braces{1,2,\cdots,L}$, we define the $\Gset$-divergence between these two distributions as
\begin{align}\label{eqn:g_div}
    \Gdiv{\mathcal{P}}{\mathcal{Q}}&\defn\sup_{g\in\mathcal{G}}\max_{y=1,2,\ldots,L}\abss{\Ep{Z\sim \mathcal{P}}{\ones{g(Z)= y}}-\Ep{Z\sim \mathcal{Q}}{\ones{g(Z)= y}}}.
\end{align}
Note that the $\Gset$-divergence defined in Eq.~\eqnref{g_div} can be seen as an extension of the $\Hset$-divergence introduced in~\cite{ben2010theory} to multiclass classification. It is also an instantiation of the Integral Probability Metrics~\citep{muller1997integral} with a specific choice of function class. 


\subsection{Two baseline DA algorithms}
With the background of domain adaptation in place, we now proceed to introduce two DA algorithms in this subsection. The first algorithm is the conditional invariant penalty (CIP) that finds conditionally invariant representation across multiple source domains. The second algorithm is domain invariant projection (DIP), which works with a single source domain but also requires target covariates. These two DA algorithms serve as baselines for evaluating the methods we introduce in the subsequent sections.

\subsubsection{Conditional invariant penalty (CIP)}\label{sec:CIP_defn}
The \textit{conditional invariant penalty} (CIP) algorithm uses the multiple labeled source environments to learn a feature representation that is conditionally invariant across all source domains~\citep{gong2016domain, heinze2021conditional}. More precisely, the CIP algorithm is a two-stage algorithm which minimizes the average source risk across domains while enforcing the first-stage features to be conditionally invariant.

\paragraph{Population CIP \citep{gong2016domain, heinze2021conditional}:} The population CIP classifier is formulated as a constrained optimization problem with a matching penalty on the conditional:
    \begin{align}
      \label{eqn:pop_CIP_naive}
      \begin{split}
      \hcip &= \gcip \circ \phicip,  \\
      \gcip,\phicip &= \argmin_{g\in\Gset,\phi\in\Phi} \ \ \frac{1}{M}\sum_{m=1}^{M}\risksrc{g\circ\phi} \\
      &\hspace{0.2in} \text{subject to } \ \ \discrepancy{\Pcondsrc{\phi(X)}{Y}}{\Pcondsrcmp{\phi(X)}{Y}}  = 0 \text{ for all $m\neq m'\in\{1,\ldots,M\}$},
        \end{split}
      \end{align}
    where $\risksrc{\cdot}$ is the $m$-th population source risk, and $\discrepancy{\cdot}{\cdot}$ is a distributional distance between two distributions such as the maximum mean discrepancy (MMD)~\cite{gretton2012kernel} or generative adversarial networks (GAN) based distance~\citep{ganin2016domain}. The optimization is over the set of all two-stage {functions} where the first stage function belongs to $\Phi$ and the second stage to $\Gset$. The constraint on the conditional $\phi(X) \mid Y$ enforces CIP to use CICs across all $\Psrc$ ($1\leq m\leq M$) to build the final classifier. Here the hope is that if feature mappings are conditionally invariant across the heterogeneous source distributions $\Psrc$, they are likely to be also conditionally invariant under the target distribution $\Ptar$. As a result, a classifier built on these features would generalize to the target distribution.
    
    While CIP was originally introduced in the context of domain adaptation~\citep{gong2016domain} and distributional robustness~\citep{heinze2021conditional}, it is closely related to the problem of \textit{domain generalization}, or \textit{out-of-distribution} (OOD) generalization~\citep{ben2009robust}. The main distinction between domain generalization and domain adaptation is the absence of any unlabeled target data during the training phase in domain generalization~\cite{wang2022generalizing}. Since CIP does not require unlabeled target data for training, it has been effectively applied to domain generalization, demonstrating significant empirical success~\citep{li2018domain,li2018deep,jiang2020implicit}. In this paper, we use the terms “domain generalizaton” and “domain adaptation” interchangeably since our focus is on evaluating and comparing the prediction performance of models on new target domains.

\paragraph{Finite-sample CIP:} In finite-sample case, instead of putting a strict constraint in the optimization, CIP adds the conditional invariant penalty on the empirical distributions, using a pre-specified parameter $\lamcip>0$ to control the strength of regularization as follows:
    \begin{align}
        \label{eqn:finite_CIP_naive}
        \begin{split}
        \hhcip &= \ghcip \circ \phihcip,  \\
        \ghcip,\phihcip &= \argmin_{g\in\Gset,\phi\in\Phi} \ \ \frac{1}{M}\sum_{m=1}^{M}\riskhsrc{g\circ\phi} + \frac{\lamcip}{LM^2}\cdot\sum_{y=1}^L\sum_{m\neq m'}  \probdiv{\Phcondsrc{\phi(X)}{Y=y}}{\Phcondsrcmp{\phi(X)}{Y=y}}.
        \end{split}
    \end{align}

While CIP makes use of CICs across multiple source distributions to construct the classifier, it does not exploit the target covariates that are also available in our DA setting. Indeed, CIP has often been employed as a method for domain generalization, where the goal is to generalize to the target domain without access to unlabeled target data. Because CIP finds invariant features through multiple source domains without considering target covariates during training, the learned features may be overly conservative for generalizing to the target domain. Later, we show how these CICs can be used to improve target prediction performace in our DA setting. Next, we introduce another class of DA algorithm which takes advantage of target covariates.

\subsubsection{Domain invariant projection (DIP)}\label{sec:DIP_defn}

In contrast to CIP which utilizes multiple source data, \textit{domain invariant projection} (DIP)~\citep{baktashmotlagh2013unsupervised, ganin2016domain} uses labeled data from a single source as well as unlabeled data from the target domain to seek a common representation that is discriminative about the source labels. The idea of finding a common representation is realized via matching feature representations across source and target domains. Without loss of generality, we formulate DIP using the first source distribution $\Psrcone$, but in principle it can be formulated with any source distribution $\Psrc$.

\paragraph{Population DIP \citep{baktashmotlagh2013unsupervised, ganin2016domain}:} \label{def:DIP}
We define the population DIP as a minimizer of the source risk under the constraint of marginal distribution matching in feature representations:
\begin{align}\label{eqn:pop_DIP}
    \begin{split}
        \hdip &= \gdip \circ \phidip, \\
        \gdip, \phidip &= \argmin_{g\in\Gset,\phi\in\Phi} \ \ \risksrcone{g\circ\phi}  \\
        &\hspace{0.2in} \text{subject to } \ \ \probdiv{\Pmarsrcone{\phi(X)}}{\Pmartar{\phi(X)}}  = 0,
    \end{split}
\end{align}
where $\discrepancy{\cdot}{\cdot}$ is a distributional distance between two distributions as in Eq.~\eqnref{pop_CIP_naive}. The constraint ensures that the source marginal distribution in the {feature representation} space is well aligned with the target marginal distribution. While our formulation only utilizes the single source, DIP also has a multi-source pooled version~\citep{peng2019moment} where marginal distributions in the representation space are matched across all $\Psrc$ ($1\leq m \leq M$) and $\Ptar$.

\paragraph{Finite-sample DIP:} In the finite sample setting, the hard constraint used in population DIP is relaxed to take a regularization form, and therefore
\begin{align}\label{eqn:finite_DIP}
    \begin{split}
        \hhdip &= \ghdip \circ \phihdip, \\
        \ghdip, \phihdip &= \argmin_{g\in\Gset,\phi\in\Phi} \ \ \riskhsrcone{g\circ\phi} + \lamdip\cdot\probdiv{\Phmarsrcone{\phi(X)}}{\Phmartar{\phi(X)}},
    \end{split}
\end{align}
where $\lamdip>0$ is a regularization parameter that balances between the source risk and the matching penalty across the empirical source and target marginal distributions for the feature representation.

While DIP makes use of the target covariates to learn domain-invariant representations, in general it does not have target risk guarantees. In particular, the matching penalty can force DIP to learn label-flipping features (see~\defref{label-flipping}) because unlike CIP, it aligns features in the marginal representation space. DIP leverages information about target covariates to learn invariant features, but this comes at the cost of potential label-flipping issue (c.f. see~\figref{dip_vs_jdip} and~\thmref{SEM_guarantee}).
Furthermore, DIP can fail when the marginal distribution of $Y$ is perturbed under an anticausal data generation model. For example, \cite{chen2021domain} demonstrates that DIP can be perform worse than CIP in the presence of label shift.

\subsection{Related work}\label{sec:related_work}
As a subfield of transfer learning, \textit{domain adaptation} (DA), also known as transductive transfer learning~\citep{redko2020survey}, aims to develop effective algorithms when distribution shifts exist across training data (source domains) and test data (target domain). In DA, it is usually assumed that we have access to some unlabeled target data, and we aim to build a model specifically for the target domain. Another related terminology \textit{domain generalization} (DG), or out-of-distribution (OOD) generalization~\citep{ben2009robust}, instead assumes that unlabeled target data is unobtainable, and seeks a model to generalize appropriately for all possible test domains. In this paper, we do not distinguish these two terms particularly, and simply consider the DA setting, i.e., we have access to an unlabeled target dataset, and our goal is to find a classifier that performs well on the target domain.

\paragraph{Domain adaptation with common feature representations}

One line of work in DA focuses on relating source and target domains by learning a common feature representation across the domains. \cite{pan2008transfer, pan2010domain} first proposed to find transferable components across domains in a reproducing kernel Hilbert space, such that data distributions in different domains are close. Similar designs of creating intermediate representations across domains were investigated in \citep{gopalan2011domain, gong2012geodesic}, and later the formal idea of matching probability distributions and extracting invariant information was introduced in~\citep{baktashmotlagh2013unsupervised}. They named their approach domain invariant projection (DIP), which projects data to a low-dimensional latent space where source and target covariates distributions are matched. This DIP type of approach was then widely employed in research on neural networks for DA \citep{sun2016deep, ganin2016domain, hoffman2018cycada}.
Specifically, \cite{ganin2016domain} introduced the Domain-Adversarial Neural Network (DANN), which leverages Generative Adversarial Network (GAN) based distributional distances and neural networks to learn feature representations. \cite{sun2016deep} proposed CORAL, a method that aligns the first and second moments of last layer activations in neural networks, {whereas \cite{courty2016optimal} proposed to match source and target domains via regularized optimal transport.} Since then, subsequent research has emerged that applies different neural-network-based distances~\citep{long2017deep, courty2017joint, long2018conditional, hoffman2018cycada, peng2019moment}. While these studies have effectively demonstrated the empirical success of their methodologies across various image and text datasets, there is still a lack of understanding about the specific conditions under which these methods can achieve good target performance.

From a theoretical point of view, the rationale for DIP has been rigorously established by~\cite{ben2006analysis, ben2010theory}, where they prove a target risk bound via Vapnik-Chervonenkis (VC) theory. Further studies have analyzed the source and target risk difference using different divergence measures~\citep{mansour2009domain, cortes2010learning, cortes2011domain, cortes2014domain, hoffman2018algorithms}; see the survey by \cite{redko2020survey} for a complete review. In all these works, the target risk bound typically includes three components: the source risk of the classifier, a divergence term measuring invariance of the representation, and an optimal joint error term. Then DIP objective can be viewed as minimizing the sum of the first two terms in order to achieve a low target risk. However, \cite{johansson2019support, zhao2019learning} argued that DIP can completely fail in certain scenarios because the joint error term is not observable and cannot be controlled by DIP. For instance, there can be label-flipping features that achieve perfect source accuracy and invariance of representation, but result in poor performance on the target domain. Addressing how to avoid such cases in DIP has not been adequately explored.

\paragraph{DA methods under conditional invariance and label shift}

The second class of DA methods arises from exploring invariance solely from multiple labeled source domains, without using target covariates. One such category of invariance is known as conditional invariance, which aims to discover feature representations $\phi(X)$ that are invariant conditional on the label $Y$ across source distributions. It was first introduced in \cite{gong2016domain}, where the authors proposed to find conditional transferable components after proper location-scale transformations. Later \cite{heinze2021conditional} applied a related approach by classifying features into "core" ones, which are conditionally invariant across domains, and "style" ones, whose distribution may shift significantly. They sought to construct a classifier built upon only the core features by imposing the conditional invariant penalty (CIP). \cite{chen2021domain} further developed a theoretical framework under structural equation models to analyze the effect of CIP when the data generation process is anticausal. The emergence of conditionally invariant features, or CICs, is a natural consequence under anticausal data generation where the covariates, which are descendants of the labels, remain unperturbed. Consequently, most tasks tackled by DA methods utilizing conditional invariance typically involve anticausal problems.

It is also worth noting that under conditional invariance, it is convenient to study the label shift problem, where the marginal distributions of the label $Y$ are shifted across domains. This shift in label distribution is common in many scenarios, for example, the distribution of ArXiv paper categories can be influenced by changes in research topic trends over time~\citep{wu2021online}. \cite{lipton2018detecting} first proposed the label shift correction algorithm, which estimates the amount of label shift by using the conditional invariance of the covariates and the moment matching equation. To ensure numerically stability, several variants of the algorithms have been further introduced. For instance, \cite{azizzadenesheli2019regularized} formulated an $\ell_2$-norm regularized least squares optimization problem, while~\cite{tachet2020domain} considered a constrained least squares problem and the generalized label shift assumption. \cite{garg2020unified} introduced a maximum likelihood estimation approach and provided a unified view of previous label shift correction algorithms. More recently,~\cite{chen2022estimating} has considered a distributional shift named Sparse Joint Shift (SJS), which allows for both labels and a few covariates to shift, however, the generalized label shift considered in~\cite{tachet2020domain} and this paper permits distribution shifts of more general latent feature representations.

\paragraph{Other DA methods}

Besides conditional invariance, other types of invariance have also attracted increasing attention. For example, \cite{arjovsky2019invariant} came up with the Invariant Risk Minimization (IRM) method to identify causal features via enforcing the invariance of label $Y$ given the features $\phi(X)$. The idea is that the optimal linear predictor on top of these features $\phi(X)$ will remain the same across all source domains---this is in contrast to the assumption on CICs where there exist conditionally invariant features conditioned on the labels $Y$. 
Following IRM, analogous invariance has been explored. For example, \citep{krueger2021out, xie2020risk} proposed V-REx, which maintains risk invariance by reducing the variance of risks across source domains. Meanwhile, \citep{koyama2020invariance,shi2021gradient} maximized the gradient inner product between different domains to leverage invariant gradient direction, with the latter referring to their method as Fish.
While these approaches were initially motivated to identify invariant features in causal data models, they have also been applied to anticausal problems. However, as pointed out by~\cite{rosenfeld2020risks,kamath2021does}, IRM may fail to capture the correct invariance, especially when the model is non-linear. 

Counterfactual invariance is another type of invariance introduced in~\citep{veitch2021counterfactual,jiang2022invariant}, where the aim is to seek representations that are counterfactually invariant to a spurious factor of variation. \citep{wang2022unified} show that the formulation of learning counterfactual invariant representations is closely related to different formulations of invariant representation learning algorithms, depending on the underlying causal structure of the data. Aside from the robustness achieved by enforcing specific invariance, a more general framework is distributionally robust optimization (DRO)~\citep{ben2013robust, duchi2021statistics}. DRO aims to minimize the worst-case loss over an uncertainty set of distributions, such as a ball around the training distribution. However, this approach may result in models that are overly conservative. GroupDRO~\citep{hu2018does,oren2019distributionally} intends to address this issue by defining the uncertainty set as mixtures of source distributions (or groups). Furthermore, \citep{sagawa2019distributionally} improve the worst-group performance of groupDRO using an online algorithm to update group weights. 
It should be noted that related study of invariance and robustness has also appeared in works from a causal inference point of view~\citep{peters2016causal, meinshausen2018causality, rothenhausler2018anchor, magliacane2018domain}.

Other DA algorithms have also been introduced based on different assumptions relating source and target data, which do not fall into the above classes of methods relying on invariance. For instance, self training, which originated from the semi-supervised learning literature~\citep{chapelle2009semi}, recursively adapts a classifier to fit pseudolabels predicted by a previous model using unlabeled data in a new domain~\citep{amini2003semi}. Theoretical properties of self training has been studied recently in~\citep{kumar2020understanding, wei2020theoretical}. Data augmentation serves as another beneficial way of obtaining better performance~\citep{simard1998transformation, zhang2017mixup, yao2022improving}, especially when possible perturbations across domains are well understood. In addition, a different perspective of DA stems from meta-learning~\citep{thrun2012learning, finn2017model}. This ``learning-to-learn'' approach aims to distill knowledge of multiple learning episodes to improve future learning performance.


%
%

\section{Importance-weighted CIP with  target risk guarantees}
\label{sec:CIP}

In the previous section, we introduced two baseline DA algorithms, namely CIP and DIP. However, it is crucial to acknowledge the limitations of these algorithms, as they rely on specific assumptions that may not hold in more complex DA scenarios. For instance, while CIP identifies CICs to build the classifier, its ability to generalize to the target domain is limited when the marginal label distributions shift across source and target. DIP faces similar limitations and is also subject to the additional uncertainty of learning label-flipping features.

In this section, we present the first role of CICs in addressing situations where neither covariate shift nor label shift assumptions hold, and we introduce the importance-weighted conditional invariant penalty (IW-CIP) algorithm. The intuition behind label shift correction with CICs is as follows. It is known that one can correct for label distribution shift if the conditional $X\mid Y$ is invariant and only the label $Y$ distribution changes across source and target~\citep{lipton2018detecting}. However, the assumption on invariance of $X\mid Y$ can be too rigid. Here we assume {the existence of datasets from multiple source domains} and a conditionally invariant feature mapping $\phiinv$ such that $\phiinv(X) \mid Y$ remains invariant across source and target distributions. By identifying a conditionally invariant $\phiinv(X)$ via CIP, we can apply the label shift correction algorithm to correct the label shift. Once the label shift is corrected, the joint distribution of $(\phiinv(X), Y)$ becomes invariant under source and target distributions, allowing us to control the target risk of any algorithms built upon the source $(\phiinv(X), Y)$. IW-CIP is an extension of CIP by incorporating a label shift correction step before building the classifier based on CICs.


The rest of the section is structured as follows. In~\secref{iw_estimation}, we offer a review on label shift correction and its application in our context. Then, we introduce our new algorithm, IW-CIP, in~\secref{iwcip_defn}. In~\secref{target_risk_bound_iwcip}, we establish target risk guarantees for IW-CIP.

\subsection{Importance weights estimation}\label{sec:iw_estimation}
When the $m$-th source distribution and the target distribution share the same conditional $X\mid Y$ but have different label distributions, the main idea of label shift correction in~\cite{lipton2018detecting} lies in that the true importance weights vector $\wm \in \real_+^L$, defined as
\begin{align}\label{eqn:true_w}
    \coordw{\wm}{j} \defn \frac{\prob{\Ytar=j}}{\prob{\Ysrc=j}},
\end{align}
can be estimated by exploiting the invariance of the conditional $X \mid Y$. In our setting, while the source and target distribution does not share the same conditional $X\mid Y$, according to our assumption, we have a feature {mapping} $\phiinv$ whose corresponding feature representation is a CIC by \defref{cic}, i.e., $\phiinv(X) \mid Y$ is invariant under source and target distributions. Then by treating $\phiinv(X)$ as the new features, we can still correct for the label shifts.

More precisely, for $\Yset=\braces{1, \ldots, L}$, we have the following distribution matching equation between the $m$-th source domain and the target domain: for any $g \in \Gset$ mapping to $\Yset$, and for any $i, j \in\Yset$,
\begin{align}
	\label{eqn:label_shift_equation}
  \prob{g\circ\phiinv(\Xtar) = i} &= \sum_{j=1}^L \prob{g\circ\phiinv(\Xtar)=i \mid \Ytar = j} \prob{\Ytar= j} \nonumber\\
  &\overset{(*)}{=} \sum_{j=1}^L \prob{g\circ\phiinv(\Xsrc)=i \mid \Ysrc = j} \prob{\Ytar= j} \nonumber\\
  &= \sum_{j=1}^L \prob{g\circ\phiinv(\Xsrc)=i, \Ysrc = j} \coordw{\wm}{j},
\end{align}
where {step ($*$)} follows from the invariance of $\phiinv(X) \mid Y$. In the matrix-vector form, we can write
\begin{align}
    \label{eqn:label_shift_equation_matrix}
    \mu_{g\circ\phiinv} = C_{g\circ\phiinv}^{\tagm{m}} \wm,
\end{align}
where $\mu_{h}$ denotes the predicted probability of $h$ under the target covariate distribution, and $C_{h}^{\tagm{m}}$ is the confusion matrix on the $\Psrc$ given by
\begin{align}\label{eqn:confusion_matrix}
    C_{h}^{\tagm{m}}[i,j] = \prob{h(\Xsrc)=i, \Ysrc = j}.
\end{align}
It is then sufficient to solve the linear system~\eqnref{label_shift_equation_matrix} to obtain $\wm$. In practice, with finite-sample source data, we use $\hhcip=\ghcip\circ\phihcip$ in place of $g\circ\phiinv$. To obtain the estimated importance weights $\whm$, we replace $\mu_{\hhcip}$ and $ C_{\hhcip}^{\tagm{m}}$ with their empirical estimates. \sloppy In our multiple source environments scenario, we write $\w = (\wmi{1},\wmi{2},\ldots, \wmi{M})\in\R^{L\times M}$ and $\wh = (\whmi{1},\whmi{2},\ldots, \whmi{M})\in\R^{L\times M}$ to denote the true and estimated importance weights for all source distributions, respectively.

\subsection{Our proposed IW-CIP algorithm}\label{sec:iwcip_defn}
When the target label distribution remains unchanged, {the population CIP in Eq.~\eqnref{pop_CIP_naive}} is capable of generalizing to target data because of the invariance of the joint distribution $(\phicip(X), Y)$. However, although CIP ensures invariance in the conditional distribution $\phicip(X) \mid Y$, the joint distribution $(\phicip(X), Y)$ may not be invariant when label shift is present. Indeed, CIP may perform poorly on the target data if the target label distribution substantially deviates from the source label distribution (see experiments on synthetic data and rotated MNIST in \secref{exp}). To address such distributional shift, we propose the IW-CIP algorithm, which combines importance weighting for label shift correction with CIP.

We define the $m$-th weighted source risk for a hypothesis $h \in \Hset$ and a weight vector $\wwm \in \real^L$ as follows:
\begin{align*}
  \risksrcw{h}{\wwm} = \EE{\coordw{\wwm}{\Ysrc}\cdot\ones{h(\Xsrc)\neq\Ysrc}}.
\end{align*}
\sloppy In particular, if $\wwm = \wm$, i.e., $\coordw{\wwm}{j}=\frac{\prob{\Ytar=j}}{\prob{\Ysrc=j}}$ for all $j=1,\ldots,L$, then it is easy to see that $\risksrcw{h}{\wwm} = \risktar{h}$ as long as the conditional distributions $h(X) \mid Y=y$ are invariant across $\Psrc$ and $\Ptar$. Hence, with an appropriate choice of the weight vector $\wwm$, the weighted source risk can serve as a proxy for the target risk. We are ready to introduce the importance-weighted CIP (IW-CIP) algorithm.

\paragraph{Population IW-CIP:} The population IW-CIP is obtained in three steps.
\begin{enumerate}[Step 1:]
    \item Obtain a conditionally invariant feature mapping $\phicip$ and the corresponding CIP classifier $\hcip$ via the CIP algorithm in Eq.~\eqnref{pop_CIP_naive}.
    \item Use $\hcip$ in place of $g\circ\phiinv$ in Eq.~\eqnref{label_shift_equation_matrix} to obtain importance weights $\wm$.
    \item Compute a function $g\in\Gset$ as well as a new conditionally invariant feature mapping $\phi\in\Phi$ to minimize the importance-weighted source risks as follows.
    \begin{align}
        \label{eqn:pop_CIP_label_shift_corrected}
            \hiwcip &= \giwcip \circ \phiiwcip,  \notag\\
            \giwcip, \phiiwcip &= \argmin_{g\in\Gset,\phi\in\Phi} \ \ \frac{1}{M}\sum_{m=1}^{M}\risksrcw{g\circ\phi}{\wm}  \\
            &\hspace{0.2in} \text{subject to } \ \ \discrepancy{\Pcondsrc{\phi(X)}{Y}}{\Pcondsrcmp{\phi(X)}{Y}}  = 0, \ \forall m\neq m'\in\{1,\ldots,M\}.\notag
    \end{align}
\end{enumerate}
IW-CIP enforces the same constraint on the data representation $\phi$ as in Eq.~\eqnref{pop_CIP_naive}. However, unlike CIP, the objective of IW-CIP is  the importance-weighted source risk, which can serve as a proxy for the target risk under the label distribution shifts. Therefore, IW-CIP can generalize better on the target environment in the presence of label shifts.

Our formulation of IW-CIP in Eq.~\eqnref{pop_CIP_label_shift_corrected} involves retraining the feature mapping $\phiiwcip$, instead of using $\phicip$ obtained in Step 1, despite that both feature mappings are conditionally invariant. From a theoretical perspective, retraining $\phi$ enables a more simplified derivation of the generalization bound for IW-CIP where we compare the objective of the retrained classifier with that of the optimal conditionally invariant classifier (c.f. see Eq.~\eqnref{termD} in the proof of~\thmref{IW-CIP_gen_bound2}). Empirically, we observe that retraining often leads to better empirical performance, providing both theoretical and practical advantages of our formulation.



\paragraph{Finite-sample IW-CIP:} The finite-sample IW-CIP is obtained from the population IW-CIP after replacing all the population quantities by the corresponding empirical estimates {and turning constraints into a penalty form}. We first solve finite-sample CIP from Eq.~\eqnref{finite_CIP_naive}, then estimate importance weights and solve
\begin{align}
\label{eqn:finite_CIP_LabelCorr_CIP}
        \hhiwcip &= \ghiwcip \circ \phihiwcip,\notag\\
        \ghiwcip, \phihiwcip &= \argmin_{g\in\Gset, \phi\in\Phi} \ \ \frac{1}{M}\sum_{m=1}^{M}\riskhsrcw{g\circ\phi}{\whm}\\
        &\hspace{8em} +\frac{\lamiwcip}{LM^2}\cdot\sum_{y=1}^L\sum_{m\neq m'}  \probdiv{\Phcondsrc{\phi(X)}{Y=y}}{\Phcondsrcmp{\phi(X)}{Y=y}},\notag
\end{align}
where $\whm$ is an estimate of $\wm$, obtained by solving $\muh_{\hhcip} = \widehat{C}_{\hhcip}^{\tagm{m}}\whm$. Here $\muh_{\hhcip}$ and $\widehat{C}_{\hhcip}^{\tagm{m}}$ are empirical estimates of $\prob{\hhcip(\Xtar)=i}$ and $\prob{\hhcip(\Xsrc)=i, \Ysrc=j}$, respectively. For any $\phi\in\Phi$, we define a shorthand for the empirical conditional invariant penalty used in the finite-sample IW-CIP by
\begin{align}\label{eqn:iwcip_penalty_defn}
    \penaltyiwcip{\phi} &\defn \frac{\lamiwcip}{LM^2}\sum_{y=1}^L\sum_{m\neq m'}  \probdiv{\Phcondsrc{\phi(X)}{Y=y}}{\Phcondsrcmp{\phi(X)}{Y=y}}.
\end{align}

{The multi-step procedure of IW-CIP raises a natural question of whether we can iterate the estimation of CIP and reweighting multiple times, similar to the EM algorithm~\citep{dempster1977maximum}. However, without target labels, it is unclear whether this IW-CIP can be improved via reiteration. The primary purpose of IW-CIP is to maintain the same level of conditional invariance as CIP while adjusting the classifier to accommodate target label shifts. At the end of \secref{control1}, we explain why this approach based on iteration does not work theoretically. }

\subsection{Target risk upper bounds of IW-CIP}\label{sec:target_risk_bound_iwcip}
In this subsection, we state our main theorems on the target risk upper bounds for IW-CIP. In a nutshell, the target risk bound can be decomposed into multiple terms involving the source risk or the optimal target risk, error in importance weights estimation, and the deviation from conditional invariance of the finite-sample CIP or IW-CIP feature mapping. {Consequently, when the importance weights are accurately estimated and the identified features are near conditional invariance, IW-CIP achieves high target accuracy.}

To simplify the theoretical analysis that follows, our finite-sample results are stated by considering the case where the whole source dataset is split into three parts of equal size. That is, the $\ell$-th ($\ell=1,2,3$) source dataset $\dataset_\ell=\{\dataset_\ell^{(1)},\dataset_\ell^{(2)},\ldots,\dataset_\ell^{(M)}\}$ is denoted by
\begin{align}\label{eqn:dataset_split}
	\dataset^\tagm{m}_{\ell} = \{(\Xsrc_{\ell,k}, \Ysrc_{\ell,k}) \}_{k=1}^{\nsrc} \hspace{.5em} \text{ for $m\in\{1,\ldots,M\}$}.
\end{align}
The number of samples for each source dataset is given by $\nsrc$. When it is clear which dataset we are referring to, we simply omit the dataset subscript $\ell$ in covariates and labels by writing $\Xsrc_k,\Ysrc_k$. In our finite-sample theory, the first $\dataset_1$ is used for solving the finite-sample CIP, the second $\dataset_2$ is used for estimating importance weights and correcting the label shift, and the last $\dataset_3$ is used for solving the finite-sample IW-CIP. We do not apply a similar split to the unlabeled target dataset, because it is only used ``once'' for estimating the importance weights through solving the linear system Eq.~\eqnref{label_shift_equation_matrix}.

Given importance weights $\wwmi{m}\in\real^L$ $(1\leq m \leq M)$ for each source distribution, we write $\ww=(\wwmi{1},\ldots,\wwmi{M})\in\R^{L\times M}$ and introduce the following shorthand for the average weighted source risk across $M$ environments,
\begin{align*}
      \riskbar{h }{\ww } \defn \frac{1}{M}\sum_{m=1}^M \risksrcw{h }{\wwm }.
\end{align*}
Before we introduce our main theorems, we define a key quantity called \textit{deviation from conditional invariance} of a feature mapping, which measures how conditionally invariant it is across source and target distributions.
\begin{definition}
    \name{(Deviation from Conditional Invariance)} Recall the $\Gset$-divergence given in Eq.~\eqnref{g_div}. For any feature mapping $\phi:\real^p\rightarrow\real^q$, we define its deviation from conditional invariance as
\begin{align}\label{eqn:Psi_phi_defn}
\Pssi{\phi} \defn \max_{\substack{m=1,\ldots,M, \\ y=1,\ldots,L}} \Gdiv{\Pcondtar{\phi(X)}{Y=y}}{\Pcondsrc{\phi(X)}{Y=y} }.
\end{align}
\end{definition}
The deviation from conditional invariance is defined via the maximal $\Gset$-divergence between any pair of conditionals $\phi(X) \mid Y = y$ in the source and target environments. When the feature representation is exactly conditionally invariant, this quantity attains its minimum value of zero. Our first theorem bounds the difference between target population risk and the average source population risk of a classifier, where $\wm$, $m=1,\ldots,M$, are the true importance weights introduced in Eq.~\eqnref{true_w}.
\begin{subtheorem}
  \label{thm:IW-CIP_gen_bound}
  For any classifier $h=g\circ\phi$ where $g\in\Gset, \phi\in\Phi$ and any estimated importance weights $\wh=(\whmi{1}, \ldots, \whmi{M})\in\real^{L\times M}$, the following target risk bound holds:
  \begin{align*}
    \risktar{h} \leq  \riskbar{h}{\wh} + \frac{1}{M}\sum_{m=1}^{M}\vecnorm{\wm - \whm}{\infty} + \Pssi{\phi}.
  \end{align*}
\end{subtheorem}

The proof of this theorem is given in \appref{IW-CIP_gen_hound}. {We observe that IW-CIP is designed to minimize the upper bounds in~\thmref{IW-CIP_gen_bound}}. Specifically, it is expected that for the case of IW-CIP, the estimated importance weights are close to the true importance weights and $\phi$ is close to the conditionally invariant feature mapping. Then the average weighted source risk can closely approximate the target risk, and \thmref{IW-CIP_gen_bound} allows us to establish an upper bound for the target risk of IW-CIP via the average weighted source risk.


To provide more specific risk guarantees for the finite-sample IW-CIP, we establish the following bound on the target risk of the finite-sample IW-CIP via the target risk of the optimal conditionally invariant classifier $\hstar$, as defined in~\defref{optimal_cie}. Here, $\penaltyiwcip{\phi}$ denotes the empirical IW-CIP penalty given in Eq.~\eqnref{iwcip_penalty_defn} and $\wm$'s are the true importance weights given in Eq.~\eqnref{true_w}.

\begin{subtheorem}\label{thm:IW-CIP_gen_bound2}
	Let $\wh=(\whmi{1},\ldots,\whmi{M})\in\real^{L\times M}$ be the estimated importance weights. Then, for any $\delta\in(0,1)$, with probability at least $1-\delta$, the following target risk bound holds for the finite-sample IW-CIP,\footnote{The probability is with respect to the randomness of source samples $(\Xsrc_k, \Ysrc_k)\iidsim\Psrc$ in $\dataset_3$ ($1\leq m\leq M, 1\leq k\leq\nsrc$); see Eq.~\eqnref{dataset_split}.} 
    \begin{align*}
        \risktar{\hhiwcip} &\leq \risktar{\hstar} + \penaltyiwcip{\phistar} + \frac{2}{M}\sum_{m=1}^{M}\vecnorm{\wm - \whm}{\infty}+ \Pssi{\phihiwcip} + \sampcomp,
    \end{align*}
    where\footnote{The sample complexity term $\sampcomp$ depends on $\delta, \Gset,\Phi, \wm, \nsrc$, and $\Psrc$ for $m=1,\ldots,M$.}
    \begin{align*}
        \sampcomp&= \max_{m=1,\ldots,M} \brackets{4\rad{\nsrc,\Psrc}{\Hsetm} + 2\vecnorm{\wm}{\infty}\sqrt{\frac{2\log(M/\delta)}{\nsrc}} },
    \end{align*}
    and $\Hsetm\defn\braces{f(x,y)=\wmy\ones{g(\phi(x))\neq y}: g\in\Gset,\phi\in\Phi}$.
\end{subtheorem}

{The proof proceeds by decomposing the target risk of IW-CIP into multiple components and bounding each term individually, which is given in~\appref{IW-CIP_gen_hound2}.} According to~\thmref{IW-CIP_gen_bound2}, the target risk of $\hhiwcip$ is bounded by that of the optimal conditionally invariant target classifier $\hstar$ with additional error terms: the empirical IW-CIP penalty of $\phistar$, the importance weight estimation error, the deviation from conditional invariance of $\phihiwcip$, and the sample complexity term $\sampcomp$. The empirical IW-CIP penalty term $\penaltyiwcip{\phistar}$ measures the conditional invariance of $\phistar(X)$ across empirical source environments. Because $\phistar$ is a conditionally invariant feature mapping, this term is expected to decrease in large sample scenarios. Similarly, the sample complexity term $\sampcomp$, which is based on Rademacher complexity, also diminishes as the sample size $\nsrc$ increases in the source environments. In this case, the theorem shows that accurate estimation of importance weights and minimal deviation from conditional invariance of feature representations of IW-CIP can guarantee IW-CIP to achieve a target risk similar to that of $\hstar$.

\thmref{IW-CIP_gen_bound2} provides the target risk bound of the finite sample IW-CIP in the most generic settings. In the following subsections, we demonstrate how the remaining terms can be controlled with additional assumptions. Specifically, in~\secref{control1}, we establish refined bounds for the empirical IW-CIP penalty by constraining the choice of IW-CIP penalty (see~\propref{iwcip_phistar}). Then we bound the weight estimation error using the deviation from conditional invariance of $\phihcip$ (see~\propref{label_shift_estimation_error}). In~\secref{control2}, under the general anticausal model~\ref{def:model1}, we bound the deviation from conditional invariance of feature mappings via a form of conditional invariant penalty (see~\propref{linear_sem_example}), and bound the target risk of $\hstar$ relative to the oracle classifier $\horacle$ (see~\propref{linear_sem_example2}).

\subsubsection{Upper bounds on the empirical IW-CIP penalty and weight estimation error}\label{sec:control1}

To refine the target risk upper bounds established in~\thmref{IW-CIP_gen_bound2}, we present two propositions. The first proposition shows that the empirical IW-CIP penalty diminishes to zero as the source sample size grows to infinity. The second proposition shows that the weight estimation error can be controlled through the deviation from conditional invariance $\Psi$---{this result is intuitive because Eq.~\eqnref{label_shift_equation_matrix} for weight estimation relies on the conditional invariance of CICs.}
Recalling that $\phistar$ is the feature mapping for the optimal conditionally invariant target classifier given in~\defref{optimal_cie}, we present the following proposition regarding the empirical IW-CIP penalty term $\penaltyiwcip{\phistar}$ in Eq.~\eqnref{iwcip_penalty_defn}.

\begin{proposition}\label{prop:iwcip_phistar}
Assume that the $\Gset$-divergence in Eq.~\eqnref{g_div} is used as the distributional distance in the empirical IW-CIP penalty. Then for any $\delta\in (0,1)$, with probability at least $1-\delta$, the following bound holds, \footnote{The probability is with respect to the randomness of source samples $(\Xsrc_k, \Ysrc_k)\iidsim\Psrc$ in $\dataset_3$ ($1\leq m\leq M, 1\leq k\leq\nsrc$); see Eq.~\eqnref{dataset_split}.}
\begin{align*}
    \penaltyiwcip{\phistar}\leq 2\lamiwcip \parenth{2\max_{\substack{m\in\{1,\ldots,M\}\\y, y'\in\{1,\ldots,L\}}}\rad{\nsrc, \Pcondsrc{\phistar(X)}{Y=y}}{\Gsetyp} + \max_{m=1,2,\ldots,M}\sqrt{\frac{\log \parenth{2LM/\delta}}{2\nsrc}}},
\end{align*}
where $\Gsety \defn \braces{f(z)=\ones{g(z)=y}, g\in\Gset}$.
\end{proposition}

This proposition is proved in~\appref{iwcip_phistar}. 
The bound on $\penaltyiwcip{\phistar}$ is given by the sum of a Rademacher complexity term and a finite-sample error term, both of which diminish to zero as the sample size grows to infinity{---this result is expected given that $\phistar$ is conditionally invariant across the population source distributions.} While calculating the exact $\Gset$-divergence may be challenging in practice, this result offers a vanishing bound without additional assumptions on the underlying data generation model or on the class of feature mappings $\Phi$. A more practical and simpler choice of the distributional distance is the squared mean distance, which penalizes the squared difference of conditional means of $\phi(X)\mid Y$ between source distributions. In this case, a refined bound of $\penaltyiwcip{\phistar}$ can be obtained with additional assumptions about the data generation model and the class of feature mappings $\Phi$. For more details on the calculation of this bound, see~\appref{iwcip_phistar2}.

Next, we show that when the true importance weights $\wm$ in Eq.~\eqnref{true_w} are estimated using the finite-sample CIP, the weight estimation error can be upper bounded via the deviation from conditional invariance of the feature mapping  $\phihcip$ learned through CIP. Let $\whm$ denote the estimated importance weight obtained by solving the linear system in Eq.~\eqnref{label_shift_equation_matrix}, where the population quantitites are replaced by their empirical estimates and the finite-sample CIP is used. Then the following proposition provides the upper bound on the estimation error $\vecnorm{\wm-\whm}{2}$.
\begin{proposition}
    \label{prop:label_shift_estimation_error}
    Assume that the confusion matrix in the $m$-th source distribution $C_{\hhcip}^\tagm{m}$, given in Eq.~\eqnref{confusion_matrix}, is invertible with conditional number $\condi_m$. Then, for any $\delta\in (0,1)$, with probability at least $1-\delta$, the error of importance weights estimation is bounded by\footnote{The probability is with respect to the randomness of $(\Xsrc_k, \Ysrc_k)\iidsim\Psrc$ and $(\Xtar_k, \Ytar_k)\iidsim\Ptar$ in $\dataset_2$; see Eq.~\eqnref{dataset_split}.} 
    \begin{align*}
        \vecnorm{\wm-\whm}{2}&\leq 2\condi_m\parenth{\sqrt{L}\Pssi{\phihcip} + \sqrt{\frac{L\log (4L/\delta)}{2\nttar }} + \sqrt{\frac{3\log (4L/\delta) }{\ntsrc}}\vecnorm{\wm}{2}},
    \end{align*}
	as long as $\ntsrc\geq 12\condi_m^2\log (4L/\delta)$.
\end{proposition}

The proof of this proposition is given in \appref{label_shift_estimation_error}. \propref{label_shift_estimation_error} reveals that the finite-sample estimation error of importance weights is determined by the condition number of the confusion matrix of $\hhcip$, the deviation from conditional invariance of  $\phihcip$, and sample error terms decaying roughly on the order of $1/\sqrt{\ntsrc}$ or $1/\sqrt{\nttar}$. The condition number $\condi_m$ reflects the performance of the CIP classifier---if $\hhcip$ achieves perfect classification accuracy on the $m$-th source distribution, the condition number $\condi_m$ takes the value of $1/\min_{y=1,2,\ldots,L}\prob{\Ysrc=y}$; however, if $\hhcip$ performs poorly on the $m$-th source distribution, the condition number $\condi_m$ can be large, which can lead to inaccurate estimation of importance weights. The bound also introduces additional deviation from invariance term $\Pssi{\phihcip}$, similar to $\Pssi{\phihiwcip}$ that appeared in~\thmref{IW-CIP_gen_bound2}.

{\propref{label_shift_estimation_error} also implies why iterating the two-step procedure of IW-CIP may not be beneficial. The estimation error of importance weights depends on the condition number $\kappa_m$ of the matrix $C^{(m)}_{\widehat{h}_{\text{CIP}}}$ and the deviation from conditional invariance $\Psi_{\mathcal{G},\widehat{\phi}_{\text{CIP}}}$. If we obtain IW-CIP in the first iteration and use it to update the importance weights, then it is uncertain whether we would see further gains in both $\kappa_m$ and $\Psi_{\mathcal{G},\widehat{\phi}_{\text{CIP}}}$. For $\kappa_m$, the accuracy of IW-CIP would be lower than that of CIP on the “source” domain, which can lead to a deterioration in $\kappa_m$; for $\Psi_{\mathcal{G},\widehat{\phi}_{\text{CIP}}}$, it is not feasible to assume that IW-CIP improves conditional invariance compared to CIP, given that both are subject to the same penalty during training.}

\subsubsection{Deviation from conditional invariance and target risk bound on the optimal conditionally invariant classifier}\label{sec:control2}

With~\propref{iwcip_phistar} and~\ref{prop:label_shift_estimation_error} now established, it remains to control the deviation from conditional invariance for both $\phihcip$ and $\phihiwcip$ and the target risk of the optimal conditionally invariant target classifier $\hstar$ in~\thmref{IW-CIP_gen_bound2}. In this subsection, we quantify both of these terms under additional assumptions about the data generation model. By quantifying these terms, we gain a comprehensive understanding of the upper bounds presented in Theorem~\ref{thm:IW-CIP_gen_bound2}.

To facilitate our analysis, we focus on the general anticausal model as defined in \defref{model1} and introduce the following assumption regarding the type of perturbations on the mechanism functions.

\begin{assumption}\label{assump:linear_cic}
Suppose that source and target data are generated under the general anticausal model in \defref{model1}. Assume that $\fsrc, \ftar$ are perturbed linearly as follows. For each $y\in\braces{1,\ldots, L}$, there exist an orthogonal matrix $P_y \in \real^{p\times d_y}$, $0\leq d_y\leq p$, and vectors $\vm{m}{y}, \vtar{y}\in\real^{d_y}$, $m=1,\ldots,M$ with $\vm{1}{y}=0$, such that
\begin{align}\label{eqn:model1_interventions}
\begin{split}
    \fsrc(y)&=\fsrcone(y) + P_y \vm{m}{y}, \text{ for all } m\in\braces{1,2,\ldots,M}, \text{ and }\\
    \ftar(y)&=\fsrcone(y) + P_y \vtar{y}.
\end{split}
\end{align}
In addition, the noise terms follow a normal distribution $\noisesrc,\noisetar\iidsim\Peps=\Normal(0, \Sigma)$.
\end{assumption}

Note that the perturbation matrix $P_y$ is dependent only on the label $y$, and remains fixed across source and target data generation processes, whereas the vectors $\vm{m}{y}, \vtar{y}\in\real^{d_y}$ are allowed to vary with both the source index $m$ and the labels $y$. The assumption states that, conditional on the labels, the perturbation in each environment lies within the low-dimensional space spanned by the columns of $P_y$. Consequently, for a classifier to generalize to the target data, it is important that the classifier only utilizes the covariates that remain invariant, i.e., those which are orthogonal to the columns of $P_y$.

Before we state our result, we introduce several population terms related to a feature mapping $\phi$. Let $\penaltycipm{m}{y}$ denote the difference of expected mean of $\phi(X)$ conditional on $Y=y$ between the first and the $m$-th source distributions, and let $\varcip$ denote the conditional covariance matrix of $\phi$ under the $m$-th source distribution, i.e.,
\begin{align}
    \begin{split}\label{eqn:varcip_defn}
        \penaltycipm{m}{y} &\defn \EEst{\phi(\Xsrc)}{\Ysrc=y}-\EEst{\phi(\Xsrcone)}{\Ysrcone=y}, \text{ and }\\
        \varcip &\defn \VVst{\phi(\Xsrc)}{\Ysrc=y}. 
    \end{split}
\end{align}
Assuming that $\varcip$ is invertible, we further define 
\begin{align}
    \begin{split}\label{eqn:pihat_defn}
        \Pihatphi{y} &\defn \frac{1}{M-1}\sum_{m=2}^M {\penaltycipm{m}{y}}^{\top}\varcip^{-1}\penaltycipm{m}{y}.
    \end{split}
\end{align}
$\Pihatphi{y}$ evaluates the differences of conditional means in the direction of the eigenvector of the covariance matrix. With these notions in hand, we can now establish the following deterministic bound for $\Pssi{\phi}$.
\begin{proposition}\label{prop:linear_sem_example}
    Suppose that source and target data are generated under~\assumpref{linear_cic}. Let $\Phi=\{\phi(x)=Ax+b, A\in\real^{q\times p}, b\in\real^{q}\}$ be the linear class of feature mappings from $\real^p$ to $\real^q$ ($q\leq p$), and let $\Gset$ be any hypothesis class mapping $\real^q$ to $\Yset$.
    If for each $y\in\braces{1,2,\ldots,L}$, there exists $\kk_y>0$ such that for all $ m\in\braces{1,2,\cdots, M}$,
    \begin{align}\label{eqn:linear_sem_prop_finite_assump}
    \parenth{\vtar{y}-\vm{m}{y}}\parenth{\vtar{y}-\vm{m}{y}}^{\top}\preceq\frac {\kk_y} {M-1}\sum_{m=2}^M \vm{m}{y}(\vm{m}{y})^{\top},
    \end{align}
    then for any $\phi\in\Phi$ such that $A\Sigma A^\top$ is non-singular, its deviation from conditional invariance~\eqnref{Psi_phi_defn} satisfies
    \begin{align*}
        \Pssi{\phi}^2&\leq2\max_{y=1,2,\ldots,L}\bigg\{\kk_y\Pihatphi{y}\bigg\}.
    \end{align*}
\end{proposition}

The proof of this proposition is given in \appref{linear_sem_example}. {The proof proceeds by connecting the $\Gset$-divergence with the total variation distance and applying Pinsker's inequality and data processing inequality}. \propref{linear_sem_example} shows that under the general anticausal model with a specific form of linear perturbations, the deviation from conditional invariance $\Pssi{\phi}$ is governed by $\kk_y$ and $\Pihatphi{y}$. The term $\kk_y$ captures perturbations in the underlying target data generation model, which are beyond our control. Therefore, to obtain good conditionally invariant representations, it is important to make the term $\Pihatphi{y}$ small. Because $\Pihatphi{y}$ measures the discrepancy in conditional means across source distributions, under~\assumpref{linear_cic}, the squared mean distance can be utilized as a penalty in both CIP and IW-CIP.

Concretely, in~\propref{linear_sem_example}, we consider a setting where $q=1$, i.e., $A$ is a row vector, and assume that the covariance matrix of the noise terms is $\Sigma=\sigma^2\ident_p$ for some $\sigma^2>0$, in the SCM under~\assumpref{linear_cic}. Then each term in $\Pihatphi{y}$ corresponds to $$\vecnorm{\penaltycipm{m}{y}}{2}^2/(\sigma^2\cdot \vecnorm{A}{2}^2)=\left(\penaltycipm{m}{y}\right)^2/(\sigma^2\cdot \vecnorm{A}{2}^2).$$ If we assume $\vecnorm{A}{2}=1$, a reasonable assumption given that the scale of classifier $g\in\Gset$ can be adjusted accordingly, then $\Pihatphi{y}$ is the squared difference in conditional means between source environments, which is precisely what CIP  penalizes when it uses mean distance. In cases where the CIP algorithm uses a different distributional distance, such as the MMD, $\Pihatphi{y}$ will implicitly decreases as the differences in conditional distributions between source environments reduce.



Condition~\eqnref{linear_sem_prop_finite_assump} is necessary for the validity of \propref{linear_sem_example}. In practice, verifying the correctness of this condition may be challenging. However, we can show that the condition is satisfied with high probability when the perturbations follow a Gaussian distribution and a sufficient number of source distributions are present. Specifically, in~\lemref{linear_sem_example}, we establish that approximately $M=\O{\max_y\{d_y\}}$ many source domains are required to ensure the validity of condition~\eqnref{linear_sem_prop_finite_assump} with high probability (See~\appref{linear_sem_example_cor} for more details on the precise statement of~\lemref{linear_sem_example} and its proof). This indicates that the number of source domains needs to be at least linear with respect to the dimension of perturbations.

Applying \propref{iwcip_phistar},~\ref{prop:label_shift_estimation_error} and~\ref{prop:linear_sem_example} to \thmref{IW-CIP_gen_bound2}, we know that IW-CIP has a guaranteed target risk compared to the optimal conditionally invariant classifier $\hstar$. The last goal of this subsection is to connect $\hstar$ with the oracle target classifier $\horacle$ defined in Eq.~\eqnref{oracle_risk}. Again, we consider the general anticausal model under \assumpref{linear_cic}, but in a simpler binary classification setting given as follows.

\begin{assumption}\label{assump:linear_cic_special}
Suppose that source and target data are generated under \assumpref{linear_cic} with binary labels $\Yset=\{1,2\}$ and $\probtar{1}=\probtar{2}=1/2$. Assume that $\Sigma = \sigma^2\Ind_p, P_1=P_2=\pmat{\Ind_{d},\zeros_{d\times(p-d)}}^\top$, and $\fsrcm{1}(1)=-\fsrcm{1}(2)=\xi\cdot\sigma\one{p}$ for some $\xi>0$, where $d_1=d_2=d$. Assume further that the source perturbations $\{\vm{m}{\ell}\}_{1\leq \ell\leq 2, 1\leq m\leq M}$ span the whole space of $\real^d$ and the target perturbations follow $\vtar{1},\vtar{2}\iidsim\Normal(0,\tau^2\cdot \sigma^2\Ind_d)$ for some $\tau>0$.
\end{assumption}

The assumption illustrates a specific binary classification problem in domain adaptation, where only the first $d$ coordinates of the covariates are perturbed, with the remaining $p-d$ coordinates remaining invariant conditioned on the label. It also assumes that we have observed a sufficient number of perturbations in source domains, which span the entire space of the first $d$ dimensions, while perturbations in the target domains are allowed to vary according to a normal distribution. Under this data generation model, we prove that the risk of $\hstar$ is close to that of $\horacle$ when the dimension of CICs is substantial.

\begin{proposition}\label{prop:linear_sem_example2}
Consider the domain adaptation problem under \assumpref{linear_cic_special}. Let the hypothesis class of classifiers $\Hset$ be $\Gset\circ\Phi$, where $\Gset = \{g: \real\rightarrow \{1,2\}, g(x) = 1\cdot\ones{x \leq 0} + 2\cdot \ones{x > 0}\}$ consists of one fixed function, and $\Phi = \{\phi:\real^{p}\rightarrow \real, \phi(x) = \beta^\top x + \beta_0, \vecnorm{\beta}{2}=1, \beta  \in\real^p,\beta_0\in\real\}$ consists of linear feature mappings. There exists a constant $c_{\xi,\tau}>0$ such that for any $\delta \in(0,1)$, the risk difference between $\horacle$ and $\hstar$ is bounded by
    \begin{align}\label{eqn:special_oracle_risk_diff}
        \risktar{\hstar} - \risktar{\horacle} \leq c_{\xi,\tau}\parenth{\sqrt{d}+\sqrt{\log\parenth{ 1/\delta}}}\exp\parenth{-\frac{\xi^2(p-d)}8},
    \end{align}
with probability at least $1-\delta$.\footnote{The probability is with respect to the randomness of target perturbations $\vtar{1},\vtar{2}\iidsim\Normal(0,\tau^2\sigma^2\Ind_d)$.}
\end{proposition}


The proof of this proposition is given in \appref{linear_sem_example2}. As the dimension of covariates $p$ goes to infinity, we can see that the risk difference between $\hstar$ and $\horacle$ converges to zero at an exponential rate, provided that the dimension of perturbations $d$ is smaller than $\alpha p$ for some constant $\alpha<1$. This result shares similarities with those obtained in a regression setting in~\cite[Corollary 6]{chen2021domain}, where the authors demonstrate a polynomial decay rate for the risk difference. 


\begin{remark}
    We have established Propositions~\ref{prop:iwcip_phistar}, \ref{prop:label_shift_estimation_error}, \ref{prop:linear_sem_example}, \ref{prop:linear_sem_example2} to control terms appeared in \thmref{IW-CIP_gen_bound},~\ref{thm:IW-CIP_gen_bound2}. These propositions respectively bound the empirical IW-CIP penalty term, error of importance weights estimation, the deviation from conditional invariance, and the risk of the optimal conditionally invariant classifier. As a result, \thmref{IW-CIP_gen_bound2} now guarantees that finite-sample IW-CIP has a target risk close to that of the oracle classifier $\horacle$.
\end{remark}

\section{Enhancing DA with CICs: risk detection and improved DIP}\label{sec:DIP}

While the previous section demonstrates that IW-CIP achieves target predictive performance close to that of the optimal conditionally invariant classifier, classifiers based on CICs can often be conservative, and their performance may not always be optimal. This conservativeness arises because finding invariant features from multiple source domains requires discarding some features that could be useful for target prediction. This can be seen in~\propref{linear_sem_example2} where if the dimension of perturbations $d$ is large, and thus $M=\O{d}$ is large, there is a gap between the oracle target risk and the risk of conditionally optimal invariant classifier. Practitioners may seek alternative DA algorithms, such as DIP, which directly find invariant features between a single (or a few) source domains and the target domain. However, in the absence of target labels, these algorithms can fail completely without any indication of failure. Furthermore, it is crucial to improve these algorithms when only unlabeled target data is available.

In this section, we present two additional roles of CICs to effectively address these challenges. First, we investigate how CICs can be used to detect large target risks for other DA algorithms (second role). Second, we examine how CICs enhance the reliability of DIP by addressing its label-flipping issue (third role).
{The empirical study for the second role is presented in~\secref{SEM_label_flipping_exp} and~\secref{mnist_label_flipping_exp}, while the third rold is demonstrated throughout~\secref{exp}.}
For the purpose of this section, we assume the absence of label shift, meaning that the label distributions are invariant across both the source and target domains. If label shift is present, the procedure outlined in~\secref{CIP} can be employed to correct it by applying CIP and adjusting for the importance weights.

In practice, assessing the success or failure of a DA algorithm is challenging due to the unavailability of target labels. While some research has been devoted in establishing target risk guarantees of DIP in the limited settings, it is still considered a ``risky'' algorithm. DIP seeks invariant feature representations across source and target domains, enforcing only the marginal distribution of these feature representations to be invariant across these domains. However, having the marginal distribution be invariant is not sufficient to ensure conditional invariance. Consequently, DIP may only learn representations that maintain marginal invariance across source and target domains but fail to be conditionally invariant given the labels. Because DIP merely minimizes the source risk, those representations may entirely flip the prediction of labels when applied to the target data~\citep{wu2019domain,zhao2019learning,wu2020representation}; see~\figref{dip_vs_jdip} (a)(b) for the illustrative examples. Furthermore, without target labels, it is difficult to detect this potential label flipping of DIP.

\begin{figure}[htb]
    \centering
    \includegraphics[width=0.99\textwidth]{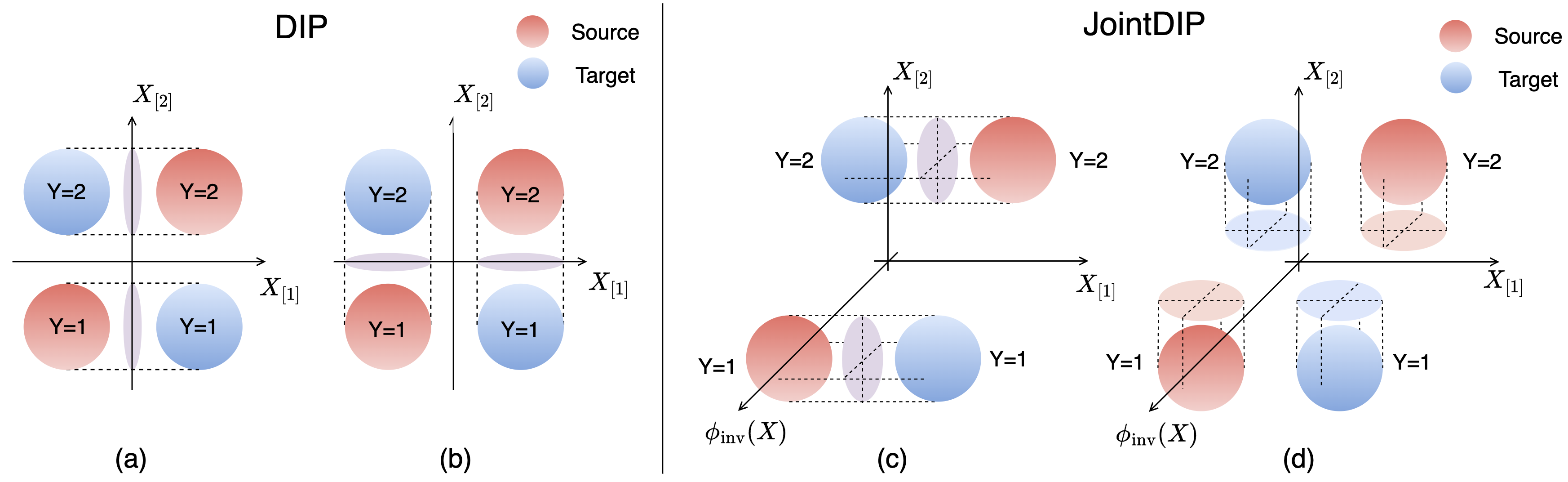}
    \caption{A binary classification example illustrating the difference between DIP and JointDIP. (a) DIP correctly matches the source and target covariates by projecting onto the feature $\coord{X}{2}$, which can generalize to the target distribution. (b) DIP matches the source and target covariates by projecting onto the label-flipping feature $\coord{X}{1}$, which cannot generalize to the target distribution. (c) JointDIP finds the correct feature $\coord{X}{2}$ by jointly matching the covariates with a conditional invariant component (CIC) $\phiinv(X)$. (d) JointDIP discards the label-flipping feature $\coord{X}{1}$ because the joint distribution of $(\coord{X}{1}, \phiinv(X))$ cannot be matched across source and target distributions.}
    \label{fig:dip_vs_jdip}
\end{figure}
The label-flipping issue in DIP raises concerns about its reliability when applied blindly without additional validation. In the absence of target labels, previous works often rely on assumptions about the data generation process in order for DIP-type of algorithms to achieve low target risk, e.g.~linear structural equation models~\citep{chen2021domain}. In this paper, given the availability of multiple source domains, we propose the use of CICs to enhance the reliability of DIP algorithms.

Suppose that we have obtained the conditionally invariant feature mapping $\phiinv$, and the classifier $\hinv$ built upon CICs denoted by
\begin{align}\label{eqn:gref_defn}
    \hinv = \ginv\circ \phiinv.
\end{align} 
As seen in the previous sections, both $\phiinv$ and $\hinv$ can be approximately obtained by solving CIP---in this case, while finite-sample feature mapping $\phihcip$ may not be perfectly conditionally invariant, we expect that they exhibit a small deviation from conditional invariance {under the appropriate assumptions, as implied by \propref{linear_sem_example}}.\footnote{{Because labeled data is available for multiple source domains, one can verify the conditional invariance of finite-sample CIP by comparing the conditional distributions across these source distributions. If one is willing to assume that CICs exist across source and target, then conditional invariance also holds for the target domain.}} Therefore, we assume the exact conditional invariance of $\phiinv$ for simplicity throughout this section. By leveraging $\phiinv$ and $\hinv$, we demonstrate that CICs can guide other DA algorithms in two significant ways:
\begin{enumerate}
    \item Large risk detection: when the classifier $\hinv$ built on CICs has low target risk, it can be used to lower-bound the risk of other DA algorithms, even  in the absence of target labels. This enables the identification of label-filpping issues in algorithms like DIP.
    \item Joint matching: we propose the JointDIP algorithm which uses CICs $\phiinv(X)$ to learn invariant features between source and target covariates. JointDIP, as a new DA algorithm, enhances the reliability of DIP by addressing the label-flipping issue often encountered in DIP when the CICs are generic enough.
\end{enumerate}
For the remainder of the section, when a single source is considered, we assume without loss of generality that the first source domain is used.

\subsection{Detect failed DA algorithms using CICs}\label{sec:detect_failure}

We present the second role of CICs in DA, specifically, in detecting whether another DA classifier has large target risk. Given the conditionally invarinat classifier $\hinv$ as defined in Eq.~\eqnref{gref_defn}, we prove the following theorem which controls the difference in source and target risks for any classifier.
\begin{theorem}
	\label{thm:conditional_invariant_proxy_y}
		Let $\hinv$ be a conditionally invariant classifier and assume that there is no label shift between source and target distributions. For any classifier $h$, its risk difference in source and target is controlled as follows, \footnote{In practice, it is difficult to find $\hinv$ that is conditionally invariant across source and target distributions, for instance, if we approximate it via CIP, i.e., $\hinv=\hhcip$. In this case, the upper bound has an additional term $L\Pssi{\phiinv}$; see the proof for a generalized version of the theorem.}
    \begin{multline}\label{eqn:conditional_invariant_proxy_y}
        \abss{\brackets{\risktar{h} - \risksrcone{h}} - \brackets{\prob{h(\Xtar) \neq \hinv(\Xtar)} - \prob{h(\Xsrcone) \neq \hinv(\Xsrcone)}  }} \\ \hspace{8em}\leq   2 \risksrcone{\hinv} .
    \end{multline}
\end{theorem} 
This result is proved in~\appref{conditional_invariant_proxy_y}. It shows that a large discrepancy between source and target risks of any classifier $h$ can be detected by examining the source risk of $\hinv$ and the alignments between the predictions of $h$ and $\hinv$ across source and target, which are always available. While~\thmref{conditional_invariant_proxy_y} is stated in terms of the first source domain, the result is applicable to any source domain and extends to multiple source domains via averaging. In practice, we may take $\hhcip$ as a proxy for $\hinv$ and choose the source domain which gives the tightest bound---specifically, the one that exhibits small discrepancy in prediction alignments between $h$ and $\hinv$ across source and target domains, and that shows high accuracy for CIP or IW-CIP algorithms.

By rearranging terms, an empirical version of Eq.~\eqnref{conditional_invariant_proxy_y} can be derived to establish a lower bound for the target risk 
\begin{multline}\label{eqn:conditional_invariant_proxy_y_lower_bound}
	\risktar{h}\geq \riskhsrcone{h} - 2\riskhsrcone{\hhcip} \\ + \parenth{\frac{1}{\ntar}\sum_{k=1}^{\ntar}\ones{h(\Xtar_k)\neq \hhcip(\Xtar_k)} - \frac{1}{\nsrcone}\sum_{k=1}^{\nsrcone}\ones{h(\Xsrcone_k)\neq \hhcip(\Xsrcone_k)}}.
\end{multline}

This lower bound on the target risk can be directly translated to an upper bound on accuracy, which can serve as a certificate to detect the failure of any DA classifier $h$. In~\secref{SEM_label_flipping_exp} and~\secref{mnist_label_flipping_exp}, we demonstrate this with numerical examples using CIP to detect the failure of DIP, where we test on both synthetic data and the MNIST dataset. We observe that the accuracy upper bound derived from Eq.~\eqnref{conditional_invariant_proxy_y_lower_bound} to detect the failures of DIP becomes more accurate when CIP has a higher predictive  accuracy on the source domain. If CIP does not perform well across the entire input space, by restricting to a subset of the space where CIP performs well (e.g. region far from decision boundary of a CIP classifier), we can obtain similar results which we will discuss in the subsequent subsection.

\subsubsection{Detect large target risk by restricting to a subset}

While~\thmref{conditional_invariant_proxy_y} provides a method for identifying when DA algorithms fail, the effectiveness of this approach may decrease if the invariant classifier $\hinv$, such as CIP, performs poorly in the source domain. In this case, it becomes difficult to detect the failure of DA algorithms across the entire region. However, if a subset of the region can be found where $\hinv$ achieves higher accuracy, we can still establish a more precise lower bound on the performance of DA algorithms within this restricted subset.

In particular, the result of~\thmref{conditional_invariant_proxy_y} can be extended in a straightforward manner to the case where we condition on the event $\Xsrcone\in \Aset$ for any subset $\Aset\subseteq \Xset$. We write $\risksrconesubset{\Aset}{h}$ to denote the source risk conditioned on $\Aset$, i.e., 
\[\risksrconesubset{\Aset}{h} \defn \probc{h(\Xsrcone)\neq \Ysrcone}{\Xsrcone\in \Aset}, \]
and similarly for the target risk $\risktarsubset{\Aset}{h}$. Analogous to~\thmref{conditional_invariant_proxy_y}, we can obtain the following bound on the risk difference conditioned on set $\Aset$.
\begin{corollary}
	\label{cor:conditional_invariant_proxy_y_A}
	Let $\hinv$ be a conditionally invariant classifier and assume that there is no label shift between source and target distributions. For any classifier $h$ and any subset $\Aset\subseteq\Xset$, we have
	\begin{multline}\label{eqn:conditional_invariant_proxy_y_A}
		 \abss{\risktarsubset{\Aset}{h} - \risksrconesubset{\Aset}{h}} \leq 2\risksrconesubset{\Aset}{\hinv}  \\+  \abss{ \probsubset{\Aset}{h(\Xsrcone) \neq \hinv(\Xsrcone)} - \probsubset{\Aset}{h(\Xtar) \neq \hinv(\Xtar)}},
	\end{multline}
	where for a random vector $X$, we write
	\[  \probsubset{\Aset}{h(X) \neq \hinv(X)} \defn \probc{h(X) \neq \hinv(X)}{X\in \Aset} .\]
\end{corollary}
From Eq.~\eqnref{conditional_invariant_proxy_y_A}, we can obtain a target risk lower bound for $h$, conditioned on set $\Aset$, analgous to Eq.~\eqnref{conditional_invariant_proxy_y_lower_bound}. If one can identify a region $\Aset\subseteq\Xset$ where the classifier $\hinv$ has near perfect source accuracy, due to its conditional invariance, the output of $\hinv$ can serve as a good proxy of the unobserved target labels. One example of such a choice is to choose $\Aset$ as the region where $\hinv$ gives a high predicted probability. Indeed, our experiments in~\secref{SEM_label_flipping_exp} and~\secref{mnist_label_flipping_exp} compare the tightness of the bounds across the entire space and on a subset $\Aset$ where $\hinv$ exhibits a high predicted probability. We find that \corref{conditional_invariant_proxy_y_A} provides a more accurate estimation of the target risk because $\hinv$ has higher accuracy when constrained to $\Aset$. It is important to note that while this approach does not provide insights on performance outside $\Aset$, it indicates that if certain DA algorithms fail within this subset, they may not be reliable to use overall. Conversely, if they perform well in $\Aset$, further investigation into these algorithms guided by domain knowledge is justified to evaluate their applicability to the broader region.

\corref{conditional_invariant_proxy_y_A} shows that a desired property for $h$ to have a low target risk is to align its output with the output of $\hinv$ on $\Aset$. We formalize this idea which is a direct consequence of~\corref{conditional_invariant_proxy_y_A}.

\begin{corollary}
	\label{cor:match_with_CIP_A}
    Let $\hinv$ be a conditionally invariant classifier and assume that there is no label shift between source and target distributions. Suppose that $h\in\Hset$ satisfies
	\begin{equation}\label{eqn:match_with_CIP_A}
		h(x)=\hinv(x) \text{ for all $x\in \Aset$}.
	\end{equation}
	Then the target risk of $h$ on $\Aset$ is bounded by
	\[\risktarsubset{\Aset}{h} \leq 3\risksrconesubset{\Aset}{\hinv}.\]
	In particular, if $\hinv$ has a near perfect accuracy on $\Aset$ under the source distribution, i.e., $\risksrconesubset{\Aset}{\hinv}\leq C$ for a small $C\geq 0$, then $\risktarsubset{\Aset}{h} \leq 3C$.
\end{corollary}
\corref{match_with_CIP_A} guarantees the low target population risk of the classifier on the region where the invariant classifier $\hinv$ is confident about its prediction and can serve as a good proxy for the target labels $Y$. This does \emph{not} necessarily imply that the representation learned via Eq.~\eqnref{match_with_CIP_A} can generalize to target data outside the region $\Aset$. Nevertheless, if domain experts all recognize the importance of the region $\Aset$, Eq.~\eqnref{match_with_CIP_A} becomes a natural requirement for assessing the quality of any DA classifier.

\subsection{JointDIP by matching DIP features jointly with CICs}

In this subsection, we demonstrate the third role of CICs in enhancing DIP. According to \thmref{conditional_invariant_proxy_y}, if the predictions of $h$ and $\hinv$ align well across source and target, the target risk of $h$ won't be too far from its source risk. It turns out that we can enforce this alignment by incorporating the CICs into the DIP matching penalty. This leads to our new algorithm, joint domain invariant projection (JointDIP).

\paragraph{Population JointDIP}\label{def:JointDIP}
	The population JointDIP minimizes the source risk while matching the joint distributions of $\phi$ and $\phiinv$ in the representation space across source and target:\footnote{$(\phi, \phiinv)(X)$ denotes the vector concatenated by $\phi(X)$ and $\phiinv(X)$.}
	\begin{align}\label{eqn:pop_JointDIP}
        \begin{split}
            \hjdip &= \gjdip \circ \phijdip, \\
            \gjdip, \phijdip &= \argmin_{g\in\Gset,\phi\in\Phi} \ \ \risksrcone{g\circ\phi}  \\
            &\hspace{0.2in} \text{subject to } \ \ \probdiv{\Pmarsrcone{(\phi,\phiinv)(X)}}{\Pmartar{(\phi,\phiinv)(X)}}  = 0.
	    \end{split}
    \end{align}
	Unlike DIP which only matches the marginal distribution of $\phi(X)$, JointDIP takes advantage of CICs to extract invariant feature representations. Intuitively, if $\phiinv$ is highly correlated with the labels $Y$, it can effectively serve as a surrogate for the unobserved target $Y$. Thus, matching the joint distribution of feature representations $(\phi(X),\phiinv(X))$ closely approximates matching the joint distribution of $(\phi(X), Y)$. This helps ensure that $\phi$ does not inadvertently learn features that flip the labels. Note that if label shift is present, it is not hard to extend the current formulation to an importance-weighted JointDIP (IW-JointDIP) by applying CIP to correct label shift prior to applying JointDIP.
\paragraph{Finite-sample JointDIP} The finite sample formulation of JointDIP enforces the joint invariance of the representations via a regularization term,
	\begin{align}\label{eqn:finite_JointDIP}
        \begin{split}
            \hhjdip &= \ghjdip \circ \phihjdip, \\
            \ghjdip, \phihjdip &= \argmin_{g\in\Gset,\phi\in\Phi} \ \ \riskhsrcone{g\circ\phi} + \lamjdip\cdot\probdiv{\Phmarsrcone{(\phi,\phiinv)(X)}}{\Phmartar{(\phi,\phiinv)(X)}},
	    \end{split}
    \end{align}
	where $\lamjdip>0$ is a regularization parameter that controls the strength of the joint matching penalty.


To illustrate the advantage of JointDIP over the ordinary DIP, consider a binary classification example under the general anticausal model~\defref{model1} where the data is generated as in~\figref{dip_vs_jdip} (a)(b), similar to the example in~\cite{johansson2019support}. Suppose that $\Phi=\{\phi_1(x)=\coord{x}{1}, \phi_2(x) = \coord{x}{2}\}$ and $\Gset= \{g_a(z)=\ones{z>a}+1, a\in\real\}$. Since DIP only matches the marginal distributions in the representation space, it could choose either $\phi_1$ or $\phi_2$ as the feature mappings to perfectly match the marginal  distributions of the features and achieve zero loss under the source distribution. However, choosing $\phi_1$ leads to zero accuracy under the target distribution, as $\phi_1(X)=\coord{X}{1}$ is the label-flipping feature. Only $\phi_2(X)=\coord{X}{2}$ is the conditionally invariant feature that can generalize to the target distribution. Now if we have access to CICs through the conditionally invariant feature mapping $\phiinv$, and we match these jointly with DIP, we would only get the correct feature $\phi_2(X)=\coord{X}{2}$, as illustrated in~\figref{dip_vs_jdip} (c)(d). JointDIP would never select $\phi_1(X)=\coord{X}{1}$ because the joint distribution of $(\coord{X}{1},\phiinv(X))$ is different between the source and target distributions.

\subsubsection{Theoretical comparison of CIP, DIP, and JointDIP under general  anticausal model}

To quantitatively compare the target risk of DA classifiers, we focus on data generated from the general anticausal model defined in~\defref{model1}. Additionally, we introduce~\assumpref{uniform_y} where marginal distribution of $Y$ is uniform under source and target distributions, as follows.

\begin{assumption}\label{assump:uniform_y}
Suppose that data is generated according to the general anticausal model defined in~\defref{model1}. Further, assume that the label distribution is uniform under source and target distributions, i.e.,$\forall m\in\{1,2,\ldots,M\}$ and $\forall y\in\{1,2,\ldots,L\}$, we have $\probsrc{y}=\probtar{y}=1/L$.
\end{assumption}

With the above assumptions on the data generation process in place, the following theorem compares the target risk of population CIP, DIP, and JointDIP.
\begin{theorem}\label{thm:SEM_guarantee}
	Suppose that source and target data are generated under~\assumpref{uniform_y}.
	Let $\Phi=\{\phi(x)=Ax+b, A\in\real^{q\times p}, b\in\real^q\}$ be the class of linear feature mapping from $\R^p$ to $\R^q$, which is used in the optimization of population DIP~\eqnref{pop_DIP} and population JointDIP~\eqnref{pop_JointDIP}. Assume that both algorithms match feature distributions exactly, i.e., \[\Pmarsrcone{\phidip(X)}=\Pmartar{\phidip(X)} \text{ and } \Pmarsrcone{(\phijdip,\phiinv)(X)}=\Pmartar{(\phijdip,\phiinv)(X)}.\]
	Then the following statements hold:
\begin{enumerate}
		\item[(a)] There exist distributions $\Psrcone,\Ptar$ such that the feature mapping of population DIP, $\phidip$, flips the labels after matching the marginal distributions in the representation space, i.e.,
		\begin{equation}\label{eqn:DIP_label_flip}
			\Pcondsrcone{\phidip(X)}{Y=y} = \Pcondtar{\phidip(X)}{Y=\pi(y)},  \ \ \forall y\in\{1,\ldots,L\},
		\end{equation}
		for some permutation $\pi\neq \ident$ over $\Yset$.
        \item[(b)] Suppose $\phiinv: \real^p\rightarrow \real^r$ is a linear conditionally invariant feature mapping such that
        \begin{align}\label{eqn:linear_cic_condition}
            \EEst{\phiinv(\Xsrcone)}{\Ysrcone=i}\neq \EEst{\phiinv(\Xsrcone)}{\Ysrcone=j}, \ \ \forall i\neq j\in\Yset.
        \end{align}
        Then the feature mapping of population JointDIP, $\phijdip$, is conditionally invariant across $\Psrcone$ and $\Ptar$. 
        If additionally $r\leq q$, then the target risk of JointDIP is no greater than that of the optimal classifier built on $\phiinv^0\defn\pmat{\phiinv,  \zeros_{q-r}}$, i.e.,
        \begin{align*}
            \risktar{\hjdip}\leq \min_{g\in\Gset}\risktar{g\circ \phiinv^0},
        \end{align*}
        and when $\phiinv=\phistar$, we have $\risktar{\hjdip}\leq \risktar{\hstar}$.
        \item[(c)] Suppose $\phiinv: \real^p\rightarrow \real^r$ is a conditionally invariant feature mapping such that the matrix
        \begin{align}\label{eqn:nonlinear_cic_condition}
            C_{\phiinv}(a) = \begin{pmatrix}
                1 & 1 & \cdots & 1\\
                m_{1}^{1}(a) & m_{2}^{1}(a) & \cdots & m_L^1(a)\\
                \vdots & \vdots & \ddots & \vdots\\
                m_1^{L-1}(a) & m_2^{L-1}(a) & \cdots & m_L^{L-1}(a)
            \end{pmatrix}\in\real^{L\times L},
        \end{align}
		is full rank for some vector $a \in\real^r$, where $m_j^{\ell}(a)=\EEst{\parenth{a^\top\phiinv(\Xsrcone)}^\ell}{\Ysrcone=j}$. 
        Then the feature mapping of population JointDIP, $\phijdip$, is conditionally invariant across $\Psrcone$ and $\Ptar$.
	\end{enumerate}
\end{theorem}

The proof of this result is given in~\appref{SEM_guarantee}. {The proofs for~\thmref{SEM_guarantee}(a) and~\thmref{SEM_guarantee}(b) rely on the matching property of the two mixing distributions, whereas the proof for~\thmref{SEM_guarantee}(c) analyzes distribution matching in the space of characteristic functions}. \thmref{SEM_guarantee}(a) shows that while DIP uses the additional target covariates information, the corresponding representations can fail to satisfy conditional invariance across source and target distributions; and in fact, if the features correspond to the label-flipping features, as illustrated in~\figref{dip_vs_jdip}, it can potentially hurt the target prediction performance.


By contrast, \thmref{SEM_guarantee}(b) shows that when $\phiinv$ is a linear conditionally invariant feature mapping, the features obtained by JointDIP are conditionally invariant across $\Psrcone$ and $\Ptar$, therefore avoiding the label-flipping issue of DIP. The condition~\eqnref{linear_cic_condition} requires that conditional means of $\phiinv$ are different for any pair of labels $i$ and $j$. This condition is a reasonable expectation for any $\phiinv$ exhibiting a good prediction performance as otherwise there is no way for classifiers built on $\phiinv$ to distinguish between labels $i$ and $j$. See~\appref{a_condition_b} for an example of the general anticausal model that satisfies the condition~\eqnref{linear_cic_condition}. Additionally, \thmref{SEM_guarantee}(b) assures that JointDIP cannot be worse than the optimal classifier built upon $\phiinv$. This result is expected because CICs tend to be conservative. For instance, CICs identified via CIP are forced to be conditionally invariant across $M$ many source distributions (and ideally the target distribution), which can potentially eliminate features that are useful for predicting target labels. JointDIP addresses this issue by seeking conditional invariant representations across a single source and target domains to construct a more effective classifier than one solely based on CICs.

Finally, for any conditionally invariant feature mapping $\phiinv$ which may not necessarily be linear, \thmref{SEM_guarantee}(c) guarantees that the features obtained by JointDIP are conditionally invariant across $\Psrcone$ and $\Ptar$, as long as the matrix $C_{\phiinv}(a)$ given in Eq.~\eqnref{nonlinear_cic_condition} is full rank. The matrix $C_{\phiinv}(a)$ is full rank if the CICs under different labels are in a generic position. Comparing with the condition~\eqnref{linear_cic_condition} in the linear case, this condition may be more difficult to verify generally as it requires computation of higher order moments. We provide concrete examples in~\appref{a_condition_c} where this condition is satisfied. In practice, we can either use $\phihcip$ as $\phiinv$, or refer to domain experts for suggestions on reasonable CICs.

\section{Numerical experiments}\label{sec:exp}


In this section, we investigate the target performance of our proposed DA algorithms and compare them with existing methods.\footnote{All code to reproduce experiments is available at \url{https://github.com/KeruWu/Roles_CICs}.} In particular, we demonstrate through our experiments the effectiveness of the importance-weights correction in the presence of both covariate and label distribution shifts, the capability of detecting DIP's failure using estimated CICs, and the superior performance of JointDIP over DIP when label-flipping features are present. We consider DA classification tasks across five datasets: synthetic data generated from linear Structural Causal Models (SCMs), the MNIST data~\citep{lecun1998mnist} under rotation intervention, the CelebA data~\citep{liu2015deep} under color intervention, the Camelyon17 data from WILDS~\citep{koh2021wilds, sagawa2021extending}, and the DomainNet data~\citep{peng2019moment} transformed into a binary classification problem. Except for the Camelyon17 and DomainNet data, the domain shifts in other datasets are synthetically created.

We conduct a thorough comparative analysis where our methods are benchmarked against various DA algorithms. In addition to DIP, CIP, IW-CIP, and JointDIP which have been introduced in previous sections, we additionally explore several variants of these algorithms, including \textit{IW-DIP} and \textit{IW-JointDIP}. IW-DIP is the algorithm that applies CIP for importance weighting to correct label shift prior to DIP, while IW-JointDIP applies this importance weighting step before JointDIP. Moreover, for ERM, DIP, and JointDIP, we consider both their single-source versions (e.g. DIP) and their multi-source versions (e.g. DIP-Pool). Contrary to the approach in Section~\ref{sec:target_risk_bound_iwcip}, where source datasets are split for theoretical analysis, our empirical findings indicate superior performance when using the entire dataset, and thus we do not apply such splits in our experiments. In terms of distributional distances, both squared mean distance and Maximum Mean Discrepancy (MMD) are considered. Lastly, we compare these methods with existing well-known DA algorithms such as Invariant Risk Minimization (IRM)~\citep{arjovsky2019invariant}, V-REx~\citep{krueger2021out}, and groupDRO~\citep{sagawa2019distributionally}.
A detailed description of each DA algorithm, models architectures, and the training setup are presented in Appendix~\ref{app:exp}.

Certain DA algorithms, such as DIP or IW-DIP, rely on a single source domain to learn invariant feature representations and construct the final classifiers (see Appendix~\ref{app:exp}). {In this work, we do not focus on how to choose the best single source domain for these algorithms, but instead simply select the last source domain.} By design, this single source domain is usually similar to the target domain in terms of the considered covariate shifts, such as mean shift (linear SCMs), rotation angle (MNIST), and color balance (CelebA). However, this domain may have label-flipping features compared to the target domain. {We deliberately design these settings to demonstrate that}, under such settings, DIP might induce label flipping, while JointDIP avoids this issue.

\subsection{Linear structural causal models (SCMs)} 
\label{sub:linear_structural_models}

We perform experiments on synthetic datasets generated according to linear SCMs under the general anticausal model defined in \defref{model1}. We first compare the performance of various DA methods, and then show how to use CICs to detect the failure of DIP without access to target labels. 

\subsubsection{Linear SCMs under different interventions}
We consider three different types of domain shifts: mean shift, label shift, and a shift to introduce label-flipping features. In all SCMs, we introduce the mean shift across domains. Depending on the presence or absence of label shift and the shift that introduces label-flipping features, we obtain four combinations and their corresponding SCMs. We use $M'=M+1$ to denote the total number of source and target domains. The last source domain (the $M$-th domain) is always set as the single source domain for DA algorithms which utilize only one source domain (e.g. DIP). The last domain (the $M'$-th domain) serves as the target domain, and we generate 1000 samples per domain. Instead of specifying each $\fsrc(\cdot)$ and $\ftar(\cdot)$ in \defref{model1}, we provide explicit representations of the data generation model for each SCM. For simplicity, we only present the data generation model for source data, and the target data generation model can be obtained by replacing the superscript $(m)$ with $\tagtar$.

\begin{itemize}
    \item SCM \rom{1}: mean shift exists; no CICs; no label shift; no label-flipping features; $M'=4$ and $p =  10$. The data generation model is
    \begin{align*}
        \Ysrc &\sim \text{Bernoulli}(0.5), \\
        \Xsrc &= 0.2Y \cdot \mathds{1}_{10} + 0.25 \cdot\Normal(0, \Ind_{10}) + \Asrc,
    \end{align*}
    where $\mathds{1}_{10}\in\R^{10}$ denotes a vector consisting of ones and $\Ind_{10}\in\R^{10\times 10}$ denotes the identity matrix. For all source domains, the mean shift $\Asrc$, $1\leq m \leq M$, is generated by $0.2\cdot\Normal(0, \Ind_{10})$, while the target domain suffers a large intervention $\Atar\sim 2\cdot\text{sign}\parenth{\Normal(0, \Ind_{10})}$. Note that we do not introduce any CICs in SCM \rom{1}. This allows us to examine the most extreme case of mean shift where all coordinates of $X$ are perturbed. However, in the following three SCMs (SCM \rom{2}, \rom{3}, and \rom{4}), we ensure the presence of CICs and expect that certain DA algorithms can exploit these CICs.

    \item SCM \rom{2}: mean shift exists; CICs exist; label shift exists; no label-flipping features; $M'=12$ and $p =  9$. The data generation model is
    \begin{align*}
        \Ysrc &\sim \text{Bernoulli}(p^{(m)}), \\
        \coordw{\Xsrc}{1:6} &= 0.2(\Ysrc-0.5) \cdot \mathds{1}_6 + 0.25 \cdot\Normal(0, \Ind_6) + \Asrc,\\
        \coordw{\Xsrc}{7:9} &= 0.2(\Ysrc-0.5) \cdot \mathds{1}_3 + 0.25 \cdot\Normal(0, \Ind_3),
    \end{align*}
    where the label distribution $p^{(m)}=0.5$, for $1\leq m\leq M,$ is balanced in source domains but perturbed in target domain with $p^{\tagtar}=0.1$. The mean shift only exists in the first six coordinates of $X$, where $\Asrc$, $1\leq m\leq M$ is generated by $\Normal(0, \Ind_6)$ in the source domains, while the target domain suffers a large intervention $\Atar\sim 2\cdot\text{sign}\parenth{\Normal(0, \Ind_6)}$. The last three coordinates of $X$ remain unperturbed, and they serve as CICs. 

    \item SCM \rom{3}: mean shift exists; CICs exist; no label shift; label-flipping features exist; $M'=12$ and $p =  18$. The data generation model is
    \begin{align*}
        \Ysrc &\sim \text{Bernoulli}(0.5), \\
        \coordw{\Xsrc}{1:6} &= 0.3(\Ysrc-0.5)\cdot \mathds{1}_6 +0.4\cdot\Normal(0, \Ind_6) + \Asrc,\\
        \coordw{\Xsrc}{7:12} &=\begin{cases}
            0.3(0.5-\Ysrc)\cdot \mathds{1}_6 + 0.1 \cdot \Normal(0, \Ind_6), &  $m$ \text{ is odd},\\
            0.3(\Ysrc-0.5)\cdot \mathds{1}_6 + 0.1 \cdot \Normal(0, \Ind_6), &  $m$ \text{ is even},
        \end{cases}\\
        \coordw{\Xsrc}{13:18} &= 0.3(\Ysrc-0.5)\cdot \mathds{1}_6 +0.4 \cdot \Normal(0, \Ind_6),
    \end{align*}
	The first six coordinates of $X$ suffer mean shift, with $\Asrc$, $1\leq m \leq M$ and $\Atar \sim \Normal(0, \Ind_d)$ across all domains. In particular, we make the last source domain and target domain share similar mean shift interventions: $\Asrcm{11}=a_0+a_1$ and $\Asrcm{12}=a_0+a_2$, where $a_0\sim 0.8\cdot\Normal(0,\Ind_6)$ and $a_1,a_2\sim 0.6\cdot\Normal(0,\Ind_6)$. This setting potentially enables DIP-based methods to better rely on the last source domain because of the similarity of $\coord{X}{1:6}$ between this domain and the target domain. However, the $7^{\text{th}}$ to $12^{\text{th}}$ coordinates of $X$ are label-flipping features according to \defref{label-flipping}: in the first six domains $\coord{X}{7:12}$ has positive correlation with $Y$, while the correlation is negative in the remaining six domains, including the target domain. There are no interventions on the last six coordinates of $X$, and they serve as CICs. 

    \item SCM \rom{4}: mean shift exists; CICs exist; label shift exists; label-flipping features exist, $M'=12$ and $p =  18$. The data generation model of $X\mid Y$ is the same as SCM \rom{3},
    but we additionally perturb the marginal distribution of $Y$ across source and target distributions. Specifically, the label distribution $\Ysrc\sim\text{Bernoulli}(0.5)$ is balanced in source domains but perturbed in target domain with $\Ytar\sim\text{Bernoulli}(0.3)$. 
\end{itemize}

\begin{table}[t]
	\centering
	\small
	\begin{tabular}{l | ll | ll}
		\toprule
		{} & \multicolumn{2}{c|}{SCM \rom{1}} & \multicolumn{2}{c}{SCM \rom{2}} \\\hline
		Mean shift & \multicolumn{2}{c|}{Y} & \multicolumn{2}{c}{Y} \\
		CICs & \multicolumn{2}{c|}{N} & \multicolumn{2}{c}{Y} \\
		Label shift & \multicolumn{2}{c|}{N} & \multicolumn{2}{c}{Y} \\
		Label flip & \multicolumn{2}{c|}{N} & \multicolumn{2}{c}{N} \\
		\hline
		DA Algorithm & src\_acc & tar\_acc & src\_acc & tar\_acc \\
		\hline
		Tar & 70.6$\pm$13.8 & \textbf{89.0$\pm$0.7} & 69.4$\pm$3.9 & \textbf{92.9$\pm$1.0} \\
		ERM & 89.6$\pm$1.0 & 56.1$\pm$12.0 & 87.5$\pm$1.0 & 57.4$\pm$33.3 \\
		ERM-Pool & 88.3$\pm$2.3 & 54.4$\pm$10.2 & 78.5$\pm$1.3 & 58.6$\pm$29.3 \\
		DIP & 88.3$\pm$1.0 & \textbf{\textcolor{red}{87.6$\pm$1.5}} & 84.5$\pm$3.1 & 62.0$\pm$2.9 \\
		DIP-Pool & 86.7$\pm$2.8 & \textbf{86.4$\pm$2.2} & 75.8$\pm$0.7 & 60.1$\pm$3.1 \\
		CIP & 87.4$\pm$3.3 & 55.9$\pm$12.0 & 75.2$\pm$0.7 & 75.7$\pm$6.5 \\
		IW-ERM & 54.8$\pm$11.6 & 52.4$\pm$10.4 & 59.3$\pm$9.6 & 54.1$\pm$37.7 \\
		IW-CIP & 53.7$\pm$9.9 & 54.0$\pm$11.8 & 50.3$\pm$0.7 & \textbf{90.4$\pm$0.8} \\
		IW-DIP & 56.2$\pm$12.8 & 54.3$\pm$11.7 & 71.1$\pm$10.8 & \textbf{\textcolor{red}{92.1$\pm$2.7}} \\
		JointDIP & 87.1$\pm$1.3 & \textbf{86.8$\pm$1.9} & 82.8$\pm$2.7 & 70.6$\pm$6.2 \\
		IW-JointDIP & 68.4$\pm$19.1 & 68.0$\pm$19.4 & 51.7$\pm$3.9 & \textbf{90.0$\pm$1.3} \\
		IRM & 87.6$\pm$2.2 & 56.7$\pm$10.4 & 70.9$\pm$2.5 & 71.9$\pm$17.9 \\
		V-REx & 87.3$\pm$2.5 & 55.6$\pm$11.5 & 77.9$\pm$1.2 & 62.9$\pm$25.7 \\
		groupDRO & 88.3$\pm$2.3 & 54.4$\pm$10.0 & 77.8$\pm$1.1 & 64.2$\pm$25.7 \\
		\bottomrule
	\end{tabular}
	\caption{Source and target accuracy in SCM \rom{1} and SCM \rom{2}. For each setting of SCMs, DA algorithms are run on 10 different datasets, each generated using 10 different random seeds. Tar represents the oracle method where the model is trained on the labeled target data. The top four methods are highlighted in bold. The best method (excluding Tar) is colored in red. }
	\label{tab:SEM_results_1_2}
\end{table}

\begin{table}[t]
	\centering
	\small
	\begin{tabular}{l | ll | ll}
		\toprule
		{} & \multicolumn{2}{c|}{SCM \rom{3}} & \multicolumn{2}{c}{SCM \rom{4}} \\\hline
		Mean shift & \multicolumn{2}{c|}{Y} & \multicolumn{2}{c}{Y} \\
		CICs & \multicolumn{2}{c|}{Y} & \multicolumn{2}{c}{Y} \\
		Label shift & \multicolumn{2}{c|}{N} & \multicolumn{2}{c}{Y} \\
		Label flip & \multicolumn{2}{c|}{Y} & \multicolumn{2}{c}{Y} \\
		\hline
		DA Algorithm & src\_acc & tar\_acc & src\_acc & tar\_acc \\
		\hline
		Tar & 70.1$\pm$2.9 & \textbf{97.9$\pm$0.4} & 69.8$\pm$3.0 & \textbf{96.9$\pm$0.7} \\
		ERM & 97.9$\pm$0.3 & 58.9$\pm$5.0 & 98.0$\pm$0.3 & 58.7$\pm$10.7 \\
		ERM-Pool & 83.3$\pm$0.8 & 75.3$\pm$7.7 & 83.3$\pm$0.7 & 78.9$\pm$6.8 \\
		DIP & 93.5$\pm$3.4 & 34.5$\pm$14.9 & 94.4$\pm$2.7 & 35.3$\pm$14.6 \\
		DIP-Pool & 84.1$\pm$0.5 & 82.0$\pm$1.1 & 84.4$\pm$0.5 & 82.3$\pm$3.8 \\
		CIP & 82.2$\pm$0.4 & \textbf{81.8$\pm$1.3} & 82.2$\pm$0.4 & 82.1$\pm$1.2 \\
		IW-ERM & 80.3$\pm$9.1 & 75.1$\pm$8.9 & 77.2$\pm$9.0 & 79.2$\pm$4.8 \\
		IW-CIP & 82.5$\pm$0.4 & 81.2$\pm$4.1 & 81.0$\pm$0.9 & \textbf{83.8$\pm$2.2} \\
		IW-DIP & 88.4$\pm$13.0 & 37.2$\pm$14.9 & 83.7$\pm$8.6 & 64.2$\pm$7.3 \\
		JointDIP & 88.8$\pm$1.3 & \textbf{\textcolor{red}{85.4$\pm$2.1}} & 88.6$\pm$1.9 & 82.8$\pm$1.9 \\
		IW-JointDIP & 87.8$\pm$1.5 & \textbf{82.9$\pm$8.1} & 84.1$\pm$3.4 & \textbf{\textcolor{red}{85.1$\pm$3.7}} \\
		IRM & 80.1$\pm$1.8 & 80.2$\pm$1.8 & 84.2$\pm$0.6 & 83.7$\pm$3.3 \\
		V-REx & 83.8$\pm$0.8 & 80.4$\pm$7.6 & 84.3$\pm$0.6 & 83.8$\pm$3.7 \\
		groupDRO & 84.0$\pm$0.7 & 81.3$\pm$7.4 & 83.9$\pm$0.6 & \textbf{84.3$\pm$3.2} \\
		\bottomrule
	\end{tabular}
	\caption{Source and target accuracy in SCM \rom{3} and SCM \rom{4}. For each setting of SCMs, DA algorithms are run on 10 different datasets, each generated using 10 different random seeds. Tar represents the oracle method where the model is trained on the labeled target data. The top four methods are highlighted in bold. The best method (excluding Tar) is colored in red.}
	\label{tab:SEM_results_3_4}
\end{table}

We use a linear model in all SCM experiments to predict the labels; see Appendix~\ref{app:architec} for details. Table~\ref{tab:SEM_results_1_2} and~\ref{tab:SEM_results_3_4} compare the target performance of different DA algorithms. In SCM~\rom{1} where only mean shift exists and no CICs exist, DIP gives the best performance, with JointDIP showing comparable accuracy. ERM and other DA algorithms which aim to find invariance across all source domains (e.g. CIP, IRM, V-REx) fail to generalize to the target domain due to the lack of CICs and the substantial mean shift in the target domain. In SCM~\rom{2} where mean shift exists and label shift is added as another intervention, IW-DIP achieves the highest accuracy. Both IW-CIP and IW-JointDIP also achieve over 90\% correct predictions. However, IW-ERM which directly applies importance weighting without using CICs completely fails, indicating that identifying conditionally invariant features before applying label correction is necessary in this scenario.

\begin{figure}[t]
    \centering
    \includegraphics[width=0.95\textwidth]{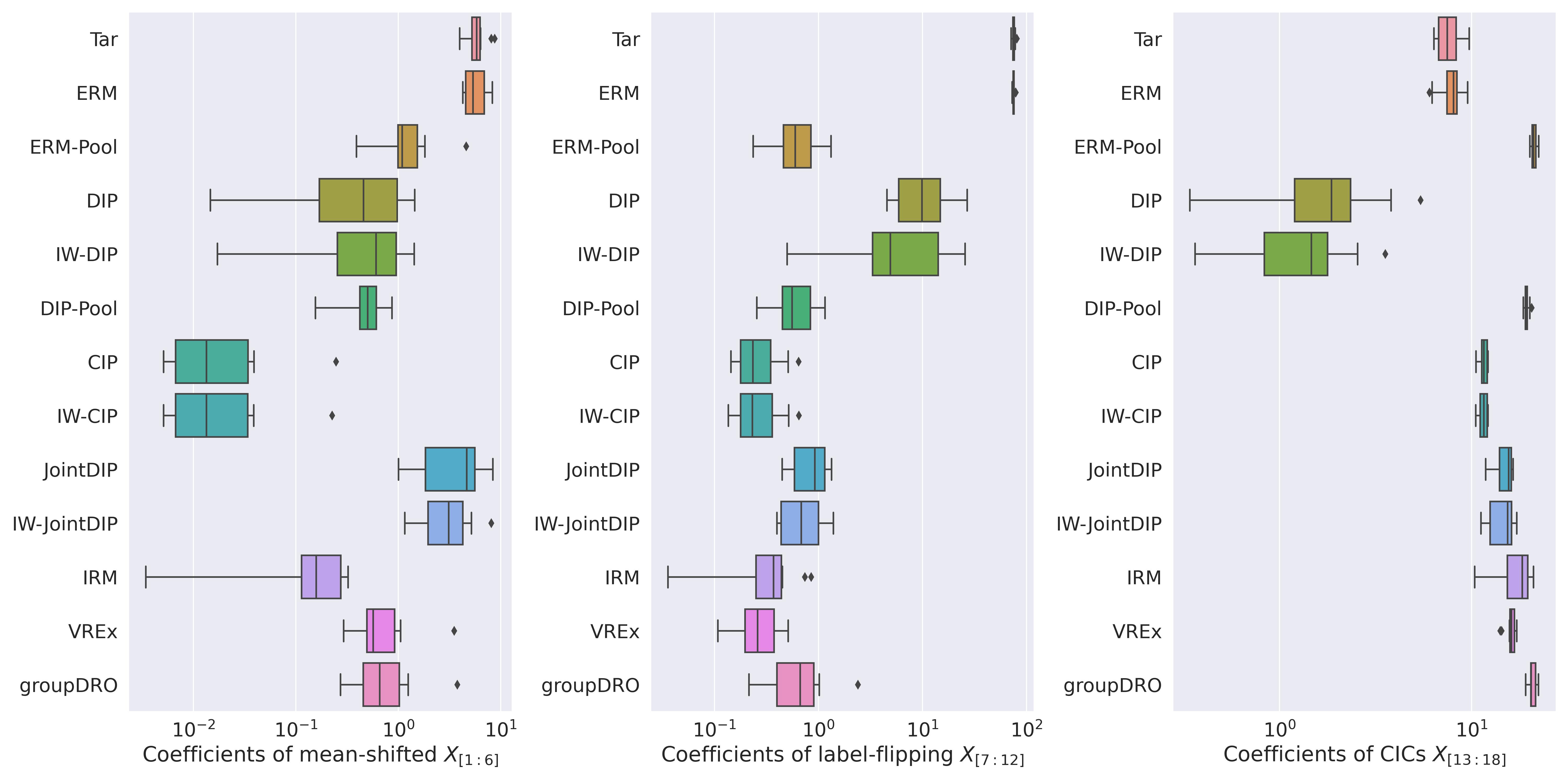}
    \caption{$L^1$ norm of coefficients in SCM \rom{3}. Coefficients are grouped into three categories, depending on how corresponding coordinates are perturbed. Contrary to DIP, JointDIP does not have large coefficients on label-flipping features, demonstrating that JointDIP can avoid label-flipping issue observed in DIP.}
    \label{fig:SCM_coef}
\end{figure}

In SCM \rom{3} where label-flipping features exists, we observe that DIP results in an accuracy lower than a random guess. On the contrary, JointDIP achieves the highest accuracy. To gain a deeper understanding of the classifiers obtained by each algorithm, in Figure~\ref{fig:SCM_coef} we illustrate the $L^1$ norm of the coefficients of three categories of coordinates: mean-shifted $\coord{X}{1:6}$, label-flipping $\coord{X}{7:12}$, and CICs $\coord{X}{13:18}$. Three key observations can be made from the figure: First, DIP shows large coefficients on the label-flipping features $\coord{X}{7:12}$, providing insight for its suboptimal performance. Second, DA methods that seek invariant features across all source domains, such as CIP, IRM, V-REx, show small coefficients on both the mean-shifted features and label-flipping features, but large coefficients on the CICs. This result aligns with their fundamental objective of identifying invariant representations across source domains. Lastly, JointDIP demonstrates small coefficients on label-flipping features, but relatively larger coefficients on mean-shifted coordinates, confirming that JointDIP effectively discards label-flipping features $\coord{X}{7:12}$, retains the invariant features $\coord{X}{13:18}$ identified by CIP while exploiting the useful mean-shifted features $\coord{X}{1:6}$.

Lastly in SCM \rom{4} where mean shift, label shift, and label flipping features exist simultaneously, we find that IW-JointDIP outperforms other methods in the target accuracy. This can be attributed to its three-step approach---first finding CICs, then correcting label shift via CICs, and finally aligning potential features via JointDIP---which effectively addresses the combination of these three shifts. 

{Overall, our experiments across various SCM settings reveal that while traditional DA methods like CIP and DIP may struggle with certain types of shifts, JointDIP excels in handling label-flipping features, and importance-weighted variants of the DA methods perform well in scenarios with label shift.}

\subsubsection{Detecting failure of DIP in linear SCMs}\label{sec:SEM_label_flipping_exp}

Previously in \secref{detect_failure}, we discussed the second role of CICs to detect the failure of DA algorithms without requiring access to target labels. {To demonstrate this second role in practice}, here we apply \thmref{conditional_invariant_proxy_y} and \corref{conditional_invariant_proxy_y_A} to DIP and JointDIP within the context of SCM \rom{3}. We take the finite-sample CIP as $\phiinv$, and compare the accuracy upper bound derived from~\thmref{conditional_invariant_proxy_y} against the actual accuracy of DIP and JointDIP as shown in Figure~\ref{fig:coro1} (a)(b). We vary the DIP penalty parameter in DIP and the JointDIP penalty parameter in JointDIP (both represented by $\lambda$). The CIP penalty parameter utilized in CIP and JointDIP is set to the optimal value found via hyperparameter search. The figure illustrates that while the accuracy upper bound from~\thmref{conditional_invariant_proxy_y} is valid, it exceeds the true accuracy by a wide margin, which is undesirable.\footnote{This issue might result from the relatively low accuracy of CIP (around $80\%$) in SCM \rom{3}. In our subsequent experiments with MNIST, such issue does not appear because CIP has much higher accuracy (over 90\%). See Figure~\ref{fig:coro1} (c)(d) for the comparison.} Consequently, it is difficult to directly apply~\thmref{conditional_invariant_proxy_y} to test the failure of DIP in SCM \rom{3}. 

\begin{figure}[t]
    \centering
    \includegraphics[width=\linewidth]{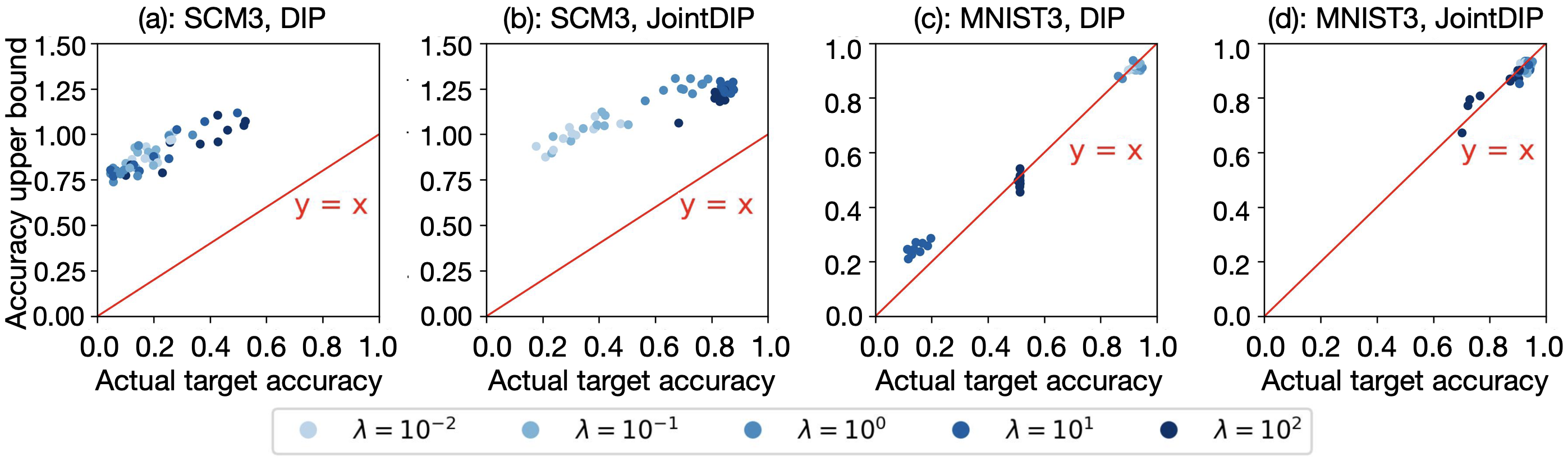}
    \caption{SCM \rom{3} and MNIST \rom{3}: target accuracy upper bound obtained via~\thmref{conditional_invariant_proxy_y} vs. actual target accuracy. For SCM \rom{3}, the accuracy of CIP is low, resulting in an upper bound that does not provide a tight estimation of the true target accuracy for both DIP and JointDIP. Conversely, in MNIST \rom{3}, CIP achieves high accuracy and the accuracy upper bound precisely reflects the true target accuracy. The red line indicates $y=x$. }
    \label{fig:coro1}
\end{figure}

\begin{figure}[t]
    \centering
    \includegraphics[width=\textwidth]{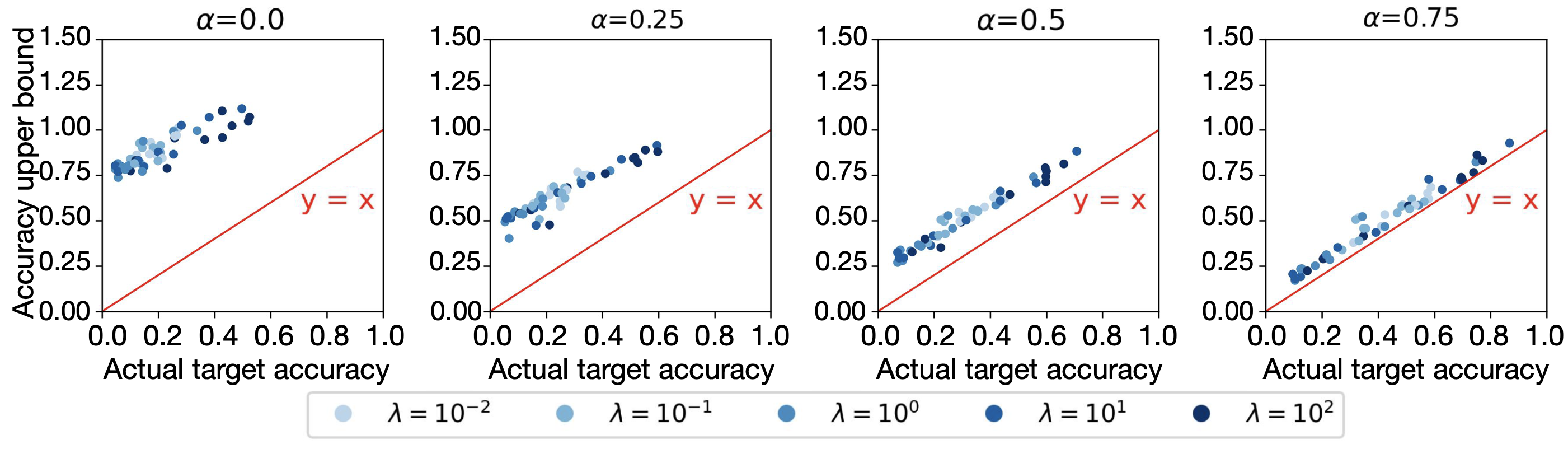}
    \caption{DIP in SCM \rom{3}: target accuracy upper bound obtained via~\corref{conditional_invariant_proxy_y_A} vs. actual target accuracy in region $\Aset$. We define region $\Aset$ as $\Aset=Q_\alpha\defn\{X\in\mathcal{X}\mid \text{CIP predicted probability} \geq q_{\alpha} \}$, where $q_\alpha$ is the threshold such that $(1-\alpha)\times 100\%$ of target covariates have a CIP predicted probability greater than $q_\alpha$. The larger $\alpha$ is, the more confident CIP is about the prediction on the covariate samples in $Q_\alpha$. The red line indicates $y=x$. }
    \label{fig:SCM3_coro2}
\end{figure}

We then turn to apply \corref{conditional_invariant_proxy_y_A} to compare the accuracy upper bound with the actual accuracy of DIP within region $\Aset$, where we define $\Aset$ as the region such that the CIP predicted probability exceeds a threshold $q_\alpha$. This threshold $q_\alpha$ is defined as the $\alpha\times 100$-th percentile of CIP predicted probability for the target covariates, i.e., $(1-\alpha)\times 100\%$ of target covariates have a CIP predicted probability greater than $q_\alpha$. As shown in Figure~\ref{fig:SCM3_coro2}, by increasing $\alpha$, we observe that our upper bound becomes increasingly precise within region $\Aset$. For example, DIP with certain values of $\lambda$ yields an upper bound lower than 0.5, suggesting that DIP flips the label. Although a low accuracy within region $\Aset$ does not necessarily translate to low accuracy across the entire target domain, in practice we can refer to domain experts and ask them whether such suboptimal performance of DIP within region $\Aset$ is reasonable or not, allowing us to avoid the need to acquire and validate all target labels.  

\subsection{MNIST under rotation intervention} 
\label{sec:mnist_rotation_intervention}
In this section, we consider binary classification of the MNIST dataset~\citep{lecun1998mnist} under rotation interventions, where digits 0-4 are categorized as label 0 and digits 5-9 are categorized as label 1. Similar to our experiments in SCMs, we first evaluate the prediction performance of various DA methods, then discuss how to detect potential failure of DIP without requiring access to target labels.

\subsubsection{Rotated MNIST under different interventions}

\begin{figure}[t]
	\centering
	\includegraphics[width=\textwidth]{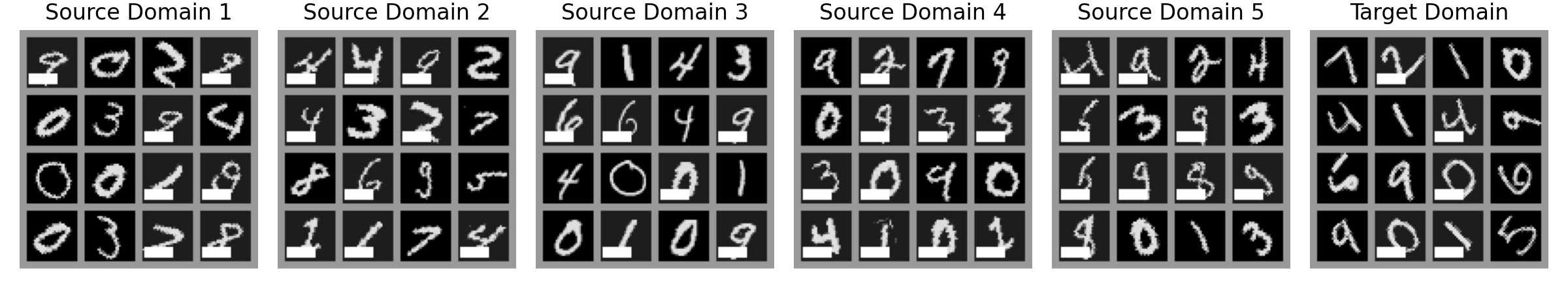}
	\caption{Sample images from MNIST \rom{3}. In this dataset, two different types of interventions have been applied: a rotation shift and the inclusion of label-flipping features. For rotation shift, each image in the $m$-th domain is rotated clockwise by $(m\times 15 - 30)^\circ$. As for the label-flipping features, a white bar of $6\times 16$ pixels is added in the bottom-left corner of the images. The correlation between the label and the white bar patch changes across the domains, with positive correlation in the $1^{\text{st}}, 3^{\text{rd}}$, and $5^{\text{th}}$ domains, and negative correlation in the $2^{\text{nd}}, 4^{\text{th}}$, and $6^{\text{th}}$ (target) domains.}
	\label{fig:MNIST3_example}
\end{figure}

We create five source domains and one target domain (the $6^{\text{th}}$ domain). The $5^{\text{th}}$ domain is fixed as the single source domain for DA algorithms that rely on only one source domain. Each source domain consists of 20\% of the images from the MNIST training set, and the target domain includes all images from the MNIST test set. We introduce three different types of interventions as follows:
\begin{itemize}
    \item Rotation shift: Each image in the $m$-th domain is rotated clockwise by $(m\times 15 - 30)^\circ$.
    \item Label shift: In the target domain, we remove $50\%$ of the images labeled as 0.
    \item Label-flipping features: In the $1^{\text{st}}, 3^{\text{rd}}$, and $5^{\text{th}}$ domains, $90\%$ of the images with label 1 and $10\%$ of the images with label 0 are patched by a white bar of $6\times16$ pixels in the bottom-left corner. This creates a correlation of $0.8$ between the label and the patch. Conversely, in the $2^{\text{nd}}, 4^{\text{th}}$, and $6^{\text{th}}$ domains, we add this white bar to $90\%$ of the images with label 0 and $10\%$ of the images with label 1. This creates a correlation of $-0.8$ between the label and the patch.
\end{itemize}
Similar to the settings of SCMs, we consider four DA problems under different combinations of these interventions. The rotation shift intervention is applied to all MNIST experiments, and four different cases are constructed based on whether the label shift or label-flipping features exist: MNIST \rom{1} includes rotation shift, but does not have label shift and label-flipping features; MNIST \rom{2} includes rotation shift and label shift, but does not have label-flipping features; MNIST \rom{3} includes rotation shift and label-flipping features, but does not have label shift; and MNIST \rom{4} includes all three interventions---rotation shift, label shift, and the presence of label-flipping feature. Image samples under MNIST \rom{3} are displayed in Figure~\ref{fig:MNIST3_example} for illustration.


We train a Convolutional Neural Network (CNN) model similar to LeNet5~\citep{lecun1998gradient} for all MNIST experiments; see Appendix~\ref{app:architec} for more details. Tables~\ref{tab:MNIST_results_1_2} and~\ref{tab:MNIST_results_3_4} summarize the performance of various DA methods on MNIST. Similar to our findings in linear SCM experiments, we conclude that DIP effectively addresses rotation shift in the absence of label-flipping features (MNIST~\rom{1}), while label correction addresses label shift (MNIST~\rom{2} and~\rom{4}), and JointDIP mitigates the issue of learning label-flipping features for DIP (MNIST~\rom{3} and~\rom{4}). Consequently, DA methods which incorporate steps specifically designed to address each shift have the best performance or are among the top-performing methods. 

\begin{table}[t]
	\centering
	\small
	\begin{tabular}{l | ll | ll}
		\toprule
		{} & \multicolumn{2}{c|}{MNIST \rom{1}} & \multicolumn{2}{c}{MNIST \rom{2}} \\\hline
		Rotation shift & \multicolumn{2}{c|}{Y} & \multicolumn{2}{c}{Y} \\
		Label shift & \multicolumn{2}{c|}{N} & \multicolumn{2}{c}{Y} \\
		Label flip & \multicolumn{2}{c|}{N} & \multicolumn{2}{c}{N} \\
		\hline
		DA Algorithm & src\_acc & tar\_acc & src\_acc & tar\_acc \\
		\hline
		Tar & 71.2$\pm$0.9 & \textbf{100.0$\pm$0.0} & 69.3$\pm$1.5 & \textbf{99.7$\pm$0.3} \\
		ERM & 99.9$\pm$0.1 & 96.9$\pm$0.4 & 99.9$\pm$0.2 & \textbf{96.9$\pm$0.5} \\
		ERM-Pool & 99.5$\pm$0.4 & 94.6$\pm$0.8 & 99.5$\pm$0.3 & 94.6$\pm$1.3 \\
		DIP & 100.0$\pm$0.1 & \textbf{\textcolor{red}{97.5$\pm$0.3}} & 99.9$\pm$0.2 & 96.7$\pm$0.5 \\
		DIP-Pool & 99.6$\pm$0.2 & 95.7$\pm$0.4 & 99.6$\pm$0.2 & 94.6$\pm$1.5 \\
		CIP & 99.6$\pm$0.3 & 94.9$\pm$0.9 & 99.6$\pm$0.3 & 94.9$\pm$2.1 \\
		IW-ERM & 99.4$\pm$0.4 & 94.8$\pm$0.8 & 99.4$\pm$0.2 & 94.8$\pm$0.5 \\
		IW-CIP & 99.7$\pm$0.2 & 94.8$\pm$0.8 & 99.2$\pm$0.4 & 94.9$\pm$1.2 \\
		IW-DIP & 99.9$\pm$0.3 & \textbf{97.3$\pm$0.4} & 99.8$\pm$0.2 & \textbf{\textcolor{red}{97.4$\pm$0.2}} \\
		JointDIP & 99.9$\pm$0.1 & \textbf{97.3$\pm$0.3} & 99.9$\pm$0.2 & \textbf{96.9$\pm$0.6} \\
		IW-JointDIP & 99.6$\pm$0.6 & 97.2$\pm$0.5 & 99.5$\pm$0.2 & 96.5$\pm$1.3 \\
		IRM & 99.5$\pm$0.5 & 94.3$\pm$1.4 & 99.6$\pm$0.1 & 94.3$\pm$1.4 \\
		V-REx & 99.4$\pm$0.2 & 94.6$\pm$0.8 & 99.6$\pm$0.1 & 94.5$\pm$0.9 \\
		groupDRO & 99.6$\pm$0.3 & 94.9$\pm$1.1 & 99.5$\pm$0.3 & 95.1$\pm$1.2 \\
		\bottomrule
	\end{tabular}
	\caption{Source and target accuracy in rotated MNIST \rom{1} and \rom{2}. For each setting of the rotated MNIST, DA algorithms are run on 10 different datasets, each generated using 10 different random seeds. Tar represents the oracle method where the model is trained on the labeled target data. The top four methods are highlighted in bold. The best method (excluding Tar) is colored in red.}
	\label{tab:MNIST_results_1_2}
\end{table}

\begin{table}[t]
	\centering
	\small
	\begin{tabular}{l | ll | ll}
		\toprule
		{} & \multicolumn{2}{c|}{MNIST \rom{3}} & \multicolumn{2}{c}{MNIST \rom{4}} \\\hline
		Rotation shift & \multicolumn{2}{c|}{Y} & \multicolumn{2}{c}{Y} \\
		Label shift & \multicolumn{2}{c|}{N} & \multicolumn{2}{c}{Y} \\
		Label flip & \multicolumn{2}{c|}{Y} & \multicolumn{2}{c}{Y} \\
		\hline
		DA Algorithm & src\_acc & tar\_acc & src\_acc & tar\_acc \\
		\hline
		Tar & 67.4$\pm$1.1 & \textbf{100.0$\pm$0.1} & 66.2$\pm$1.9 & \textbf{99.9$\pm$0.1} \\
		ERM & 100.0$\pm$0.1 & 87.8$\pm$1.9 & 100.0$\pm$0.1 & 86.8$\pm$1.8 \\
		ERM-Pool & 99.4$\pm$0.5 & 88.7$\pm$2.0 & 99.6$\pm$0.2 & 89.1$\pm$2.5 \\
		DIP & 100.0$\pm$0.0 & 91.0$\pm$1.9 & 100.0$\pm$0.0 & 90.1$\pm$1.6 \\
		DIP-Pool & 99.4$\pm$0.5 & \textbf{92.8$\pm$1.2} & 99.4$\pm$0.4 & 91.2$\pm$2.1 \\
		CIP & 99.4$\pm$0.4 & 89.1$\pm$1.7 & 99.1$\pm$0.9 & 90.0$\pm$1.3 \\
		IW-ERM & 99.4$\pm$0.5 & 88.0$\pm$1.7 & 98.8$\pm$0.8 & 88.3$\pm$3.3 \\
		IW-CIP & 99.6$\pm$0.2 & 89.8$\pm$1.7 & 99.4$\pm$0.4 & 90.1$\pm$1.0 \\
		IW-DIP & 100.0$\pm$0.0 & \textbf{92.6$\pm$1.1} & 100.0$\pm$0.0 & \textbf{90.9$\pm$1.5} \\
		JointDIP & 99.9$\pm$0.2 & \textbf{\textcolor{red}{93.5$\pm$1.2}} & 99.7$\pm$0.2 & \textbf{91.5$\pm$1.1} \\
		IW-JointDIP & 99.9$\pm$0.2 & \textbf{93.1$\pm$1.1} & 99.6$\pm$0.5 & \textbf{\textcolor{red}{93.4$\pm$0.9}} \\
		IRM & 99.3$\pm$0.7 & 88.4$\pm$1.9 & 99.1$\pm$0.3 & 89.3$\pm$1.4 \\
		V-REx & 99.1$\pm$0.3 & 89.5$\pm$1.4 & 99.4$\pm$0.6 & 89.8$\pm$1.3 \\
		groupDRO & 99.0$\pm$0.2 & 90.7$\pm$1.1 & 99.0$\pm$0.2 & 90.8$\pm$1.9 \\
		\bottomrule
	\end{tabular}
	\caption{Source and target accuracy in rotated MNIST \rom{3} and \rom{4}. For each setting of the rotated MNIST, DA algorithms are run on 10 different datasets, each generated using 10 different random seeds. Tar represents the oracle method where the model is trained on the labeled target data. The top four methods are highlighted in bold. The best method (excluding Tar) is colored in red.}
	\label{tab:MNIST_results_3_4}
\end{table}

\begin{figure}[t]
    \centering
    \includegraphics[width=\linewidth]{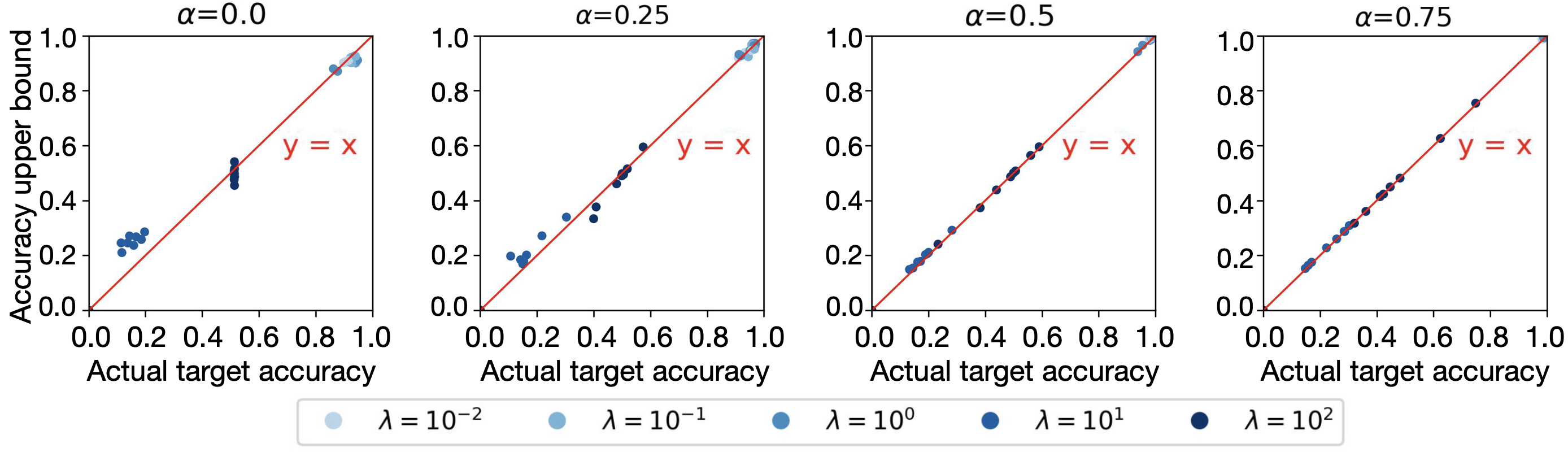}
    \caption{DIP in MNIST \rom{3}: target accuracy upper bound obtained via~\corref{conditional_invariant_proxy_y_A} vs. actual target accuracy in region $\Aset$. We define region $\Aset$ as $\Aset=Q_\alpha\defn\{X\in\mathcal{X}\mid \text{CIP predicted probability} \geq q_{\alpha} \}$, where $q_\alpha$ is the threshold such that $(1-\alpha)\times 100\%$ of target covariates have a CIP predicted probability greater than $q_\alpha$. The larger $\alpha$ is, the more confident CIP is about the prediction on the covariate samples in $Q_\alpha$. The red line indicates $y=x$.}
    \label{fig:MNIST3_coro2}
\end{figure}

\subsubsection{Detecting failure of DIP in rotated MNIST}\label{sec:mnist_label_flipping_exp}

{Analogous to \secref{SEM_label_flipping_exp}, we apply~\thmref{conditional_invariant_proxy_y} and~\corref{conditional_invariant_proxy_y_A} to MNIST~\rom{3} to detect potential failures of DIP, demonstrating the second role of CICs.} Figure~\ref{fig:coro1} (c)(d) illustrate that the theoretical upper bound on target accuracy obtained via \thmref{conditional_invariant_proxy_y} matches closely with the actual target accuracy, and that large values of $\lambda$ such as 10 and 100 may lead DIP to learn label-flipping features, while no such issues appear in JointDIP. Figure~\ref{fig:MNIST3_coro2} presents the result of applying~\corref{conditional_invariant_proxy_y_A} with different choices of region $\Aset=Q_\alpha$ for $\alpha\in\{0, 0.25, 0.5, 0.75\}$. Similar to our findings in Figure~\ref{fig:SCM3_coro2}, we observe that as CIP becomes more confident in region $\Aset$ (i.e., as $\alpha$ grows), the upper bound becomes more accurate. In Figure~\ref{fig:MNIST3_alpha_visualize}, we visualize images from the target domain sampled within the region $Q_\alpha$ for different values of $\alpha$. As $\alpha$ increases, the handwritten digits become more clear and distinguishable. In practice where we suspect label flipping by DIP but cannot access or afford to obtain many target labels, we can investigate target samples where CIP and DIP disagree in the region $Q_\alpha$ with a large value of $\alpha$, and refer to domain experts to evaluate whether these distinguishable images should be correctly classified by DIP. If DIP fails to do so, our procedure based on~\corref{conditional_invariant_proxy_y_A} can serve as a diagnostic tool to accurately estimate DIP's target performance and evaluate its validity in region $\Aset$, especially when~\thmref{conditional_invariant_proxy_y} provides an uninformative upper bound.

\begin{figure}[t]
    \centering
    \includegraphics[width=\textwidth]{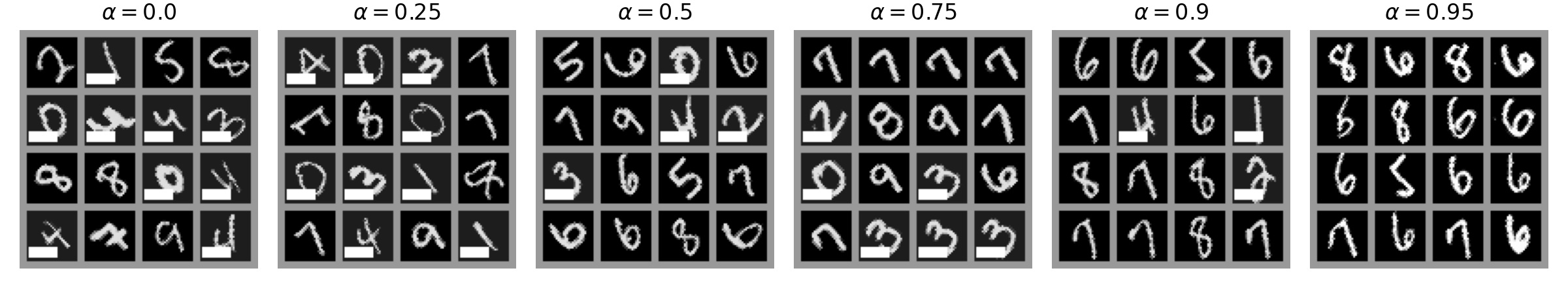}
    \caption{Sample images from MNIST \rom{3} target domain within region $Q_\alpha$ where CIP disagrees with DIP. The images are shown for varying $\alpha\in\{0, 0.25, 0.5, 0.75,0.9,0.95\}$. As $\alpha$ increases, the selected handwritten digits become more distinguishable.}
    \label{fig:MNIST3_alpha_visualize}
\end{figure}

\subsection{CelebA under color shift}

The CelebA dataset introduced by \citep{liu2015deep} is a large-scale face image dataset with multiple face attributes. It includes over 200K celebrity images, each annotated with 40 binary attributes such as Gender, Eyeglasses, and Smiling. In our experiments, we take Smiling as the label and create six different settings for DA problems. The first three CelebA settings are designed to study the color shift together with label-flipping features, similar to the approach taken with SCM \rom{3} and MNIST \rom{3}. The remaining three CelebA settings are developed to examine color shift in conjunction with label shift, similar to the approach taken with SCM \rom{2} and MNIST \rom{2}. For each setting, we construct three source domains and one target domain, each consisting of 20K images randomly sampled from CelebA. The last source domain is used as the single source domain for DA algorithms that rely on a single source domain. To predict the labels, we utilize a CNN model architectures across all problem settings (see Appendix~\ref{app:architec} for more details).

\begin{figure}[t]
    \centering
    \includegraphics[width=\textwidth]{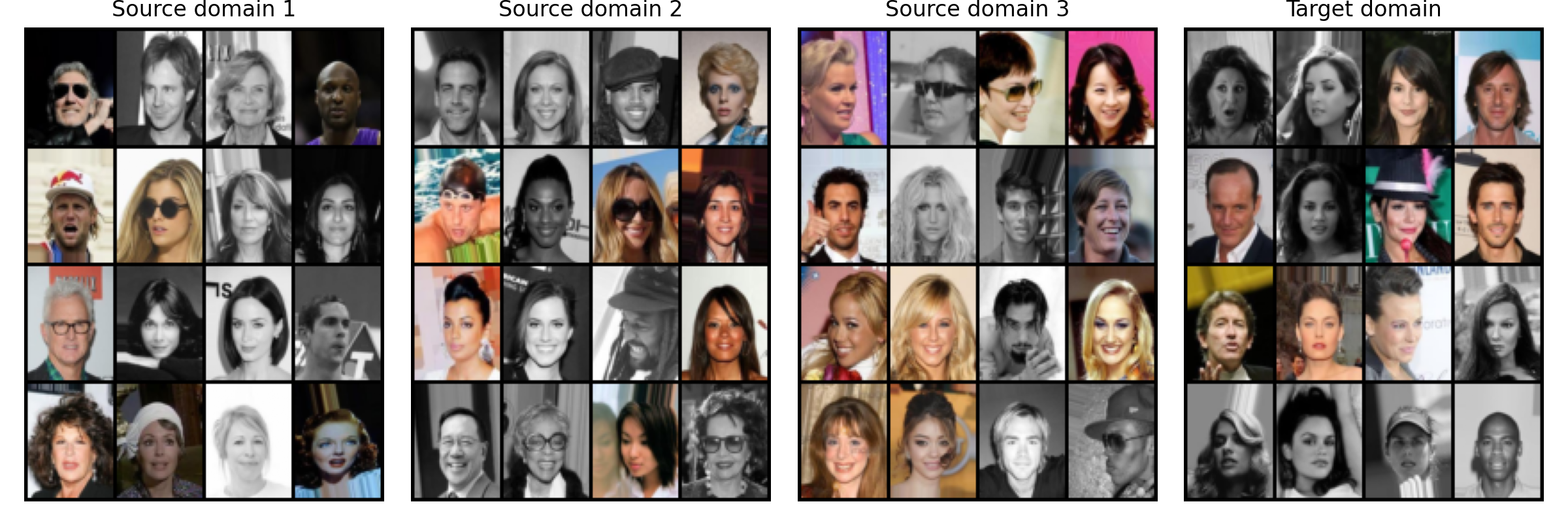}
    \caption{Sample images from CelebA \rom{3}. In this dataset, we synthetically create a feature Color\_Balance, which takes 1 for a full-color image and 0 for a black-and-white image. We also consider the original attribute Mouth\_Slightly\_Open, which takes 1 if the mouth in the image is slightly open, and 0 otherwise. The correlations between these two features and the label Smiling are varied across domains, as shown in Table~\ref{tab:CelebA_settings_mechanism}.}
    \label{fig:CelebA1_example}
\end{figure}

\subsubsection{Color shift with label-flipping features}
The settings of CelebA \rom{1}, CelebA \rom{2}, and CelebA \rom{3} are presented in Table~\ref{tab:CelebA_settings_mechanism}. In these settings, we manipulate two features, Color\_Balance and Mouth\_Slightly\_Open, to generate distribution shift across domains. The Color\_Balance feature is a synthetic attribute, which takes a value of 1 for a full-color image and 0 for a black-and-white image. The Mouth\_Slightly\_Open feature is an original binary attribute annotated in the CelebA dataset, with the value 1 if the mouth in the image is slightly open, and 0 otherwise. We vary these two features so that they are conditionally independent given the label Smiling, and their correlations with this label match the values shown in Table~\ref{tab:CelebA_settings_mechanism}. In all three settings, we align the Color\_Balance feature of the last source domain and the target domain most closely, while making the correlation between Smiling and Mouth\_Slightly\_Open reversed. As a result, the Color\_Balance feature can somewhat generalize to the target domain for label prediction, whereas the Mouth\_Slightly\_Open feature serves as a label-flipping feature between the last source domain and the target domain as defined in~\defref{label-flipping}. Figure~\ref{fig:CelebA1_example} illustrates image examples from each domain in CelebA \rom{3}. An ideal DA algorithm that utilizes the last source domain should be able to partially make use of the Color\_Balance feature, but should avoid the Mouth\_Slightly\_Open feature to predict the labels.

\begin{table}[t]
	\centering
	\small
	\begin{tabular}{l | c c c | c c c}
		\toprule
		& \multicolumn{3}{c|}{$\rho({\text{Smiling}}, {\text{Color\_Balance}})$} & \multicolumn{3}{c}{$\rho({\text{Smiling}}, {\text{Mouth\_Slightly\_Open}})$} \\\hline
		Domain & CelebA \rom{1} & CelebA \rom{2} & CelebA \rom{3} & CelebA \rom{1} & CelebA \rom{2} & CelebA \rom{3} \\\hline
		Source Domain 1 & -0.8 & -0.8 & 0.2 & -0.9 & -0.5 & -0.9 \\
		Source Domain 2 & -0.6 & -0.6 & 0.4 & 0.9 & 0.5 & 0.9 \\
		Source Domain 3 & 0.6 & 0.6 & 0.6 & 0.9 & 0.5 & 0.9 \\
		Target Domain   & 0.8 & 0.8 & 0.8 & -0.9 & -0.5 & -0.9 \\
		\bottomrule
	\end{tabular}
	\caption{Settings of CelebA \rom{1}, \rom{2}, \rom{3} under color shift with label-flipping features. Smiling is the label attribute for prediction. We vary the Color\_Balance and the Mouth\_Slightly\_Open features across domains. The correlation between Color\_Balance and Smiling remains relatively stable between the last source domain and the target domain, while Mouth\_Slightly\_Open serves as the label-flipping feature.}
	\label{tab:CelebA_settings_mechanism}
\end{table}

\begin{table}[t]
	\centering
	\small
	\begin{tabular}{l | c c | c c | c c }
		\toprule
		&  \multicolumn{2}{c|}{CelebA \rom{1}} &  \multicolumn{2}{c|}{CelebA \rom{2}}  &  \multicolumn{2}{c}{CelebA \rom{3}} \\\hline
		DA Algorithm & src\_acc & tar\_acc & src\_acc & tar\_acc & src\_acc & tar\_acc \\\hline
		ERM & 97.0$\pm$0.2 & 77.7$\pm$1.2 & 94.4$\pm$0.4 & \textbf{90.3$\pm$0.9} & 97.0$\pm$0.2 & 77.7$\pm$1.2 \\
		ERM-Pool & 93.3$\pm$0.3 & 77.2$\pm$1.4 & 92.8$\pm$0.4 & 83.6$\pm$1.0 & 93.1$\pm$0.2 & 89.6$\pm$0.4 \\
		DIP & 96.2$\pm$0.4 & \textbf{82.0$\pm$2.2} & 93.1$\pm$0.3 & \textbf{\textcolor{red}{90.6$\pm$1.0}} & 96.2$\pm$0.4 & 82.0$\pm$2.2 \\
		DIP-Pool & 92.3$\pm$0.5 & 80.4$\pm$1.3 & 90.9$\pm$0.4 & 85.5$\pm$0.6 & 91.6$\pm$0.3 & \textbf{90.4$\pm$0.4} \\
		CIP & 89.1$\pm$0.3 & 76.6$\pm$1.4 & 92.7$\pm$0.6 & 83.7$\pm$1.4 & 93.0$\pm$0.4 & \textbf{89.7$\pm$0.7} \\
		JointDIP & 93.0$\pm$0.8 & \textbf{\textcolor{red}{87.3$\pm$0.6}} & 93.1$\pm$0.3 & \textbf{90.4$\pm$1.2} & 91.4$\pm$1.1 & 88.7$\pm$1.0 \\
		IRM & 90.5$\pm$0.7 & 76.3$\pm$3.2 & 92.6$\pm$0.5 & 83.6$\pm$1.4 & 90.3$\pm$0.9 & 88.8$\pm$1.3 \\
		V-REx & 86.9$\pm$1.2 & \textbf{83.3$\pm$1.2} & 81.8$\pm$0.2 & 85.1$\pm$0.7 & 91.8$\pm$0.3 & \textbf{91.1$\pm$0.8} \\
		groupDRO & 91.7$\pm$1.0 & \textbf{84.3$\pm$1.0} & 91.5$\pm$0.5 & \textbf{86.4$\pm$1.5} & 92.4$\pm$0.4 & \textbf{\textcolor{red}{91.6$\pm$0.5}} \\
		\bottomrule
	\end{tabular}
	\caption{Source and target accuracy in CelebA \rom{1}, \rom{2}, \rom{3} with color shift and label-flipping features. For each setting of the CelebA, DA algorithms are run on 10 different datasets, each generated using 10 different random seeds. The top four methods are highlighted in bold. The best method is colored in red.}
	\label{tab:CelebA_mechan}
\end{table}

Given that there is no label shift in these settings, we evaluate DA methods without the importance-weighted label shift correction (IW) step, and present the results in Table~\ref{tab:CelebA_mechan}. We observe that in CelebA \rom{1}, where the correlations between Smiling and both the Color\_Balance and Mouth\_Slightly\_Open features vary significantly across domains, JointDIP achieves the best performance. In CelebA \rom{2} and CelebA \rom{3} where we reduce the variation in correlations between Smiling and one of the features, JointDIP continues to demonstrate comparable accuracy. GroupDRO also shows a good performance, likely due to the substantial shift across domains. {Overall, the results show that JointDIP maintains strong and robust performance in the presence of label-flipping features.}

\subsubsection{Color shift with label shift}

The settings of CelebA \rom{4}, CelebA \rom{5}, and CelebA \rom{6} are shown in Table~\ref{tab:CelebA_settings_label}. In these settings, we manipulate the Color\_Balance feature and additionally introduce interventions on the distribution of Smiling to create label shift between the source and target domains. Specifically, the correlation between Smiling and the Color\_Balance feature remains consistent with Table~\ref{tab:CelebA_settings_mechanism}, but we no longer vary the Mouth\_Slightly\_Open feature, thereby eliminating label-flipping features between the source and target domains. Instead, we change the label distribution across domains: the distribution of Smiling is balanced (50\% each) in all three source domains, but becomes unbalanced (either 75\% and 25\%, or 25\% and 75\%) in the target domain. We evaluate various DA methods, including those that incorporate the importance-weighted label shift correction (IW) step, as shown in Table~\ref{tab:CelebA_label}. Given the absence of label-flipping features in these settings, we do not evaluate JointDIP. We observe that IW-DIP outperforms all other methods. Note that ERM also achieves a high accuracy and outperforms ERM-Pool, indicating that training on all source domains may not be beneficial when substantial distribution shifts, in particular the label shift, exist across domains. {Overall, our experimental results highlight the importance of the label shift correction (IW) step in the presence of substantial label shift.}


\begin{table}[t]
	\centering
	\small
	\begin{tabular}{l | c c c | c c c}
		\toprule
		& \multicolumn{3}{c|}{$\rho({\text{Smiling}}, {\text{Color\_Balance}})$} & \multicolumn{3}{c}{$P(\text{Smiling}=1)$} \\\hline
		Domain & CelebA \rom{4} & CelebA \rom{5} & CelebA \rom{6}   & CelebA \rom{4} & CelebA \rom{5} & CelebA \rom{6}\\\hline
		Src Domain 1 & -0.8 & -0.8 & 0.2    & 0.5 & 0.5 & 0.5\\
		Src Domain 2 & -0.6 & -0.6 & 0.4    & 0.5 & 0.5 & 0.5 \\
		Src Domain 3 & 0.6 & 0.6 & 0.6      & 0.5 & 0.5 & 0.5\\
		Tar Domain  & 0.8 & 0.8 &  0.8     & 0.75 & 0.25 & 0.75\\
		\bottomrule
	\end{tabular}
	\caption{Settings of CelebA \rom{4}, \rom{5}, \rom{6} under color shift and label shift. Smiling is the label attribute for prediction. We vary the Color\_Balance feature and the  Smiling label distribution across domains. The correlation between Smiling and the Color\_Balance feature remains consistent with Table~\ref{tab:CelebA_settings_mechanism}. The distribution of Smiling is balanced (50\% each) in all three source domains, but becomes unbalanced (either 75\% and 25\%, or 25\% and 75\%) in the target domain.}
	\label{tab:CelebA_settings_label}
\end{table}

\begin{table}[t]
	\centering
	\small
	\begin{tabular}{l | c c | c c | c c }
		\toprule
		&  \multicolumn{2}{c|}{CelebA \rom{4}} &  \multicolumn{2}{c|}{CelebA \rom{5}}  &  \multicolumn{2}{c}{CelebA \rom{6}} \\\hline
		DA Algorithm &  src\_acc & tar\_acc &  src\_acc & tar\_acc &  src\_acc & tar\_acc\\\hline
		ERM &  94.4$\pm$0.6 & \textbf{94.1$\pm$0.5}        & 94.4$\pm$0.6 & \textbf{93.9$\pm$0.9}       & 94.5$\pm$0.6 & \textbf{94.2$\pm$0.5}\\
		ERM-Pool & 93.6$\pm$0.2 & 88.7$\pm$0.9              & 93.6$\pm$0.2 & 89.6$\pm$0.9                & 93.8$\pm$0.3 & 93.1$\pm$0.5\\
		DIP & 94.4$\pm$0.4 & \textbf{93.5$\pm$1.2}    & 94.4$\pm$0.5 & \textbf{94.4$\pm$0.9}       & 94.4$\pm$0.4 & 93.4$\pm$1.2\\
		DIP-Pool &  93.7$\pm$0.6 & 88.0$\pm$0.6      & 93.6$\pm$0.5 & 89.5$\pm$1.2                & 93.7$\pm$0.5 & 92.9$\pm$0.9\\
		CIP &  93.4$\pm$0.4 & 88.3$\pm$1.0            & 93.4$\pm$0.4 & 89.7$\pm$1.1                & 93.6$\pm$0.2 & 93.3$\pm$0.5\\
		IW-ERM & 92.2$\pm$0.9 & 90.1$\pm$1.0      & 93.2$\pm$0.4 & 90.9$\pm$0.8       & 92.1$\pm$0.7 & \textbf{94.4$\pm$0.2}\\
		IW-CIP &  91.8$\pm$0.9 & \textbf{90.3$\pm$0.5}                 & 92.7$\pm$1.1 & 90.6$\pm$1.1              & 91.7$\pm$1.3 & \textbf{94.3$\pm$0.5}\\
		IW-DIP &  93.0$\pm$0.7& \textbf{\textcolor{red}{94.8$\pm$0.3}} & 93.5$\pm$0.8 & \textbf{\textcolor{red}{95.1$\pm$0.5}} & 92.4$\pm$1.4 & \textbf{\textcolor{red}{94.7$\pm$0.4}}\\
		IRM & 93.8$\pm$0.3 & 88.7$\pm$0.9                  & 93.5$\pm$0.4 & 89.6$\pm$0.9    & 93.7$\pm$0.2 & 93.1$\pm$0.6\\
		V-REx &  92.5$\pm$0.2 & 89.5$\pm$0.9                & 92.6$\pm$0.3 & 90.2$\pm$1.2    & 93.8$\pm$0.2 & 93.4$\pm$0.6\\
		groupDRO &  93.3$\pm$0.3 & 90.0$\pm$0.8            & 93.3$\pm$0.3 & \textbf{91.1$\pm$1.0}    & 93.7$\pm$0.3 & 93.0$\pm$0.6\\
		\bottomrule
	\end{tabular}
	\caption{Source and target accuracy in CelebA \rom{4}, \rom{5}, \rom{6} under color shift and label shift. For each setting of the CelebA, DA algorithms are run on 10 different datasets, each generated using 10 different random seeds. The top four methods are highlighted in bold. The best method is colored in red.}
	\label{tab:CelebA_label}
\end{table}

\subsection{Camelyon17 from WILDS}\label{sec:camelyon}

{In this subsection, we conduct numerical experiments on Camelyon17 to demonstrate the effectiveness of JointDIP on the real data.} The Camelyon17 image dataset is one of the WILDS benchmarks~\citep{koh2021wilds, sagawa2021extending} for domain adaptation, which comprises whole-slide images (WSIs) of breast cancer metastases in lymph node sections from five hospitals. The task involves predicting whether a $96\times 96$ image patch extracted from a WSI contains any tumor pixels within its central $32\times 32$ area. The labels are obtained through manual pixel-level annotations of each WSI. Among the five different hospitals (domains) from which the WSIs are collected, the first three hospitals are treated as source domains (with 302.4K patches) and the remaining two are used as validation domain (with 34.9K patches) and target domain (with 85K patches), respectively. In addition to these labeled examples, WILDS also provides additional unlabeled WSIs from the same five hospitals and patches extracted from them, including 600K unlabeled patches from the target domain. See~\citep{koh2021wilds, sagawa2021extending} for a detailed description of the Camelyon17 dataset.

Although the 600K unlabeled target patches are collected from the same hospital as the 85K labeled target patches, there is no overlap in the WSIs where these patches are extracted from~\citep{sagawa2021extending}. Furthermore, the label distribution of unlabeled target patches is heavily skewed towards negative, while labeled target patches are sampled in a class-balanced manner~\citep{sagawa2021extending}. Consequently, this could create distribution shifts between the ``unlabeled target data" and the ``labeled target data" in the Camelyon17 dataset. In particular, this can lead DA algorithms, which utilize target covariate information (e.g. DIP or its variants), to identify different latent features depending on the source of target image patches used by these algorithms. In the WILDS leaderboard, methods that leverage data from the target domain utilize these 600K unlabeled patches, rather than treating the 85K labeled target patches as if they were unlabeled. However, if a distribution shift exists between the ``unlabeled target data" and the ``labeled target data", there is no guarantee that our proposed DA algorithms will improve target performance. Therefore, our experiments examine two approaches to utilize target data for DA methods: (1) using the 600K unlabeled target patches, and (2) leveraging the 85K labeled target patches without label access.

We use the standard DenseNet-121 model architecture~\citep{huang2017densely} and follow the training protocol described in~\cite{koh2021wilds, sagawa2021extending}. CIP, DIP-Pool, and JointDIP-Pool are implemented, where CIP only uses labeled source data, and DIP-Pool and JointDIP-Pool leverage target covariates (image patches) from the target domain. We find that the pooled versions of DIP and JointDIP increase reliability of the models by utilizing more samples from multiple source domains, compared to their single source versions. Therefore, we choose to implement the pooled version of DIP and JointDIP in our experiments. Similarly, DANN~\citep{ganin2016domain} and CORAL~\citep{sun2016deep}, which were initially developed for the single source version, have been extended in~\citep{sagawa2021extending} to take advantage of multiple source domains. Following the standard submission guidelines in WILDS~\citep{koh2021wilds, sagawa2021extending}, we do not use any data augmentation in CIP, but allow DIP-Pool and JointDIP-Pool to utilize the same color augmentation as described in \cite{koh2021wilds, sagawa2021extending}. Hyperparameters are chosen via a grid search based on the validation accuracy, where we allow a larger range of grids compared to previous experiments (see \appref{hyper} for further details).

\begin{table}[t]
	\centering
	\small  
	\begin{tabular}{c | c | c c}
		\toprule
		DA Algorithm & Target covariates & val accuracy & test accuracy\tablefootnote{The test accuracy is evaluated on 85K labeled target image patches.} \\\hline
		IRM$^*$ & \multirow{5}{*}{-} & 86.2$\pm$1.4 & 64.2$\pm$8.1 \\
		groupDRO$^*$ & & 85.5$\pm$2.2 & 68.4$\pm$7.3\\
		ERM-Pool$^*$ & & 84.9$\pm$3.1 & 70.3$\pm$6.4 \\
		CIP & & 87.0$\pm$1.2 & 71.1$\pm$12.5 \\
		Fish$^*$ & & 83.9$\pm$1.2 & \textbf{74.7$\pm$7.1}\\\hline
		DANN$^*$ & \multirow{4}{*}{600K unlabeled target patches} & 86.9$\pm$2.2 & 68.4$\pm$9.2\\
		CORAL$^*$ & & 90.4$\pm$0.9 & 77.9$\pm$6.6\\
		DIP-Pool &  & 91.6$\pm$0.8 & 81.2$\pm$8.2 \\
		JointDIP-Pool &  & 91.5$\pm$0.6 & \textbf{82.7$\pm$5.2}\\\hline
		DIP-Pool & \multirow{2}{*}{\shortstack{85K labeled target patches \\(target labels are not used)}} & 91.0$\pm$0.5 & 88.7$\pm$5.7 \\
		JointDIP-Pool &  & 91.7$\pm$0.6 & \textbf{91.9$\pm$3.1}\\
		\bottomrule
	\end{tabular}
	\caption{Accuracy of DA algorithms on Camelyon17 dataset. The first five DA methods---IRM, GroupDRO, ERM-Pool, CIP, and Fish---only use labeled data from multiple source domains, whereas DANN, CORAL, DIP-Pool, and JointDIP-Pool additionally use target image patches from the target domain, but they do not have access to the target labels. The algorithms are run over 10 trials. Results for algorithms marked with $^*$ are obtained from the WILDS benchmark.}
	\label{tab:Camelyon}
\end{table}

The results are presented in~\tabref{Camelyon}, where we compare CIP with other DA algorithms that enforce different invariance across all source domains, such as IRM~\citep{arjovsky2019invariant} and Fish~\citep{shi2021gradient}; and we also compare DIP-Pool and JointDIP-Pool with other domain-invariant algorithms using different distributional distances, like DANN~\citep{ganin2016domain} and CORAL~\citep{sun2016deep}. We observe that CIP has a similar performance compared to ERM-Pool, albeit with a slightly larger variance in test accuracy. When covariates of 600K unlabeled target patches are utilized in the DIP matching penalty, DIP-Pool results in an accuracy exceeding $80\%$. Remarkably, JointDIP-Pool further improves this accuracy by $1.5\%$ while substantially reducing the variance at the same time. In scenarios where covariates of 85K labeled target patches are used in the DIP matching penalty, the accuracy of both methods increases significantly, with JointDIP-Pool improving DIP-Pool by over $3\%$. This result indicates that jointly matching the covariates with CICs between source and target can help DIP-based algorithms to identify the appropriate features to match across domains. In addition, when comparing the accuracy of two different approaches of using target covariates in the DIP matching penalty, our results demonstrate that distribution shifts exist between the unlabeled and labeled target data in Camelyon17, and leveraging the true target covariates from the 85K labeled target patches proves more effective than using the target covariates from the 600K unlabeled target patches.


\subsection{DomainNet with binary classification}

In this subsection, we conduct numerical experiments on the DomainNet dataset. The DomainNet dataset, introduced by \citep{peng2019moment}, is a large-scale, multi-domain dataset which contains approximately 600k images collected from 6 different domains: Clipart (illustrative images, such as cartoons), Infograph (infographics containing both graphical and textual information), Painting (artistic depictions in the form of paintings), Quickdraw (simple, hand-drawn images from the Quick Draw game dataset), Real (photographic, real-world images), and Sketch (hand-drawn sketches). Each image is labeled with one of 345 object categories, which are further grouped into 24 divisions (see Tables 10, 11, and 12 of \cite{peng2019moment}). In our experiment, to make it easy to interpret the target performance, we transformed the original multi-class prediction problem into a binary classification task. Specifically, for each domain, we selected images with categories corresponding to the ``Furniture" and ``Mammal" groups, consisting of 35 and 25 categories, respectively, and labeled them as 0 and 1. Merging these categories makes the classification problem simpler and increases the number of images available for each label, so that target performance variation due to small sample size can be reduced. To ensure equal label proportions within each domain, we randomly sub-sampled images to achieve a balanced dataset with 50\% labeled as 0 and 50\% labeled as 1. This eliminates label shift across all domains and the only source of distribution shift arises from variations in visual styles, contexts, and object representations in the images. As a result, each domain contains 6,844, 6,952, 9,522, 25,000, 30,780, and 10,212 images, respectively. Each domain is used as a target domain in turn, with the remaining domains serving as source domains. This setup results in $6$ different cases for evaluating the performance of DA methods.

Our DA algorithms are built on top of a ResNet-101 model architecture~\citep{he2016deep}, pre-trained on the ImageNet dataset. Since there is no label shift in this task, we implement DA methods without importance weight estimation-based label shift correction. Similar to our approach in~\secref{camelyon}, we find that the pooled versions of DIP and JointDIP provide more reliable results by leveraging more samples from multiple source domains compared to their single-source versions. Therefore, we implement CIP, DIP-Pool, and JointDIP-Pool using MMD as the distributional distance, along with other DA methods such as IRM, V-REx, and groupDRO. Hyperparameters are chosen via a grid search, as detailed in~\appref{hyper}.

\begin{table}[t]
	\centering
	\small
	\begin{tabular}{l | c c c c c c}
		\toprule
		&  Clipart & Infograph & Painting & Quickdraw & Real & Sketch \\\hline
		DA Algorithm & tar\_acc & tar\_acc & tar\_acc & tar\_acc & tar\_acc & tar\_acc \\\hline
		ERM-Pool & \textbf{89.9$\pm$0.7} & 64.7$\pm$1.8 & \textbf{81.1$\pm$1.3} & 69.6$\pm$4.2 & 90.4$\pm$2.8 & \textbf{87.3$\pm$0.5} \\
		DIP-Pool & 89.7$\pm$0.7 & \textbf{\textcolor{red}{75.9$\pm$9.6}} & \textbf{80.4$\pm$2.4} & \textbf{71.1$\pm$4.5} & 90.9$\pm$1.2 & \textbf{\textcolor{red}{87.4$\pm$0.6}} \\
		CIP & 88.9$\pm$1.7 & \textbf{64.9$\pm$1.0} & 79.5$\pm$1.3 & \textbf{70.5$\pm$2.7} & 91.0$\pm$1.0 & 86.7$\pm$1.2 \\
		JointDIP-Pool & \textbf{90.0$\pm$0.9} & \textbf{66.0$\pm$1.6} & \textbf{\textcolor{red}{86.5$\pm$5.2}} & \textbf{\textcolor{red}{74.1$\pm$2.5}} & \textbf{\textcolor{red}{91.8$\pm$0.9}} & \textbf{87.2$\pm$0.9} \\
		IRM & \textbf{90.0$\pm$0.8} & 64.1$\pm$2.4 & \textbf{81.0$\pm$1.7} & 69.0$\pm$4.9 & \textbf{91.8$\pm$1.0} & \textbf{87.2$\pm$0.7} \\
		V-REx & \textbf{\textcolor{red}{90.0$\pm$0.5}} & \textbf{65.7$\pm$1.3} & 79.8$\pm$2.3 & \textbf{69.9$\pm$4.0} & \textbf{91.5$\pm$0.5} & 86.8$\pm$0.9 \\
		groupDRO & 86.8$\pm$0.4 & 64.2$\pm$2.2 & 77.9$\pm$0.8 & 68.0$\pm$6.1 & \textbf{91.6$\pm$0.9} & 84.1$\pm$0.8 \\
		\bottomrule
	\end{tabular}
	\caption{Target accuracies in DomainNet for predicting the labels ``Furniture'' and ``Mammal.'' Each column corresponds to a specific target domain: Clipart, Infograph, Painting, Quickdraw, Real, and Sketch, where the remaining domains are used as source domains. The algorithms are run over 10 trials with different random seeds. The top four methods are highlighted in bold. The best method is colored in red.}
	\label{tab:DomainNet}
\end{table}

The results are presented in~\tabref{DomainNet}. Overall, we observe that JointDIP-Pool consistently achieves the best or comparable performance to the top method across all $6$ target domain settings. For certain target domains such as Painting and Quickdraw, it significantly outperforms all other DA methods. This suggests that JointDIP-Pool is more effective at learning invariant features across source and target domains, enabling it to handle diverse visual characteristics when classifying “Furniture” and “Mammal.” DIP-Pool also achieves high accuracy, with notably better target performance compared to other DA methods when Infograph is used as the target domain. However, it exhibits high standard deviation in this setting, potentially indicating that DIP-Pool may have learned label-flipping features during some trials. This aligns with observations in~\cite{peng2019moment}, where generalization to Infograph was found to be more challenging, resulting in lower accuracies compared to other settings. V-REx also performs well in certain cases, likely due to its conservative approach by minimizing worst-case performance in the presence of substantial domain shifts. Finally, for challenging target domains such as Infograph and Quickdraw, where average target accuracies are generally lower, DIP-Pool, CIP, JointDIP-Pool, and V-REx demonstrate consistently strong performance, highlighting that these methods can better transfer knowledge between source and target domains in these difficult scenarios.





\section{Discussion}\label{sec:discussion}

In this paper, we highlight three prominent roles of CICs in DA. First, we propose the IW-CIP algorithm based on CICs, which can solve DA problems beyond simple covariate shift and label shift with theoretical guarantees. Specifically, we establish an explicit target risk bound for IW-CIP and further provide fine-grained analysis for each term in the bound. Second, we demonstrate the advantage of using CICs to measure the performance of other DA algorithms. For example, a conditionally invariant classifier built upon CICs can be applied as an {approximate} proxy for target labels to refute DA algorithms with poor target performance. Finally, to solve the label-flipping issue of DIP, we introduce the JointDIP algorithm which jointly matches the distribution of DIP features and CICs. Both our theoretical results and numerical experiments support the benefits of using CICs in DA.

While CICs across multiple source domains can be beneficial in DA, in practice it may not always be advantageous to use all available source domains for identifying CICs. Particularly, if not all source domains are related to the target domain, or if there aren't any features invariant across all source domains, it would be more efficient to select a subset of source domains for learning CICs. However, choosing the source domains that are related to the target domain or share common features is difficult without access to target labels. In such settings, one feasible approach is to learn invariant features across possible combinations of source domains and refer to domain experts to investigate those features, which would help determine the selection of source domains. We leave the development of a framework for best source domain selection for future research.


There are other promising directions to pursue following our work. While we provide bounds in~\secref{detect_failure} to detect failure of DA algorithms, these bounds may not be tight and potentially can be improved based on other structural assumptions on the data generation model. Moreover, although in this paper we mainly focus on combing CICs with other DA algorithms, it is possible that other forms of invariance can also be applied to improve DA algorithms. For instance, can we use features obtained by IRM~\citep{arjovsky2019invariant}, V-REx~\citep{krueger2021out}, or other domain generalization algorithms to aid and enhance existing DA methods such as DIP with target risk guarantees? Finally, this paper mainly considers the unsupervised domain adaptation scenario where target labels are unavailable, however, datasets such as Camelyon17 from WILDS~\citep{koh2021wilds} may have a small portion of target labels available. Developing DA algorithms with target risk guarantees in such settings is therefore of practical interest.

\acks{Keru Wu and Yuansi Chen are partly supported by the NSF CAREER Award DMS-2237322 and a Sloan Research Fellowship. Wooseok Ha was supported from NSF TRIPODS Grant 1740855, DMS-1613002, and 2015341. We also gratefully acknowledge partial support from NSF grants DMS-2209975 and DMS-2015341,  and  NSF grant 2023505 on Collaborative Research: Foundations of Data Science Institute (FODSI).}

\clearpage
\newpage
\appendix

\section{Illustrative examples of general anticausal model}\label{app:anticausal_example}

In this section, we provide several examples for the general anticausal model and demonstrate how the conditions in~\thmref{SEM_guarantee}(b) and~\thmref{SEM_guarantee}(c) are satisfied under these models.

\subsection{Example with CICs and label-flipping features}\label{app:a_simple_example}
We present an example of the general anticausal model, as defined in~\defref{model1}, which incorporates both CICs and label-flipping features.

\paragraph{Example 1} Let $Y\in\Yset=\{1,2\}$ be a binary random variable and balanced in both domains, i.e., $\probsrc{1}=\probsrc{2}=\probtar{1}=\probtar{2}=1/2$. Let $X\in\real^3$ be a three-dimensional covariate vector with mechanism function $\fsrcone$ and $\ftar$ as given by
\begin{align*}
    \fsrcone(1)  &= \pmat{-1\\-1\\-1}, \; \fsrcone(2)  = \pmat{1\\1\\1};\\
    \ftar(1)  &= \pmat{-1\\0\\1}, \; \ftar(2)  = \pmat{1\\2\\-1},
\end{align*}
and where the noise terms $\epssrcone$ and $\epstar$ are drawn from a Gaussian distribution $\Peps = \Normal(0, \Ind_3)$. Examining each coordinate, we can see that the first coordinate $\coord{X}{1}$ maintains its conditional invariance across source and target, and thus it is a CIC. The second coordinate $\coord{X}{2}$ is perturbed, but the correlations $\rho(\coordw{\Xsrcone}{2}, \Ysrcone)=\rho(\coordw{\Xtar}{2}, \Ytar)=\sqrt{2}/2$ remian unchanged. The third coordinate $\coord{X}{3}$ is also perturbed, and by~\defref{label-flipping}, it is a label-flipping feature, as $\rho(\coordw{\Xsrcone}{3}, \Ysrcone)= -\rho(\coordw{\Xtar}{3}, \Ytar)=\sqrt{2}/2$.

\subsection{Demonstration of condition in~\thmref{SEM_guarantee}(b)}\label{app:a_condition_b}
The condition~\eqnref{linear_cic_condition} in~\thmref{SEM_guarantee}(b) requires that for the linear conditionally invariant feature mapping $\phiinv$, CICs must have different conditional means for different labels. In other words, the first-moment of the CICs should suffice to distinguish between labels. We show that under the anticausal model given in the example above, this condition is indeed satisfied.
\paragraph{Example 1 Cont.}
In Example 1, one possible linear conditionally invariant feature mapping is $\phiinv(x) = \coord{x}{1}$, i.e., it projects to the first coordinate of the covariates. Simple algebra shows that
\begin{align*}
    \EEst{\phiinv(\Xsrcone)}{\Ysrcone=1} = -1, \text{ and } \EEst{\phiinv(\Xsrcone)}{\Ysrcone=2=1},
\end{align*}
which confirms the condition~\eqnref{linear_cic_condition}.

\subsection{Demonstration of condition in~\thmref{SEM_guarantee}(c)}\label{app:a_condition_c}
While the condition~\eqnref{linear_cic_condition} for linear $\phiinv$ is relatively straightforward, verifying the full-rank condition of matrix $C_{\phiinv}(a)$ for general $\phiinv$ in~\thmref{SEM_guarantee}(c) can be more challenging due to the computation of higher-moments. Here we provide several examples where a nonlinear conditionally invariant feature mapping $\phiinv$ meets the full rank condition.

\paragraph{Example 1 Cont.}
Consider the general anticausal model introduced in Example 1 ($L=2$) and let
\begin{align}
    \begin{split}\label{eqn:nonlinear_phiinv}
        \phiinv^k(x) &= \coord{x}{1}^k, \ \ k = 0,1,2,\ldots, \text{ and }\\
        \phiinv^e(x) &= \exp(\coord{x}{1}).
    \end{split}
\end{align}
Since $\coord{X}{1}$ is a CIC, we know that $\phiinv^k(\coord{X}{1})$ and $\phiinv^e(\coord{X}{1})$ are also CICs. Take $a=1$ in the definition of the matrix $C_{\phiinv}(a)$. From the standard calculation of moments for Gaussian random variables, we have
\begin{align*}
    C_{\phiinv^k}(a)&=\pmat{1 & 1\\
                            \Ep{u\sim\Normal(0, 1)}{(u-1)^k} & \Ep{u\sim\Normal(0, 1)}{(u+1)^k}}, \;\; C_{\phiinv^e}(a)=\pmat{1 & 1\\
    e^{-\frac12}& e^{\frac32}}.
\end{align*}
We conclude that $C_{\phiinv^k}(a)$ is full rank if and only if $k$ is odd, and $C_{\phiinv^e}(a)$ is a full rank matrix. Therefore, $\phiinv^{2k+1}(x)= \coord{x}{1}^{2k+1} (k=0,1,2,\ldots)$ and $\phiinv^e(x) =  \exp(\coord{x}{1})$ satisfy the full rank condition.

\paragraph{Example 2}
Consider the general anticausal model from~\assumpref{uniform_y} where $Y\in\Yset=\{1,2,3\}$ with $L=3$. Let $X\in\real^2$ be a two-dimensional covariate vector with mechanism functions $\fsrcone$ and $\ftar$ defined as follows:
\begin{align*}
    \fsrcone(1)  &= \pmat{1\\-1}, \; \fsrcone(2)  = \pmat{2\\0}, \ \fsrcone(3) = \pmat{3\\1}; \\
    \ftar(1)  &= \pmat{1\\1}, \; \ftar(2)  = \pmat{2\\0}, \ \ftar(3) = \pmat{3\\-1}.
\end{align*}
That is, the first coordinate $\coord{X}{1}$ is a CIC, while the second coordinate $\coord{X}{2}$ is perturbed. Assume that the noise distribution $\Peps=\Normal(0, \Ind_2)$ follows a Gaussian distribution. Consider the feature mappings $\phiinv^2$ and $\phiinv^e$ from Eq.~\eqnref{nonlinear_phiinv} as the conditionally invariant feature mappings, and take $a=1$. Basic algebra shows that
\begin{align*}
    C_{\phiinv^2}(a)&=\pmat{1 & 1 & 1\\
    \Ep{u\sim\Normal(0, 1)}{(u+1)^2} & \Ep{u\sim\Normal(0, 1)}{(u+2)^2} & \Ep{u\sim\Normal(0, 1)}{(u+3)^2}\\
    \Ep{u\sim\Normal(0, 1)}{(u+1)^4} & \Ep{u\sim\Normal(0, 1)}{(u+2)^4} & \Ep{u\sim\Normal(0, 1)}{(u+3)^4}
    }\\&=\pmat{1 & 1 & 1\\
    2 & 5 & 10\\
    10 & 43 & 138},
\end{align*}
and
\begin{align*}
    C_{\phiinv^e}(a)&=\pmat{1 & 1 & 1\\
    \Ep{u\sim\Normal(0, 1)}{e^{u+1}} & \Ep{u\sim\Normal(0, 1)}{e^{u+2}} & \Ep{u\sim\Normal(0, 1)}{e^{u+3}}\\
    \Ep{u\sim\Normal(0, 1)}{e^{2(u+1)}} & \Ep{u\sim\Normal(0, 1)}{e^{2(u+2)}} & \Ep{u\sim\Normal(0, 1)}{e^{2(u+3)}}
    }\\&=\pmat{1 & 1 & 1\\
    e^{\frac32} & e^{\frac52} & e^{\frac72}\\
    e^{4} & e^6 & e^8
    },
\end{align*}
both of which are full-rank, and therefore satisfying the full rank condition in~\thmref{SEM_guarantee}(c).

\paragraph{Example 3}
Consider the general anticausal model from~\assumpref{uniform_y} with arbitrary number of classes $L\geq 2$. Let
\begin{align*}
    \fsrcm{m}(y) = \begin{pmatrix}
                    e_y\\
                    a_y^{(m)}
                \end{pmatrix},
    \ftar(y) = \begin{pmatrix}
                    e_y\\
                    a_y^{\tagtar}
    \end{pmatrix},
\end{align*}
where $e_i\in\real^L$ denotes the zero vector with only the $i$-th element being $1$, and $a_y^{(m)}, a_y^{\tagtar}\in\real^{p-L}$ $(p>L)$ are arbitrary vectors. That is, the first $L$ covariates are CICs while the last $p-L$ coordinates can be arbitrarily perturbed. Let
\begin{align*}
    \phiinv^k(X) = \begin{pmatrix}\coord{X}{1}^k,
                           \coord{X}{2}^k,
                           \cdots,
                           \coord{X}{L}^k
    \end{pmatrix}^\top, \ \ k=1,2,\ldots.
\end{align*}
For a fixed vector $a\in\real^L$, verifying $C_{\phiinv^k}(a)$ is full rank can be technical depending on the distribution $\Peps$. Instead, we show \textit{an illustrative explanation} rather than a rigorous proof. First, let's consider $\Peps=\delta_0(\cdot)$, the Dirac distribution centered at 0. Then we can compute
\begin{align*}
    C_{\phiinv^k}(a) = \begin{pmatrix}
               1 & 1 & \cdots & 1\\
               \coord{a}{1} & \coord{a}{2} & \cdots & \coord{a}{L}\\
               \vdots & \vdots & \ddots & \vdots\\
               \coord{a}{1}^{L-1} & \coord{a}{2}^{L-1} & \cdots & \coord{a}{L}^{L-1}
    \end{pmatrix},
\end{align*}
which is a Vandermonde matrix and is full rank as long as we take $\coord{a}{i}\neq \coord{a}{j}, i\neq j$. Next we consider a case where $\Peps$ is a distribution concentrated around 0 (e.g. a Gaussian distribution with very small variance). In this case, the matrix will be slightly perturbed but is still full rank by the continuity of high-order moments and matrix determinants.

\section{Proof of~\secref{CIP} and additional results on IW-CIP}\label{app:CIP_proof}
In this section we prove our main theorems and propositions given in~\secref{CIP}, as well as additional results given in~\propref{iwcip_phistar_squared_distance} and~\lemref{linear_sem_example}. \propref{iwcip_phistar_squared_distance} provides the finite-sample bound for the IW-CIP penalty when it is defined by the squared mean distance, and \lemref{linear_sem_example} proves the number of required source domains to validate the conditions in~\propref{linear_sem_example} when the perturbations are randomly generated.

\subsection{Proof of \thmref{IW-CIP_gen_bound}}\label{app:IW-CIP_gen_hound}

    Let $\w = (\wone,\ldots,\wM)$ where $\wm$ is the true importance weight for the $m$-th source domain. Writing $p_y=\prob{\Ytar=y}$, for any classifier $h=g\circ \phi$, we have
    \begin{align*}
        &\risktar{h} - \riskbar{h}{\w} \\
        &= {\frac 1M} \sum_{m=1}^{M}\left(\EE{\ones{h(\Xtar)\neq\Ytar}} - \EE{\wmcoord{\Ysrc}\cdot\ones{h(\Xsrc)\neq\Ysrc}} \right) \\
        &= {\frac 1M} \sum_{m=1}^{M}\sum_{y=1}^{L}\left(\probc{h(\Xtar)\neq y}{\Ytar=y} - \probc{h(\Xsrc)\neq y}{\Ysrc=y} \right)p_y\\
        &= {\frac 1M} \sum_{m=1}^{M}\sum_{y=1}^{L}\left( \probc{h(\Xsrc)= y}{\Ysrc=y} - \probc{h(\Xtar)= y}{\Ytar=y}\right)p_y,
    \end{align*}
    where the second step applies $\wmy=\frac{\prob{\Ytar=y}}{\prob{\Ysrc=y}}$ and the law of total expectation. By definition of $\Gdiv{\cdot}{\cdot}$, it is then upper bounded by
    \begin{align}\label{eqn:R_minus_Rhw}
    \abss{ \risktar{h} - \riskbar{h}{\w} }
    &\leq {\frac 1M} \sum_{m=1}^{M}\sum_{y=1}^{L} \Gdiv{\Pcondtar{\phi(X)}{Y=y}}{\Pcondsrc{\phi(X)}{Y=y}}\prob{\Ytar=y}\notag\\
    &\leq \max_{\substack{m=1,\ldots,M, \\ y=1,\ldots,L}}\Gdiv{\Pcondtar{\phi(X)}{Y=y}}{\Pcondsrc{\phi(X)}{Y=y}}\sum_{m=1}^M{\frac 1M}\sum_{y=1}^L\prob{\Ytar=y}\notag\\
    &= \max_{\substack{m=1,\ldots,M, \\ y=1,\ldots,L}}\Gdiv{\Pcondtar{\phi(X)}{Y=y}}{\Pcondsrc{\phi(X)}{Y=y}} \notag\\
    &=\Pssi{\phi},
    \end{align}
    where we use $\sum_{y=1}^L\prob{\Ytar=y}=1$ and $\sum_{m=1}^M1/M=1$ in the second step. Furthermore, we can bound
    \begin{align*}
        \abss{\risksrcw{h}{\wm} - \risksrcw{h}{\whm}} &= \abss{\EE{\parenth{\wmcoord{\Ysrc} - \whmcoord{\Ysrc}}\cdot\ones{h(\Xsrc)\neq\Ysrc}} } \\
        &\leq \norm{\wm - \whm}_\infty.
    \end{align*}
    It follows that
    \begin{align}\label{eqn:Rhw_minus_Rhwhat}
    \abss{\riskbar{h}{\w} - \riskbar{h}{\wh}} &\leq {\frac 1M} \sum_{m=1}^{M}\abss{ \risksrcw{h}{\wm} - \risksrcw{h}{\whm} }\notag\\
    &\leq \frac{1}{M}\sum_{m=1}^{M}\norm{\wm - \whm}_\infty.
    \end{align}
    The result now follows from combining Eq.~\eqnref{R_minus_Rhw} and Eq.~\eqnref{Rhw_minus_Rhwhat}.

\subsection{Proof of \thmref{IW-CIP_gen_bound2}}\label{app:IW-CIP_gen_hound2}
    Denote the empirical average source risk across $M$ environments by
    \begin{align*}
        \riskhbar{h}{\wwm}=\frac 1M \sum_{m=1}^M\riskhsrcw{h}{\wwm}.
    \end{align*}
    Then we can write
    \begin{align*}
        &\quad \risktar{\hhiwcip} - \risktar{\hstar} \\
        &= \underbrace{\risktar{\hhiwcip}- \riskbar{\hhiwcip}{w}}_{(A_1)} + \underbrace{\riskbar{\hhiwcip}{w} - \riskhbar{\hhiwcip}{w}}_{(B_1)} +\\
        &\quad + \underbrace{\riskhbar{\hhiwcip}{w}-\riskhbar{\hhiwcip}{\wh}}_{(C_1)} + \underbrace{\riskhbar{\hhiwcip}{\wh}-\riskhbar{\hstar}{\wh}}_{(D)} + \underbrace{\riskhbar{\hstar}{\wh}-\riskhbar{\hstar}{w}}_{(C_2)}\\
        &\quad + \underbrace{\riskhbar{\hstar}{w} - \riskbar{\hstar}{w}}_{(B_2)} + \underbrace{\riskbar{\hstar}{w} - \risktar{\hstar}}_{(A_2)}.
    \end{align*}
    \paragraph{Term $A_1$ and $A_2$} From Eq.~\eqnref{R_minus_Rhw} in the proof of \thmref{IW-CIP_gen_bound}, we have
    \begin{align}\label{eqn:termA}
    \begin{split}
        A_1&=\risktar{\hhiwcip}- \riskbar{\hhiwcip}{w} \leq \Pssi{\phihiwcip}, \\
        A_2&=\riskbar{\hstar}{w} - \risktar{\hstar} =0.
    \end{split}
    \end{align}
    The second equation follows from the definition of $\hstar=\gstar\circ\phistar$, i.e., $\phistar$ is a conditionally invariant feature mapping and satisfies $\Pssi{\phistar}=0$.


    \paragraph{Term  $B_1$ and $B_2$} For any $h=g\circ \phi$, we have $\abss{\wmcoord{\Ysrc}\ones{h(\Xsrc)\neq \Ysrc}}\leq \vecnorm{\wm}{\infty}$. Then the generalization bound based on Rademacher complexity (see e.g.~\cite[Theorem 4.10]{wainwright2019high}) shows that, for the $m$-th source domain, the bound
    \begin{align*}
        \sup_{\substack{h=g\circ \phi\\g\in\Gset,\phi\in\Phi}}\abss{\risksrcw{h}{\wm} - \riskhsrcw{h}{\wm} } &\leq 2\rad{\nsrc,\Psrc}{\Hset(\wm, \Gset,\Phi)}\\
        &\hspace{5em}+ \vecnorm{\wm}{\infty}\sqrt{\frac{2\log(M/\delta)}{\nsrc}},
    \end{align*}
    holds with probability at least $1-\delta/M$. Using the union bound over each $m\in\{1,\ldots,M\}$, we obtain that with probability at least $1-\delta$,
    \begin{align}\label{eqn:termB12}
    \abss{B_1}+\abss{B_2} &\leq 2\sup_{\substack{h=g\circ\phi,\\g\in\Gset,\phi\in\Phi}} {\frac 1M} \sum_{m=1}^{M}\abss{ \risksrcw{h}{\wm} - \riskhsrcw{h}{\wm} } \nonumber\notag\\
    &\leq \max_{m=1,\ldots,M} \brackets{4\rad{\nsrc,\Psrc}{\Hset(\wm, \Gset,\Phi)} + 2\vecnorm{\wm}{\infty}\sqrt{\frac{2\log(M/\delta)}{\nsrc}} }.
    \end{align}

    \paragraph{Term $C_1$ and $C_2$} 
    For the $m$-th source domain, let us define
    \begin{align*}
        \widehat{\ell}^{(m)}_j = \sum_{k=1}^{n}\ones{\Ysrc_k=j}\ones{h(\Xsrc_k)\neq \Ysrc_k}.
    \end{align*}
    Then $\norm{\widehat{\ell}^{(m)}}_1 \leq n$ and so for any $h=g\circ\phi$,
    \begin{align*}
        \abss{ \riskhsrcw{h}{\wm} - \riskhsrcw{h}{\whm} }&= \abss{ \frac{1}{n}\sum_{k=1}^{n}\left(\wmcoord{\Ysrc_k} - \whmcoord{\Ysrc_k}\right)\ones{h(\Xsrc_k)\neq\Ysrc_k} }\\
        &= \abss{ \frac{1}{n}\sum_{j=1}^{L} \left(\wmcoord{j} - \whmcoord{j} \right)\widehat{\ell}^{(m)}_j } \\
        &\leq \frac{1}{n}\vecnorm{\wm - \whm}{\infty} \vecnorm{\widehat{\ell}^{(m)}}{1} \leq \vecnorm{\wm - \whm}{\infty}.
    \end{align*}
    It follows that
    \begin{align}\label{eqn:termC12}
    \abss{C_1}+\abss{C_2}&\leq 2\sup_{\substack{h=g\circ\phi,\\g\in\Gset,\phi\in\Phi}}{\frac 1M} \sum_{m=1}^{M}\abss{ \riskhsrcw{h}{\wm} - \riskhsrcw{h}{\whm} } \notag\\
    &\leq {\frac 2M} \sum_{m=1}^{M}\vecnorm{\wm - \whm}{\infty}.
    \end{align}

    \paragraph{Term $D$}
    By definition of $\hhiwcip$ from Eq.~\eqnref{finite_CIP_LabelCorr_CIP}, we know that
    \[\riskhbar{\hhiwcip}{\wh} + \penaltyiwcip{\phihiwcip} \leq  \riskhbar{\hstar}{\wh} + \penaltyiwcip{\phistar},\]
    and so rearranging terms,
    \begin{align}\label{eqn:termD}
    D = \riskhbar{\hhiwcip}{\wh} - \riskhbar{\hstar}{\wh} \leq \penaltyiwcip{\phistar} - \penaltyiwcip{\phihiwcip}\leq \penaltyiwcip{\phistar}.
    \end{align}
    The theorem then follows from combining Eq.~\eqnref{termA}, \eqnref{termB12}, \eqnref{termC12}, and \eqnref{termD}.

\subsection{Proof of \propref{iwcip_phistar}}\label{app:iwcip_phistar}

    For a family of functions $\Gset$ mapping $\Zset$ to $[0,1]$ and a distribution $\distriP$ on $\Zset$, let $\distriPhat_n$ denote the empirical distribution of $n$ $\iid$ samples $(Z_1,Z_2,\ldots,Z_n)$ from $\distriP$. The standard Rademacher complexity theory (see e.g.~\cite[ Theorem 4.10]{wainwright2019high}) states that for any $\delta>0$, with probability at least $1-\delta$, we have
    \begin{align*}
        \sup_{g\in\Gset}\abss{\Ep{Z\sim\distriP}{g(Z)}-\Ep{Z\sim\distriPhat_n}{g(Z)}}\leq 2\rad{n,\distriP}{\Gset} + \sqrt{\frac{\log (2/\delta)}{2n}}.
    \end{align*}
    By definition of $\Gset$-divergence, with probability $1-\delta$ we have
    \begin{align}\label{eqn:divG_rademacher}
    \Gdiv{\distriP}{\distriPhat_n}&=\sup_{g\in\Gset}\max_{y=1,2,\ldots,L}\abss{\Ep{Z\sim\distriP}{\ones{g(Z)=y}} - \Ep{Z\sim\distriPhat_n}{\ones{g(Z)=y}}}\notag\\
    &= \max_{y=1,2,\ldots,L} \sup_{g\in\Gset}\abss{\Ep{Z\sim\distriP}{\ones{g(Z)=y}} - \Ep{Z\sim\distriPhat_n}{\ones{g(Z)=y}}}\notag\\
    &\leq \max_{y=1,2,\ldots,L} \parenth{2\rad{n,\distriP}{\Gsety} + \sqrt{\frac{\log(2L/\delta)}{2n}}},
    \end{align}
    where $\Gsety \defn \braces{\ones{g(x)=y}, g\in\Gset}$. Then we can bound
    \begin{align}\label{eqn:divG_triangular}
    &\quad \Gdiv{\Phcondsrc{\phistar(X)}{Y=y}}{\Phcondsrcmp{\phistar(X)}{Y=y}}\notag\\
    &\overset{(i)}{\leq} \Gdiv{\Phcondsrc{\phistar(X)}{Y=y}}{\Pcondsrc{\phistar(X)}{Y=y}}+ \Gdiv{\Pcondsrc{\phistar(X)}{Y=y}}{\Pcondsrcmp{\phistar(X)}{Y=y}} \notag\\
    &\hspace{3em}+ \Gdiv{\Pcondsrcmp{\phistar(X)}{Y=y}}{\Phcondsrcmp{\phistar(X)}{Y=y}}\notag\\
    &\overset{(ii)}{=} \Gdiv{\Phcondsrc{\phistar(X)}{Y=y}}{\Pcondsrc{\phistar(X)}{Y=y}} +  \Gdiv{\Phcondsrcmp{\phistar(X)}{Y=y}}{\Pcondsrcmp{\phistar(X)}{Y=y}},
    \end{align}
    where in step $(i)$ we apply triangular inequality, and in step $(ii)$ we use the fact that $\phistar$ is a conditionally invariant feature mapping across source domains to cancel the middle term. Applying union bound, we get
    \begin{align*}
        &\penaltyiwcip{\phistar} = \frac{\lamiwcip}{LM^2}\sum_{y=1}^L\sum_{m\neq m'}  \Gdiv{\Phcondsrc{\phistar(X)}{Y=y}}{\Phcondsrcmp{\phistar(X)}{Y=y}}\\
        &\overset{(i)}{\leq} \frac{2\lamiwcip}{LM}\sum_{y=1}^L\sum_{m=1}^M  \Gdiv{\Pcondsrc{\phistar(X)}{Y=y}}{\Phcondsrc{\phistar(X)}{Y=y}}\\
        &\overset{(ii)}{\leq} \frac{4\lamiwcip}{LM}\cdot  \sum_{y=1}^L\sum_{m=1}^M \max_{y'=1,2,\ldots,L}\rad{\nsrc, \Pcondsrc{\phistar(X)}{Y=y}}{\Gsetyp} + \frac{2\lamiwcip}{M}\sum_{m=1}^M\sqrt{\frac{\log (2LM/\delta)}{2\nsrc}}\\
        &\leq 2\lamiwcip \parenth{2\max_{\substack{m=1,2,\ldots,M\\y=1,2,\ldots,L\\y'=1,2,\ldots,L}}\rad{\nsrc, \Pcondsrc{\phistar(X)}{Y=y}}{\Gsetyp} + \max_{m=1,2,\ldots,M}\sqrt{\frac{\log (2LM/\delta)}{2\nsrc}}},
    \end{align*}
    with probability at least $1-\delta$. Here step $(i)$ uses Eq.~\eqnref{divG_triangular} and step $(ii)$ uses Eq.~\eqnref{divG_rademacher}. This completes the proof of the proposition.

\subsection{Another bound on $\penaltyiwcip{\phistar}$ under mean squared distance}\label{app:iwcip_phistar2}
In~\propref{iwcip_phistar} we established a finite-sample bound for $\penaltyiwcip{\phistar}$ when $\Gset$-divergence is used. If the IW-CIP penalty is defined by the squared mean distance instead, we provide another explicit bound for $\penaltyiwcip{\phistar}$ that converges to zero as sample size goes to infinity. To show this, we require additional assumptions on the data generation model and the class of feature mappings $\Phi$.
\begin{proposition}\label{prop:iwcip_phistar_squared_distance}
    Suppose that source and target data are generated under \assumpref{linear_cic}. Let $\Phi=\{\phi(x)=Ax+b, A\in\real^{q\times p}, b\in\real^{q}\}$ be the linear class of feature mappings from $\real^p$ to $\real^q$ $(q\leq p)$, and let the squared mean distance be the distributional distance in IW-CIP penalty Eq.~\eqnref{iwcip_penalty_defn}. Suppose $\phi\in\Phi$ is a conditionally invariant feature mapping such that $A\Sigma A^\top$ is non-singular. For any $\delta > 0$ satisfying $\min_{m,y}\nsrc\probsrc{y} \geq 2\log(2ML/\delta)$, we have
    \begin{align}\label{eqn:penaltyiwcip_mean_dist_bound}
    \penaltyiwcip{\phi}&\leq  \frac{12\lamiwcip \cdot \lammax(A\Sigma{A}^\top)\cdot \parenth{p+8\log(2ML/\delta)}}{\underset{\substack{m=1,2,\ldots, M\\y=1,2,\ldots,L}}{\min} \ \nsrc\probsrc{y} - 2\log(2ML/\delta)},
    \end{align}
    with probability at least $1-\delta$.\footnote{The probability is with respect to the randomness of $(\Xsrc_k, \Ysrc_k)\iidsim\Psrc$, $m=1,\ldots,M, k=1,\ldots,\nsrc$; see Eq.~ \eqnref{dataset_split}.} Specifically, the statement holds for the optimal conditionally invariant mapping $\phistar$ (see Eq.~\eqnref{optimal_cond_inv_estimator}).
\end{proposition}
\begin{proof}
    The exact from of IW-CIP penalty under squared mean distance is
    \begin{align*}
        \penaltyiwcip{\phi}&\defn\frac{\lamiwcip}{LM^2}\sum_{y=1}^L\sum_{m\neq m'} \Bigg\|\meanhat{m}{y}-\meanhat{m'}{y}\Bigg\|_2^2,
    \end{align*}
    where
    \begin{align*}
        \meanhat{m}{y}=\frac{1}{\nsrc_y}\sum_{\substack{k=1\\ \Ysrc_k=y}}^{\nsrc}\phi(\Xsrc_k),
    \end{align*}
    is the empirical mean of $\phi(X)$ on the $m$-th source domain. Since $\phi$ is conditionally invariant, from Eq.~\eqnref{prop_a_conditional_dist} we know that for any $y\in\braces{1,2,\ldots,L}$ and for any $m\neq m'$,
    \begin{align*}
        A\Pcip\vm{m}{y} = A\Pcip\vm{m'}{y}.
    \end{align*}
    In the finite-sample case, when $\Ysrc_k=\Ytar_k=y$, we can write
    \begin{align*}
        \phi(\Xsrc_k)=A \Xsrc_k+b&=A \parenth{\fsrcone(y)+\Pcip\vm{m}{y}+\noisesrc_k} + b,\\
        \phi(\Xtar_k)=A \Xtar_k+b&=A \parenth{\fsrcone(y)+\Pcip\vtar{y}+\noisetar_k} + b,
    \end{align*}
    and
    \begin{align}
        \meanhat{m}{y}-\meanhat{m'}{y}&= \frac{1}{\nsrc_y}\sum_{\substack{k=1\\ \Ysrc_k=y}}^{\nsrc}\phi(\Xsrc_k)-\frac{1}{\nsrcm{m'}_y}\sum_{\substack{k=1\\ \Ysrcm{m'}_k=y}}^{\nsrcm{m'}}\phi(\Xsrcm{m'}_k)\notag\\
        &= A \left(\underbrace{\frac{1}{\nsrc_y}\sum_{\substack{k=1\\ \Ysrc_k=y}}^{\nsrc}\noisesrc_{k}}_{e^{(m)}_y} - \underbrace{\frac{1}{\nsrcm{m'}_y}\sum_{\substack{k=1\\ \Ysrcm{m'}_k=y}}^{\nsrcm{m'}}\noisesrcm{m'}_k}_{e_y^{(m')}}\right), \label{eqn:prop_sem_c_emy_defn}
    \end{align}
    where $\nsrcm{m}_y$ is the number of samples with label $y$ in the $m$-th domain. Since $n_y^{(m)}$ follows a multinomial distribution, we conclude from binomial (multinomial) tail probability bound (e.g.~\cite[Theorem 3.2]{chung2006concentration}) that for all ${t_y^{(m)}} >0$,
    \begin{align*}
        \prob{\nsrc_{y}\geq \nsrc \probsrc{y}(1-{t_y^{(m)}})}&\geq 1-\exp\parenth{-\frac{\nsrc \probsrc{y} {t_y^{(m)}}^2}{2}}.
    \end{align*}
    When $\nsrc_{y}=k>0$ is fixed, we have $Ae^{(m)}_y\sim\Normal(0,\frac 1 {k}A\Sigma A^\top)$, and therefore the standard concentration bound for Gaussian variables (e.g.~\cite[Example 2.11]{wainwright2019high}) implies that for all $t>0$,
    \begin{align*}
        \probc{\vecnorm{Ae^{(m)}_y}{2}^2\leq\frac 1 {k}\lammax\parenth{A\Sigma A^\top}\cdot p\parenth{1+t}}{\nsrc_y=k}\geq 1-e^{-p\min\{t,t^2\}/8}.
    \end{align*}
    Next we combine two bounds above into a single bound for $\vecnorm{Ae^{(m)}_y}{2}^2$. Given any $t_y^{(m)}>0$ and $t>0$, we have
    \begin{align*}
        &\quad\prob{\vecnorm{Ae^{(m)}_y}{2}^2\leq\frac{\lammax\parenth{A\Sigma A^\top}\cdot p\parenth{1+t}}{\nsrc \probsrc{y}(1-{t_y^{(m)}})}}\\
        &=\sum_{k=0}^{\nsrc}\probc{\vecnorm{Ae^{(m)}_y}{2}^2\leq\frac{\lammax\parenth{A\Sigma A^\top}\cdot p\parenth{1+t}}{\nsrc \probsrc{y}(1-{t_y^{(m)}})}}{\nsrc_y=k}\prob{\nsrc_y=k}\\
        &\geq \sum_{k=\nsrc \probsrc{y}(1-{t_y^{(m)}})}^{\nsrc} \probc{\vecnorm{Ae^{(m)}_y}{2}^2\leq\frac{\lammax\parenth{A\Sigma A^\top}\cdot p\parenth{1+t}}{\nsrc \probsrc{y}(1-{t_y^{(m)}})}}{\nsrc_y=k}\prob{\nsrc_y=k}\\
        &\geq \sum_{k=\nsrc \probsrc{y}(1-{t_y^{(m)}})}^{\nsrc} \probc{\vecnorm{Ae^{(m)}_y}{2}^2\leq \frac{1}{k}\lammax\parenth{A\Sigma A^\top}\cdot p\parenth{1+t}}{\nsrc_y=k}\prob{\nsrc_y=k}\\
        &\geq \sum_{k=\nsrc \probsrc{y}(1-{t_y^{(m)}})}^{\nsrc} \parenth{1-e^{-p\min\{t,t^2\}/8}}\prob{\nsrc_y=k}\\
        &\geq \parenth{1-e^{-p\min\{t,t^2\}/8}}\parenth{1-e^{-{\nsrc \probsrc{y} {t_y^{(m)}}^2}/{2}}}\\
        &\geq 1-e^{-p\min\{t,t^2\}/8} - e^{-{\nsrc \probsrc{y} {t_y^{(m)}}^2}/{2}}.
    \end{align*}

    Take ${t_y^{(m)}} = \sqrt{2\log(2ML/\delta)/(\nsrc\probsrc{y})}$ and $t=8\log(2ML/\delta)/p+\sqrt{8\log(2ML/\delta)/p}$. Then with probability at least $1-\delta/ML$, we have
    \begin{align*}
        \vecnorm{Ae^{(m)}_y}{2}^2 &\leq \frac{\lammax\parenth{A\Sigma A^\top}\parenth{p+8\log(2ML/\delta) +\sqrt{8p\log(2ML/\delta)}}}{\nsrc\probsrc{y}-\sqrt{\nsrc\probsrc{y}\cdot 2\log\parenth{2ML/\delta}}}\\
        &\leq  \frac{3\lammax\parenth{A\Sigma A^\top}\parenth{p+8\log(2ML/\delta) }}{\nsrc\probsrc{y} - 2\log(2ML/\delta)},
    \end{align*}
    where we apply the inequality $\sqrt{ab}\leq (a+b)/2$ for $a,b>0$ in the last step. Applying union bound, we get
    \begin{align*}
        \penaltyiwcip{\phi} &=\frac{\lamiwcip}{LM^2}\sum_{y=1}^L\sum_{m\neq m'}\Bigg\|\meanhat{m}{y}-\meanhat{m'}{y}\Bigg\|_2^2\\
        &=\frac{\lamiwcip}{LM^2}\sum_{y=1}^L\sum_{m\neq m'}\vecnorm{Ae^{(m)}_y-Ae^{(m')}_y}{2}^2\\
        &\leq \frac{2\lamiwcip}{LM^2}\sum_{y=1}^L\sum_{m\neq m'}\parenth{\vecnorm{Ae^{(m)}_y}{2}^2+\vecnorm{Ae^{(m')}_y}{2}^2}\\
        &\leq 4\lamiwcip\cdot \underset{\substack{m=1,2,\ldots, M\\y=1,2,\ldots,L}}{\max}\vecnorm{Ae^{(m)}_y}{2}^2\\
        &\leq \frac{12\lamiwcip \cdot \lammax(A\Sigma{A}^\top)\cdot \parenth{p+8\log(2ML/\delta)}}{\underset{\substack{m=1,2,\ldots, M\\y=1,2,\ldots,L}}{\min} \ \nsrc\probsrc{y} - 2\log(2ML/\delta)},
    \end{align*}
    with probability at least $1-\delta$, proving the result.
\end{proof}

\subsection{Proof of \propref{label_shift_estimation_error}}\label{app:label_shift_estimation_error}

    We have that for any $i\in\braces{1,2,...,L}$,
    \begin{align*}
        \prob{\hhcip(\Xtar)=i} &= \sum_{j=1}^L\probc{\hhcip(\Xtar)=i}{\Ytar=j}\cdot \prob{\Ytar=j}\\
        &= \sum_{j=1}^L\parenth{\probc{\hhcip(\Xsrc)=i}{\Ysrc=j}+\delta_{i,j}} \cdot \prob{\Ytar=j},
    \end{align*}
    where $\delta_{i,j} = \probc{\hhcip(\Xtar)=i}{\Ytar=j}-\probc{\hhcip(\Xsrcone)=i}{\Ysrcone=j}$.
    Writing the above equation in matrix form, we have
    \begin{align}\label{eqn:label_shift_lin_system_finite}
    \mu_{\hhcip} = C_{\hhcip}^{\tagm{m}} \wm + \Delta \mu^\tagtar,
    \end{align}
    where $\Delta$ is a matrix with $\delta_{i,j}$ being its $(i,j)$-th element, and $\mu^\tagtar$ denotes the target label distribution.
    Comparing with the linear system in Eq.~\eqnref{label_shift_equation_matrix}, Eq.~\eqnref{label_shift_lin_system_finite} has an additional term on the right-hand side because the finite-sample CIP $\hhcip$ does not guarantee exact conditional invariance.
    Since the importance weight $\whm$ is obtained by solving the finite-sample version
    \begin{align}\label{eqn:label_shift_lin_system_finite2}
    \muh_{\hhcip} = \widehat{C}_{\hhcip}^{\tagm{m}}\whm,
    \end{align}
    combining Eq.~\eqnref{label_shift_lin_system_finite} and Eq.~\eqnref{label_shift_lin_system_finite2}, we have
    \begin{align}\label{eqn:w_error}
    \begin{split}
        &\underbrace{\quad \mu_{\hhcip} - \muh_{\hhcip}}_{(A_1)}\\
        &= C_{\hhcip}^{\tagm{m}}\parenth{\wm-\whm} + \underbrace{\parenth{C_{\hhcip}^{\tagm{m}} - \widehat{C}_{\hhcip}^{\tagm{m}}}}_{(A_2)}\whm + \underbrace{\Delta \mu^\Diamond}_{(A_3)}.
    \end{split}
    \end{align}
    Now we bound $(A_1), (A_2)$, and $(A_3)$ separately. For bouding $(A_1)$ and $(A_2)$, we closely follow the proof of Theorem 3 in~\cite{lipton2018detecting}.

    \paragraph{Term ($A_1$)} By Hoeffding's inequality, we have
    \begin{align*}
        \prob{\abss{A_{1,i}}>t}&=\prob{\abss{\frac 1 {\nttar}\sum_{k=1}^{\nttar} \ones{\hhcip(\widetilde{X}^\tagtar_k)=i}-\EE{\ones{\hhcip({X}^\tagtar)=i}}}>t}\leq 2\exp\braces{-2\nttar t^2},
    \end{align*}
    for all $t > 0$. Therefore, for any $t_1>0$, we have
    \begin{align*}
        \prob{\norm{A_1}_2>t_1}&\leq \prob{\exists i\in\braces{1,...,L}, \abss{A_{1,i}}>\frac{t_1}{\sqrt{L}}}\\
        &\leq \sum_{i=1}^L\prob{\abss{A_{1,i}}>\frac{t_1}{\sqrt{L}}}\\
        &\leq 2L\cdot \exp\braces{-\frac{2\nttar t_1^2}{L}}.
    \end{align*}

    \paragraph{Term ($A_2$)} Note that
    \begin{align*}
        C_{\hhcip}^{\tagm{m}} - \widehat{C}_{\hhcip}^{\tagm{m}}
        &= \frac{1}{\ntsrc} \sum_{k=1}^{\ntsrc} e_{\hhcip(\widetilde{X}^\tagm{m}_k)}e_{\widetilde{Y}^\tagm{m}_k}^{\top} - \EE{e_{\hhcip(\Xsrc)}{e}_{\Ysrc}^{\top}},
    \end{align*}
    where $e_i$ denotes the zero vector with only the $i$-th element being 1. Let
    \begin{align*}
        Z_k = e_{\hhcip(\widetilde{X}^\tagm{m}_k)}e_{\widetilde{Y}^\tagm{m}_k}^{\top}  - \EE{e_{\hhcip(\Xsrc)}{e}_{\Ysrc}^{\top}}.
    \end{align*}
    Then it is straightforward to see that
    \begin{align*}
        \begin{cases}
            \EE{Z_k}=0;\\
            \norm{Z_k}_{2}\leq 2;\\
            \max\braces{\norm{\EE{Z_kZ_k^\top}}_{2}, \norm{\EE{Z_k^\top Z_k}}_{2}}\leq 1.
        \end{cases}
    \end{align*}
    By  the matrix Bernstein inequality (see e.g.~\cite[Theorem 6.1.1]{tropp2015introduction}), we have
    \begin{align*}
        \prob{\norm{A_2}_{2}\geq t_2} \leq 2L\exp\braces{-\frac{\ntsrc t_2^2/2}{1+2t_2/3}},
    \end{align*}
    for all $t_2 > 0$.

    \paragraph{Term ($A_3$)}
    By definition of $\Gdiv{\cdot}{\cdot}$ and $\Pssi{\phihcip}$, each element of $A_3$ satisfies
    \begin{align*}
        \abss{A_{3,i}} &\leq \sum_{j=1}^L\abss{\delta_{i,j}}\cdot \prob{\Ytar=j}\\
        &\leq \sum_{j=1}^L \Gdiv{\Pcondtar{\phihcip(X)}{Y=j}}{\Pcondsrc{\phihcip(X)}{Y=j}}\cdot  \prob{\Ytar=j}\leq \Pssi{\phihcip}.
    \end{align*}
    Hence we have
    \begin{align*}
        \norm{A_3}_{2}=\sqrt{\sum_{i=1}^L\abss{A_{3,i}}^2}\leq \sqrt{L}\Pssi{\phihcip}.
    \end{align*}
    Applying these bounds to Eq.~\eqnref{w_error} and using the assumption that $C_{\hhcip}^{\tagm{m}}$ has condition number $\condi^{\tagm{m}}$ (i.e., the smallest eigenvalue $\geq 1/\condi^{\tagm{m}}$ because $1$ is an eigenvalue of the matrix), we get
    \begin{align*}
        \vecnorm{\wm-\whm}{2}&\leq \condi_m\parenth{t_1 + t_2\cdot\vecnorm{\whm}{2} + \sqrt{L}\Pssi{\phihcip}}\\
        &\leq \condi_m\parenth{t_1 + t_2\parenth{\vecnorm{\wm-\whm}{2}+\vecnorm{\wm}{2}} + \sqrt{L}\Pssi{\phihcip}},
    \end{align*}
    with probability at least $1-2L\exp\parenth{-2\nttar t_1^2/L}-2L\exp\parenth{-3\ntsrc t_2^2/(6+4 t_2)}$. Finally, let $t_1=\sqrt{L\log (4L/\delta)/2\nttar }$ and $t_2=\sqrt{3\log (4L/\delta) / \ntsrc}$. Using the assumption $\ntsrc\geq 12\condi_m^2\log (4L/\delta)$, we get
    \begin{align*}
        \vecnorm{\wm-\whm}{2}&\leq 2\condi_m\parenth{\sqrt{L}\Pssi{\phihcip} + \sqrt{\frac{L\log (4L/\delta)}{2\nttar }} + \sqrt{\frac{3\log (4L/\delta) }{\ntsrc}}\vecnorm{\wm}{2}}
    \end{align*}
    with probability at least $1-\delta$. Combining the results then proves the proposition.


\subsection{Proof of \propref{linear_sem_example}}\label{app:linear_sem_example}

        Using assumption~\eqnref{linear_sem_prop_finite_assump}, for any positive semidefinite matrix $S\in\real^{d_y\times d_y}$, we have
        \begin{align}\label{eqn:prop_assumption_result}
         \parenth{\vtar{y}-\vm{m}{y}}^{\top} S\parenth{\vtar{y}-\vm{m}{y}}&=\trace\parenth{\parenth{\vtar{y}-\vm{m}{y}}^{\top}S^{\frac12}S^{\frac12}\parenth{\vtar{y}-\vm{m}{y}}}\notag\\
        &=\trace\parenth{S^{\frac12}\parenth{\vtar{y}-\vm{m}{y}}\parenth{\vtar{y}-\vm{m}{y}}^\top S^{\frac12}}\notag\\
        &\leq \trace\parenth{S^{\frac12} \parenth{\frac{\kk_y}{M-1}\sum_{m=2}^M\vm{m}{y}(\vm{m}{y})^\top}  S^{\frac12}}\notag\\
        &=\frac{\kk_y}{M-1}\sum_{m=2}^M\trace\parenth{S^{\frac12}\vm{m}{y}(\vm{m}{y})^\top S^{\frac12}}\notag\\
        &=\frac{\kk_y}{M-1}\sum_{m=2}^M {\vm{m}{y}}^{\top}S\vm{m}{y}.
        \end{align}
        Inequality above uses the fact that $\trace(A)\leq\trace(B)$ for positive semi-definite matrix $A,B$ satisfying $A\preceq B$.
        Given that $\noisesrc, \noisetar \iidsim \Normal(0,\Sigma)$, for any $y\in\braces{1,2,\ldots,L}$ and $m\in\braces{1,2,\ldots, M}$ we have
        \begin{align}
            \begin{split}\label{eqn:prop_a_conditional_dist}
                \phi(\Xsrc) \mid \Ysrc=y &\sim \Normal\parenth{A \parenth{\fsrcone(y)+\Pcip\vm{m}{y}}+b, A  \Sigma  A^\top},\\
                \phi(\Xtar) \mid \Ytar=y &\sim \Normal\parenth{A \parenth{\fsrcone(y)+\Pcip\vtar{y}}+b, A  \Sigma  A^\top}.
            \end{split}
        \end{align}
        Note that $\vm{1}{y}=0$ by \assumpref{linear_cic}, so we can further simplify terms in Eq.~\eqnref{varcip_defn} as
        \begin{align*}
            \penaltycipm{m}{y} = AP_y\parenth{\vm{m}{y}-\vm{1}{y}}=AP_y\vm{m}{y}\ \text{ and } \  \varcip = A\Sigma A^{\top}.
        \end{align*}
        With results above in hand, we can derive
        \begin{align}\label{eqn:G_div_bound}
            &\quad \Gdivsq{\Pcondtar{\phi(X)}{Y=y}}{\Pcondsrc{\phi(X)}{Y=y} }\notag\\
            &\leq \sup_{g\in\Gset}\max_{y'=1,2,\ldots,L}\abss{\Ep{Z\sim\Pcondtar{\phi(X)}{Y=y}}{\ones{g(Z)=y'}} - \Ep{Z\sim\Pcondsrc{\phi(X)}{Y=y} }{\ones{g(Z)=y'}}}\notag\\
            &\leq \sup_{g\in\Gset}\sum_{y'=1}^L\abss{\Ep{Z\sim\Pcondtar{\phi(X)}{Y=y}}{\ones{g(Z)=y'}} - \Ep{Z\sim\Pcondsrc{\phi(X)}{Y=y} }{\ones{g(Z)=y'}}}\notag\\
            &=4\sup_{g\in\Gset}\tvdistsq{\Pcondtar{g(\phi(X))}{Y=y}}{\Pcondsrc{g(\phi(X))}{Y=y} }\notag\\
            &\overset{(i)}{\leq} 2\sup_{g\in\Gset}\KLdiv{\Pcondtar{g(\phi(X))}{Y=y}}{\Pcondsrc{g(\phi(X))}{Y=y}}\notag\\
            &\overset{(ii)}{\leq} 2\KLdiv{\Pcondtar{\phi(X)}{Y=y}}{\Pcondsrc{\phi(X)}{Y=y}}\notag\\
            &\overset{(iii)}{=}\parenth{A \Pcip\parenth{\vtar{y}-\vm{m}{y}}}^\top\parenth{A\Sigma A^\top}^{-1}\parenth{A \Pcip\parenth{\vtar{y}-\vm{m}{y}}}\notag\\
            &\overset{(iv)}{\leq}\frac{\kk_y}{M-1}\sum_{m=2}^M\parenth{A \Pcip\vm{m}{y}}^\top\parenth{A\Sigma A^\top}^{-1}\parenth{A \Pcip\vm{m}{y}}\notag\\
            &\overset{(v)}{=}2\kk_y\Pihatphi{y}.
        \end{align}
        Here step~$(i)$ applies Pinsker's inequality and step~$(ii)$ applies data processing inequality. Step $(iii)$ applies the KL divergence formula between two Gaussians, i.e.,
        \begin{align*}
            \KLdiv{\Normal(\mu_1,\Sigma_1)}{\Normal(\mu_1, \Sigma_2)}&= \frac12 \parenth{\log\frac{\abss{\Sigma_2}}{\Sigma_1}}-p + \text{Tr}\parenth{\Sigma_2^{-1}\Sigma_1}+(\mu_1-\mu_2)^{\top}\Sigma_2^{-1}(\mu_2-\mu_1).
        \end{align*}
        Step~$(iv)$ uses Eq.~\eqnref{prop_assumption_result} with $S=P_y^\top A^\top (A\Sigma A^\top)^{-1} A P_y$ and the last step $(vi)$ follows from definition~\eqnref{pihat_defn}.
        Now the proposition follows by applying Eq.~\eqnref{G_div_bound} to the definition of $\Psi_{\Gset, \phi}$ in Eq.~\eqnref{Psi_phi_defn}.

\subsection{Number of source domains needed for controlling the deviation from conditional invariance}\label{app:linear_sem_example_cor}

To validate condition~\eqnref{linear_sem_prop_finite_assump} in~\propref{linear_sem_example}, we prove the following lemma which states that when the perturbations are generated randomly, the required number of source domains is approximately linear in the dimension of perturbations.

\begin{lemma}\label{lem:linear_sem_example}
    Under \assumpref{linear_cic}, further assume that perturbations are generated according to $\vm{m}{y},\vtar{y}\iidsim\Normal(0, \tau_y^2 \Ind_{d_y})$. For any $\delta\in(0,1)$, if the number of source domains $M$ satisfies
    \begin{align*}
        M\geq 1+r\parenth{\max_y\sqrt{d_y}+\sqrt{2\log(2L/\delta)}}^2,
    \end{align*}
    for some constant $r>1$, then with probability at least $1-\delta$,\footnote{The probability is with respect to the randomness of perturbations $\vm{m}{y},\vtar{y}\iidsim\Normal(0, \tau_y^2 \Ind_{d_y})$.} condition~\eqnref{linear_sem_prop_finite_assump} holds with $\kk_y=c_r (M+\log M)$ for some constant $c_r>0$. 
\end{lemma}


\begin{proof}
    We verify that when $\vm{m}{y},\vtar{y}\iidsim\Normal(0, \tau_y^2 \Ind_{d_y})$, condition~\eqnref{linear_sem_prop_finite_assump} holds with high probability. Sincet $M\geq \max_{y=1,2,\ldots,L}\{d_y\}+1$, Theorem 6.1 and Example 2.11 in~\cite{wainwright2019high} show that given $y\in\{1,2,\ldots,L\}$, for any $t_0\in(0,1), t_1,\ldots,t_M>0$ we have
    \begin{align*}
        \prob{\lammin\parenth{\frac 1{M-1}\sum_{m=2}^M\vm{m}{y}{\vm{m}{y}}^\top}\geq \tau_y^2\parenth{1-t_0-\sqrt{\frac {d_y} {M-1}}}}&\geq 1-e^{-(M-1)t_0^2/2},
    \end{align*}
and
    \vspace{-1em}
    \begin{align*}
        &\prob{\vecnorm{\vtar{y}-\vm{1}{y}}{2}^2\leq \tau_y^2  d_y \parenth{1+t_1}}\geq 1-e^{-d_y\min\{t_1,t_1^2\}/8},\\
        &\prob{\vecnorm{\vtar{y}-\vm{m}{y}}{2}^2\leq 2\tau_y^2  d_y\parenth{1+t_m}}\geq 1-e^{-d_y\min\{t_m,t_m^2\}/8}, \ \ \ \forall m\in\braces{2,\cdots, M}.
    \end{align*}
    Take $t_0=\sqrt{\frac{2}{M-1}\log\frac {2L} {\delta}}$ and $t_m=\frac 8 {d_y}\log(\frac{2ML}{\delta}) + \sqrt{\frac 8 {d_y}\log(\frac{2ML}{\delta})}$ for all $y$ such that $d_y>0$. Applying union bound, the following bound holds for all $y\in\braces{1,2,\ldots,L}$ and $m\in\braces{1,2,\ldots,M}$ with probability greater than $1-\delta$:
    \begin{align*}
        &\quad \lammax\parenth{\parenth{\vtar{y}-\vm{m}{y}}\parenth{\vtar{y}-\vm{m}{y}}^{\top}}=\vecnorm{\parenth{\vtar{y}-\vm{m}{y}}}{2}^2\\
        &\leq 2\tau_y^2\parenth{d_y+8\log (2ML/\delta) + \sqrt{8d_y\log(2ML/\delta)}}\\
        &\leq 3\tau_y^2\parenth{d_y+8\log (2ML/\delta)}\\
        &\leq \frac{3 \parenth{d_y+8\log (2ML/\delta)}}{1-\sqrt{{2}\parenth{\log(2L /{\delta})}/(M-1)}-\sqrt{{d_y}/ (M-1)}}\lammin\parenth{\frac 1{M-1}\sum_{m=2}^M\vm{m}{y}\vm{m}{y}^\top}\\
        &\leq \frac{3r}{r-1} \parenth{d_y+8\log \parenth{\frac{2ML}{\delta}}}\lammin\parenth{\frac 1{M-1}\sum_{m=2}^M\vm{m}{y}\vm{m}{y}^\top }\\
        &\leq c_r(M+\log M) \cdot \lammin\parenth{\frac 1{M-1}\sum_{m=2}^M\vm{m}{y}\vm{m}{y}^\top},
    \end{align*}
    where we apply the lower bound on $M$ in the last two steps.
\end{proof}

\subsection{Proof of \propref{linear_sem_example2}}\label{app:linear_sem_example2}


    Let $\fc(1), \fc(2)\in\real^{p-d}$ denote the vector of the last $p-d$ coordinates of $\fsrcone(1), \fsrcone(2)\in\real^p$ respectively, which are conditionally invariant across source and target. We first prove that the target risks of $\horacle$ and $\hstar$ are given by
        \begin{align}\label{eqn:oracle_risk_diff}
            \begin{split}
                \risktar{\horacle} &= r_{\probtar{1},\sigma}\parenth{\vecnorm{\ftar(1)-\ftar(2)}{2}},\\
                \risktar{\hstar} &= r_{\probtar{1},\sigma}\Big(\vecnorm{\fc(1)-\fc(2)}{2}\Big),
            \end{split}
        \end{align}
    where $r_{\probtar{1},\sigma}(x)$ is a continuous decreasing function in $x$. The risk of the oracle classifier can be derived as follows. For simplicity we write $\mutar{1} = \ftar(1)$ and $\mutar{2} = \ftar(2)$. For any $\phi(x)=\beta^\top x+\beta_0$ with $\vecnorm{\beta}{2}=1$, we know that
    \begin{align*}
        \phi(\Xtar)\mid \Ytar = 1 &\sim \Normal\parenth{\beta^\top \mutar{1} + \beta_0, \sigma^2},\\
        \phi(\Xtar)\mid \Ytar = 2 &\sim \Normal\parenth{\beta^\top \mutar{2} + \beta_0, \sigma^2}.
    \end{align*}
    Then the target risk of $h= g\circ \phi$ can be computed by
    \begin{align}
        \risktar{h} &= \EE{\ones{h(\Xtar)\neq \Ytar}}\notag\\
        &=\probc{h(\Xtar)=2}{\Ytar=1}\prob{\Ytar=1}  \notag\\
        &\hspace{15em}+ \probc{h(\Xtar)=1}{\Ytar=2}\prob{\Ytar=2}\notag\\
        &=\probc{\phi(\Xtar)> 0}{\Ytar=1}\prob{\Ytar=1} \notag\\
        &\hspace{15em}+ \probc{\phi(\Xtar)\leq 0}{\Ytar=2}\prob{\Ytar=2}\notag\\
        &=\probtar{1}\cdot\frac 12\parenth{1 - \erf{\frac{\beta^\top\mutar{1}+\beta_0}{\sqrt{2\sigma^2}}}} + (1-\probtar{1})\cdot \frac 12\parenth{1 + \erf{\frac{\beta^\top\mutar{2}+\beta_0}{\sqrt{2\sigma^2}}}}\notag\\
        &= \frac12 \parenth{1- \probtar{1}\cdot \erf{\frac{\beta^\top\mutar{1}+\beta_0}{\sqrt{2\sigma^2}}} + (1-\probtar{1})\cdot \erf{\frac{\beta^\top\mutar{2}+\beta_0}{\sqrt{2\sigma^2}}}}\label{eqn:risktar_erf_exp},
    \end{align}
    where $\erf{x} = 2\int_0^{x}e^{-t^2}dt/\sqrt{\pi}$ is the Gauss error function. The optimization problem of $\horacle$ in Eq.~\eqnref{oracle_risk} becomes
    \begin{align}
        \begin{split}\label{eqn:oracel_optimization_beta}
            \beta_{\rm orcale}, \beta_{0, {\rm orcale}} = \argmin_{\beta, \beta_0}&\hspace{1em}-\probtar{1}\cdot \erf{\frac{\beta^\top\mutar{1}+\beta_0}{\sqrt{2\sigma^2}}} + (1-\probtar{1})\cdot \erf{\frac{\beta^\top\mutar{2}+\beta_0}{\sqrt{2\sigma^2}}},\\
            \text{s.t.} & \hspace{1em}\vecnorm{\beta}{2}=1.
        \end{split}
    \end{align}
    Setting the gradient of the objective function with respect to $\beta_0$ to $0$, we get
    \begin{align*}
        \beta_0 = \frac{\sigma^2}{\beta^\top(\mutar{1}-\mutar{2})}\log\frac{\probtar{1}}{1-\probtar{1}} - \frac 12\beta^\top(\mutar{1}+\mutar{2}).
    \end{align*}
    Plugging the expression of $\beta_0$ into Eq.~\eqnref{oracel_optimization_beta} and using the method of Lagrange multipliers, we arrive at
    \begin{align*}
        & \frac{\partial}{\partial \beta}\parenth{-\probtar{1}\cdot \erf{\frac{\beta^\top\mutar{1}+\beta_0}{\sqrt{2\sigma^2}}} + (1-\probtar{1})\cdot \erf{\frac{\beta^\top\mutar{2}+\beta_0}{\sqrt{2\sigma^2}}} + \lambda(\vecnorm{\beta}{2}^2-1)}=0\\
        \Longrightarrow &\;\; \;2\lambda \beta = \probtar{1}\frac {2} {\sqrt{\pi}}\exp\parenth{-\frac{(\beta^\top\mutar{1}+\beta_0)^2}{2\sigma^2}} \\
        &\hspace{10em}\times \parenth{\frac{1}{2\sqrt{2\sigma^2}}(\mutar{1}-\mutar{2})-\frac{\sqrt{\sigma^2/2} \log \frac{\probtar{1}}{1-\probtar{1}}}{(\beta^\top(\mutar{1}-\mutar{2}))^2}(\mutar{1}-\mutar{2})}\\
        &\hspace{2em} - (1-\probtar{1})\probtar{1}\frac {2} {\sqrt{ \pi}}\exp\parenth{-\frac{(\beta^\top\mutar{2}+\beta_0)^2}{2\sigma^2}} \\
        &\hspace{10em}\times \parenth{-\frac{1}{2\sqrt{2\sigma^2}}(\mutar{1}-\mutar{2})-\frac{\sqrt{\sigma^2/2} \log \frac{\probtar{1}}{1-\probtar{1}}}{(\beta^\top(\mutar{1}-\mutar{2}))^2}(\mutar{1}-\mutar{2})}.
    \end{align*}
    That is, $\beta$ and $\mutar{1}-\mutar{2}$ are in the same direction. Using $\vecnorm{\beta}{2}=1$, we get
    \begin{align*}
        \beta_{\rm oracle} &= \frac{\mutar{1}-\mutar{2}}{\vecnorm{\mutar{1} - \mutar{2}}{2}},\\
        \beta_{0, {\rm oracle}} &= \frac{\sigma^2}{\vecnorm{\mutar{1}-\mutar{2}}{2}}\log\frac{\probtar{1}}{1-\probtar{1}} - \frac{\vecnorm{\mutar{1}}{2}^2-\vecnorm{\mutar{2}}{2}^2}{2\vecnorm{\mutar{1}-\mutar{2}}{2}}.
    \end{align*}
    Plugging the above equations into Eq.~\eqnref{risktar_erf_exp}, after some algebra we get the oracle target risk as
    \begin{align*}
        \risktar{\horacle} &= r_{\probtar{1}, \sigma}\parenth{\vecnorm{\mutar{1}-\mutar{2}}{2}} = r_{\probtar{1}, \sigma}\parenth{\vecnorm{\ftar(1)-\ftar(2)}{2}},
    \end{align*}
    where
    \begin{align}\label{eqn:r_x_defn}
        \begin{split}
            r_{\probtar{1}, \sigma}(x) &= \frac 12\Bigg(1 - \probtar{1}\erf{\frac{x}{2\sqrt{2}\sigma} + \frac{\sigma}{\sqrt{2}x}\log\frac{\probtar{1}}{1-\probtar{1}}}\\
            &\hspace{4em} - (1-\probtar{1})\erf{\frac{x}{2\sqrt{2}\sigma} - \frac{\sigma}{\sqrt{2}x}\log\frac{\probtar{1}}{1-\probtar{1}}}\Bigg).
        \end{split}
    \end{align}
    Next, turning to $\hstar$, note that $span\{\vm{2}{1}, \vm{3}{1},\ldots, \vm{M}{1}, \vm{2}{2}, \vm{3}{2},\ldots, \vm{M}{2}\}=\real^d$, and so the constraint of  conditional invariance in Eq.~\eqnref{optimal_cond_inv_estimator} requires that the first $d$ coordinates of $\betastar$ are zero, i.e., $\phistar$ only uses the last $p-d$ invariant coordinates. Following exactly the same proof as above, we have
    \begin{align*}
        \risktar{\hstar}= r_{\probtar{1}, \sigma}\Big(\vecnorm{\fc(1)-\fc(2)}{2}\Big).
    \end{align*}
    Now we verify that $r_{\probtar{1}, \sigma}(x)$ is a decreasing function in $x$, by calculating its gradient.
    \begin{align*}
        &\quad \frac{dr_{\probtar{1}, \sigma}(x)}{dx}\\
        &= -\frac 1{\sqrt\pi}\probtar{1}\exp\parenth{-\parenth{\frac{x}{2\sqrt{2}\sigma} + \frac{\sigma}{\sqrt{2}x}\log\frac{\probtar{1}}{1-\probtar{1}}}^2}\\
        &\hspace{20em}\times \parenth{\frac{1}{2\sqrt{2}\sigma}-\frac{\sigma}{\sqrt{2}x^2}\log\frac{\probtar{1}}{1-\probtar{1}}}\\
        &\hspace{1em}- \frac1{\sqrt\pi}(1-\probtar{1})\exp\parenth{-\parenth{\frac{x}{2\sqrt{2}\sigma} - \frac{\sigma}{\sqrt{2}x}\log\frac{\probtar{1}}{1-\probtar{1}}}^2}\\
        &\hspace{20em}\times\parenth{\frac{1}{2\sqrt{2}\sigma}+\frac{\sigma}{\sqrt{2}x^2}\log\frac{\probtar{1}}{1-\probtar{1}}}\\
        &=-\sqrt{\frac{\probtar{1}(1-\probtar{1})}{2\pi\sigma^2}}\exp\parenth{-\parenth{\frac{x}{2\sqrt{2}\sigma}}^2 - \parenth{\frac{\sigma}{\sqrt{2}x}\log\frac{\probtar{1}}{1-\probtar{1}}}^2}< 0.
    \end{align*}
    This intermediate result implies that the risk difference between $\hstar$ and $\horacle$ originates from the norm differences between $\ftar(1)-\ftar(2)$ and $\fc(1)-\fc(2)$ only, where the second term is equivalent to the last $p-d$ dimensions of the first term which corresponds to CICs in the general anticausal model. Intuitively, when the dimension of CICs is close to $p$, the discrepancy between the norms of $\ftar(1)-\ftar(2)$ and $\fc(1)-\fc(2)$ becomes insignificant, and the disparity in the target risks of $\hstar$ and $\horacle$ tends to be negligible. In particular, we establish the explicit bound for this disparity as follows: 
    
    When $\probtar{1}=1/2$, the explicit formula for $r_{\probtar{1}, \sigma}(x)$ given in~Eq.\eqnref{r_x_defn} shows that $r_{\probtar{1}, \sigma}(x) = \parenth{1-\erf{x/(2\sqrt{2}\sigma)}}/2$. Since $\vtar{1}-\vtar{2}\sim\Normal(0, 2\tau^2\sigma^2\Ind_d)$, the standard Gaussian concentration bound shows that
    \begin{align*}
        \prob{\vecnorm{\vtar{1}-\vtar{2}}{2} \leq \sqrt{2\tau^2\sigma^2d\parenth{1+\sqrt{\frac 8d\log\frac1 \delta} + \frac 8d\log\frac1 \delta}}}\geq 1-\delta.
    \end{align*}
    By using the target risks given in Eq.~\eqnref{oracle_risk_diff}, we can calculate
    \begin{align*}
        &\quad\risktar{\hstar} - \risktar{\horacle}\\
        &= \frac12\parenth{\erf{\frac{\vecnorm{\ftar(1)-\ftar(2)}{2}}{2\sqrt{2}\sigma}} - \erf{\frac{\vecnorm{\fc(1)-\fc(2)}{2}}{2\sqrt{2}\sigma}}}\\
        &\leq \frac12\parenth{\erf{\frac{\vecnorm{\fsrcm{1}(1)-\fsrcm{1}(2)}{2} + \vecnorm{\vtar{1}-\vtar{2}}{2}}{2\sqrt{2}\sigma}} - \erf{\frac{\vecnorm{\fc(1)-\fc(2)}{2}}{2\sqrt{2}\sigma}}}\\
        &\overset{(i)}{=} \frac12\parenth{\erf{\frac{\xi\sigma\sqrt{p} + \vecnorm{\vtar{1}-\vtar{2}}{2}}{2\sqrt{2}\sigma}} - \erf{\frac{\xi\sqrt{p-d}}{2\sqrt{2}}}}\\
        &\leq \frac12\parenth{\erf{\frac{\xi\sqrt{p} + \tau\sqrt{3d+24\log(1/\delta)}}{2\sqrt{2}}} - \erf{\frac{\xi\sqrt{p-d}}{2\sqrt{2}}}}\\
        &\overset{(ii)}{\leq} \frac{\xi(\sqrt{p}-\sqrt{p-d})+\tau\sqrt{3d+24\log(1/\delta)}}{\sqrt{8\pi}}\exp\parenth{-\frac{\xi^2(p-d)}{8}}\\
        &=\frac{1}{\sqrt{8\pi}}\parenth{\frac{cd}{\sqrt{p}+\sqrt{p-d}}+\tau\sqrt{3d+24\log\parenth{\frac 1\delta}}}\exp\parenth{-\frac{\xi^2(p-d)}{8}}\\
        &\leq \frac{1}{\sqrt{8\pi}}\parenth{{\xi}{\sqrt{d}}+\tau\sqrt{3d+24\log\parenth{\frac 1\delta}}}\exp\parenth{-\frac{\xi^2(p-d)}{8}}\\
        &\leq c_{\xi, \tau}\parenth{\sqrt{d}+\sqrt{\log(1/\delta)}}\exp\parenth{-\frac{\xi^2(p-d)}{8}},
    \end{align*}
    with probability at least $1-\delta$. Here, in step $(i)$, we use the fact that $\fc(1)$ and $\fc(2)$ are the last $p-d$ invariant coordinates of $\fsrcm{1}(1)$ and $\fsrcm{1}(2)$ respectively; and in step $(ii)$ we use the inequality $\int_a^{b}e^{-t^2}dt \leq e^{-a^2}(b-a)$ for $0<a<b$. This completes the proof of the proposition.

\section{Proof of~\secref{DIP}}\label{app:DIP_proof}
In this section we prove our theorems given in~\secref{DIP}.

\subsection{Proof of~\thmref{conditional_invariant_proxy_y}}\label{app:conditional_invariant_proxy_y}
We prove the following generalized version of the theorem, which does not require that $\phiinv$ is exactly conditionally invariant. That is, for any $\hinv=\ginv\circ\phiinv$ where $\ginv\in\mathcal{G}$ and $\phiinv\in\Phi$, for any other classifier $h$ we have
\begin{multline*}
    \abss{\brackets{\risktar{h} - \risksrcone{h}} - \brackets{\prob{h(\Xtar) \neq \hinv(\Xtar)-\prob{h(\Xsrcone) \neq \hinv(\Xsrcone)}}}} \\\hspace{6em} \leq  2 \risksrcone{\hinv} +  L\Pssi{\phiinv}.
\end{multline*}
If this inequality holds, when $\phiinv$ is a conditionally invariant feature mapping across source and target domains, we get $\Pssi{\phiinv}=0$ and arrive at the original \thmref{conditional_invariant_proxy_y}.
\begin{proof}
    Define $\delta_{ij}=\probc{\hinv(\Xtar)=j}{\Ytar=i}-\probc{\hinv(\Xsrcone)=j}{\Ysrcone=i}$, then $\abss{\delta_{ij}}\leq \Pssi{\hinv}$. Because there is no label shift, for any $i,j\in\braces{1,2,\ldots,L}$ we have
        \begin{align*}
            &\prob{\Ytar=i, \hinv(\Xtar) = j}-\prob{\Ysrcone=i, \hinv(\Xsrcone) = j}=\delta_{ij}\cdot \prob{\Ytar=i}.
        \end{align*}
    As a result, we get
        \begin{align}\label{eqn:conditional_h_hcip_expand}
            &\quad \prob{h(\Xtar) = j, \hinv(\Xtar) = j}- \prob{h(\Xsrcone) = j, \hinv(\Xsrcone) = j} \notag\\
            &=\sum_{i=1}^L \bigg[\probc{h(\Xtar) = j}{\Ytar=i, \hinv(\Xtar) = j}\cdot \prob{\Ytar=i, \hinv(\Xtar) = j}\notag\\
            &\hspace{4em}- \probc{h(\Xsrcone) = j}{\Ysrcone=i, \hinv(\Xsrcone) = j}\cdot \prob{\Ysrcone=i, \hinv(\Xsrcone) = j} \bigg]\notag\\
            &= \sum_{i=1}^L \Bigg(\bigg[\probc{h(\Xtar) = j}{\Ytar=i, \hinv(\Xtar) = j} \notag\\
            &\hspace{5em}- \probc{h(\Xsrcone) = j}{\Ysrcone=i, \hinv(\Xsrcone) = j}\bigg]\cdot \prob{\Ysrcone=i, \hinv(\Xsrcone) = j}\notag\\
            &\hspace{4em}+ \delta_{ij} \cdot \prob{\Ytar=i} \cdot \probc{h(\Xtar) = j}{\Ytar=i, \hinv(\Xtar) = j}\Bigg).
        \end{align}
    Similarly we can calculate
        \begin{align}\label{eqn:conditional_h_Y_expand}
            &\quad \prob{h(\Xtar) = j, \Ytar = j}- \prob{h(\Xsrcone) = j, \Ysrcone = j} \notag\\
            &=\sum_{i=1}^L \bigg[\probc{h(\Xtar) = j}{\Ytar=j, \hinv(\Xtar) = i}\cdot \prob{\Ytar=j, \hinv(\Xtar) = i}\notag\\
            &\hspace{4em}- \probc{h(\Xsrcone) = j}{\Ysrcone=j, \hinv(\Xsrcone) = i}\cdot \prob{\Ysrcone=j, \hinv(\Xsrcone) = i} \bigg]\notag\\
            &= \sum_{i=1}^L \Bigg(\bigg[\probc{h(\Xtar) = j}{\Ytar=j, \hinv(\Xtar) = i} \notag\\
            &\hspace{5em}  - \probc{h(\Xsrcone) = j}{\Ysrcone=j, \hinv(\Xsrcone) = i}\bigg] \cdot \prob{\Ysrcone=j, \hinv(\Xsrcone) = i}\notag\\
            &\hspace{4em}+ \delta_{ji} \cdot \prob{\Ytar=j} \cdot \probc{h(\Xtar) = j}{\Ytar=j, \hinv(\Xtar) = i}\Bigg).
        \end{align}
    Subtracting Eq.~\eqnref{conditional_h_Y_expand} from Eq.~\eqnref{conditional_h_hcip_expand}, we obtain
        \begin{align}\label{eqn:role2_proof_intermediate}
            &\brackets{\prob{h(\Xtar) = j, \hinv(\Xtar) = j}- \prob{h(\Xsrcone) = j, \hinv(\Xsrcone) = j} }  \notag\\
            & \hspace{8em} - \brackets{\prob{h(\Xtar) = j, \Ytar = j}- \prob{h(\Xsrcone) = j, \Ysrcone= j} } \notag\\
            =&\bigg[\probc{h(\Xtar) = j}{\Ytar\neq j,\hinv(\Xtar) = j}\notag\\
            &\hspace{6em}- \probc{h(\Xsrcone) = j}{\Ysrcone\neq j,\hinv(\Xsrcone) = j}\bigg]  \cdot B_j \notag\\
            &- \bigg[\probc{h(\Xtar) = j}{\Ytar = j, \hinv(\Xtar)\neq j} \notag\\
            &\hspace{6em}- \probc{h(\Xsrcone) = j}{\Ysrcone= j, \hinv(\Xsrcone)\neq j} \bigg] \cdot B'_j  \notag\\
            &\quad  +\sum_{i=1}^L\bigg[\delta_{ij}\cdot \prob{\Ytar=i} \cdot \probc{h(\Xtar) = j}{\Ytar=i, \hinv(\Xtar) = j}\notag\\
            &\hspace{6em} - \delta_{ji} \cdot \prob{\Ytar=j} \cdot \probc{h(\Xtar) = j}{\Ytar=j, \hinv(\Xtar) = i}\bigg]
        \end{align}
    where we write
    \begin{align*}
        B_j = \prob{\Ysrcone\neq j, \hinv(\Xsrcone)=j}, \ \ B'_j =  \prob{\Ysrcone = j, \hinv(\Xsrcone) \neq j}.
    \end{align*}
    Now it follows that
        \begin{align*}
            &\quad\Bigg|\brackets{ \prob{h(\Xsrcone) \neq \hinv(\Xsrcone)}  - \prob{h(\Xtar) \neq \hinv(\Xtar)}} - \left(\risksrcone{h} - \risktar{h}\right)\Bigg| \\
            &= \Bigg|\brackets{\prob{h(\Xtar) = \hinv(\Xtar)} - \prob{h(\Xsrcone) = \hinv(\Xsrcone)} }  \\
            &\hspace{2in} - \brackets{\prob{h(\Xtar) = \Ytar} - \prob{h(\Xsrcone) = \Ysrcone}}\Bigg| \notag \\
            &\overset{(i)}{\leq}\sum_{j=1}^L B_j + \sum_{j=1}^L B_j' + \Bigg|\sum_{i=1}^L\sum_{j=1}^L\delta_{ij} \cdot \prob{\Ytar=i}\\
            &\hspace{10em}\cdot \bigg(\probc{h(\Xtar) = j}{\Ytar=i, \hinv(\Xtar) = j}\\
            &\hspace{12em}- \probc{h(\Xtar) = i}{\Ytar=j, \hinv(\Xtar) = i}\bigg)\Bigg|\\
            &\overset{(ii)}{\leq} 2 \risksrcone{\hinv} + \sum_{i=1}^L\sum_{j=1}^L \abss{\delta_{ij}}\cdot \prob{\Ytar=i}\\
            &\overset{(iii)}{\leq} 2 \risksrcone{\hinv} + \sum_{i=1}^L\sum_{j=1}^L \Pssi{\phiinv}\cdot \prob{\Ytar=i}\\
            &= 2 \risksrcone{\hinv} + L\Pssi{\phiinv},
        \end{align*}
    where in step (i) we sum over $j\in\Yset=\{1,2,\ldots,L\}$ using Eq.~\eqnref{role2_proof_intermediate}, and in step (ii) we use $\sum_{j=1}^{L}B_j = \sum_{j=1}^{L}B'_j=\risksrcone{\hinv}$. Step (iii) applies $\abss{\delta_{ij}}\leq\Pssi{\phiinv}$.
\end{proof}

\subsection{Proof of~\thmref{SEM_guarantee}}\label{app:SEM_guarantee}

    For notation simplicity, for any $\phi\in\Phi$, we omit the bias term and write $\phi(x)=Ax$. All proof below holds after adding the bias term. And we use $u_1,\ldots,u_L$ and $\ut_1,\ldots,\ut_L$ to denote $\fsrcone(1),\ldots,\fsrcone(L)$ and $\ftar(1),\ldots,\ftar(L)$, respectively.

    \paragraph{Part (a)} By~\assumpref{uniform_y}, for any $A\in\Phi$, the distributions of $A\Xsrcone$ and $A\Xtar$ are given by the mixture distributions of the form
    \[\Pmarsrcone{AX} = \sum_{j=1}^{L}\frac{1}{L}P_{A\epsilon}(A u_j), \ \ \Pmartar{AX}=\sum_{j=1}^{L}\frac{1}{L}P_{A\epsilon}(A \ut_j), \]
    where $P_{A\epsilon}(u)$ denotes the distribution of $A\eps$ with mean shifted by a fixed vector $u$. Applying~\lemref{two_mixtures}, the constraint in DIP, $\Pmarsrcone{AX}=\Pmartar{AX}$, implies that we can find a permutation $\pi$ on $\{1,2,\ldots,L\}$ such that
    \begin{equation}\label{eqn:mean_match}
    Au_j = A\ut_{\pi (j)} \text{ for all $j\in\Yset$}.
    \end{equation}
    Hence any representations learned by DIP should meet the constraint given in Eq.~\eqnref{mean_match} for some permutation $\pi$. Moreover, letting  $\dim\left(\spn\{u_1-\ut_{\pi (1)},\ldots,u_L-\ut_{\pi (L)}\}\right)=k$, \lemref{linear_matching} shows that any matrix $A$ satisfying Eq.~\eqnref{mean_match} has the row space orthogonal to $\spn\{u_1-\ut_{\pi (1)},\ldots,u_L-\ut_{\pi (L)}\}$ and $\rank(A)\leq p-k$.


    Now consider all the pairs of $(A_\pi,\pi)$ that meets~\eqnref{mean_match} and the row space of $A_\pi$ is the orthogonal complement to $\spn\{u_1-\ut_{\pi (1)},\ldots,u_L-\ut_{\pi (L)}\}$. By definition of DIP in Eq.~\eqnref{pop_DIP}, we obtain $\phidip=A_{\pi_\star}$, where
    \[ \pi_\star = \argmin_{\{\pi \text{ is permutation}\}}\min_{g\in\Gset, A_\pi} \risksrcone{g\circ A_\pi}. \]
    We can clearly find examples where $\pi_\star\neq \ident$, for instance, as shown in~\figref{dip_vs_jdip}. This occurs as long as the label-flipping features offer lower source risk than the conditionally invariant features while it cannot generalize to the target domain. In this case, the conditional distributions of $\phidip(\Xsrcone)$ and $\phidip(\Xtar)$, given the labels, are aligned after the target labels are permuted by $\pi_\star$.

    \paragraph{Part (b)} 	Next, we consider JointDIP defined in Eq.~\eqnref{pop_JointDIP}. If $\phiinv$ is a linear mapping, say $\phiinv(x)=Bx$ for some matrix $B\in\real^{r\times p}$, then writing  $C=(A^\top , B^\top)^\top$ to denote the concatenation of $A$ and $B$, we have
    \begin{align*}
        \Pmarsrcone{(\phijdip, \phiinv)(X)}&=\Pmarsrcone{CX}= \sum_{j=1}^{L}\frac{1}{L}P_{C\epsilon}(C u_j),\\
        \Pmartar{(\phijdip, \phiinv)(X)}&=\Pmartar{CX}=\sum_{j=1}^{L}\frac{1}{L}P_{C\epsilon}(C \ut_j).
    \end{align*}
    Under the JointDIP matching penalty, by~\lemref{two_mixtures} we similarly have
    \begin{equation}\label{eqn:mean_joint_match}
    Cu_j = C\ut_{\pi (j)} \text{ for all $j$},
    \end{equation}
    for some permutation $\pi$. If $\pi\neq \ident$, Eq.~\eqnref{mean_joint_match} suggests that $Bu_j=B\ut_{\pi (j)}=B\ut_{j'}$ for some $j'\neq j$, which contradicts the condition~\eqnref{linear_cic_condition}. Therefore, we must have $\pi=\ident$ (note that it is a valid solution because $\phiinv$ is a CIC), which shows that $\phijdip$ is always conditionally invariant across source and target domains.

    Moreover, given that there is no label shift and $\phijdip$ is conditionally invariant across $\Psrcone$ and $\Ptar$, we have $\risktar{g\circ \phijdip} =\risksrcone{g\circ \phijdip}$ for all $g\in\Gset$. Then the optimization objective of population JointDIP becomes equivalent to minimizing $\risktar{g\circ \phijdip}$. Since $\phiinv^0=(\phiinv, \zeros_{q-r})$ is a feasible solution to the JointDIP constraint, we conclude that JointDIP is at least as good as the optimal classifier built upon $\phiinv^0$, i.e.,
    \begin{align*}
        \risktar{\hjdip}\leq \min_{g\in\Gset}\risktar{g\circ \phiinv^0}.
    \end{align*}

    \paragraph{Part (c)} In JointDIP, the joint distributions of $AX\in\R^q$ and $\phiinv(X)\in\R^r$ are matched across the source and target domains. Let
    \[ T(X) = \left( \begin{array}{c}AX \\ \phiinv(X) \end{array}\right)\in\R^{q+r}.\]
    The characteristic function of $T(\Xsrcone)$ under the source distribution is then:
    \begin{align*}
        \csrcone(t)  = \csrcone(t_1,t_2) &= \EE{e^{it^\top T(\Xsrcone)}} \\
        &=\frac{1}{L}\sum_{j=1}^L \EEst{e^{it^\top T(\Xsrcone)}}{\Ysrcone=j} \\
        &= \frac{1}{L}\sum_{j=1}^L \EEst{e^{it_1^\top A\Xsrcone + it_2^\top \phiinv(\Xsrcone)}}{\Ysrcone=j} \\
        &=  \frac{1}{L}\sum_{j=1}^L e^{it_1^\top A u_j} \cdot\underbrace{ \EEst{ e^{it_1^\top A \epssrcone + it_2^\top \phiinv(\Xsrcone)}}{\Ysrcone=j}}_{:= b_j(t_1,t_2)},
    \end{align*}
    where the second step uses the tower property, and the last step follows from the general anticausal model, where conditional on $\Ysrcone=j$, $\Xsrcone= u_j + \epssrc$ under the source distribution. Similarly, we can compute the characteristic function of $T$ under the target distribution:
    \begin{align*}
        \ctar(t) &= \ctar(t_1,t_2) \\
        &=\frac{1}{L}\sum_{j=1}^L e^{it_1^\top A \ut_j} \cdot  \EEst{ e^{it_1^\top A \epstar + it_2^\top \phiinv(\Xtar)}}{\Ytar=j}\\
        &=\frac{1}{L}\sum_{j=1}^L e^{it_1^\top A \ut_j} \cdot  b_j(t_1,t_2).
    \end{align*}
    Here the last equality holds because $\epstar$ and $\phiinv(\Xtar)$ share the same conditional distributions under the target distribution as $\epssrcone$ and $\phiinv(\Xsrcone)$ under the source distribution. The JointDIP matching penalty enforces $\csrcone(t) = \ctar(t)$ for all $t$, or
    \[\sum_{j=1}^L e^{it_1^\top A u_j}  b_j(t_1,t_2) = \sum_{j=1}^L e^{it_1^\top A \ut_j} b_j(t_1,t_2) \text{ for all $t_1\in\R^{q}, t_2\in\R^r$}. \]
    If we take partial derivaties of both sides with respect to $t_1$,
    \begin{align*}
        &\quad\parenth{\sum_{j=1}^{L} Au_j e^{it_1^\top A u_j}  b_j(t_1,t_2)}i + \sum_{j=1}^{L} e^{it_1^\top A u_j} \frac{\partial b_j(t_1,t_2)}{\partial t_1}\\
        &= \parenth{\sum_{j=1}^{L} A\ut_j e^{it_1^\top A u_j}  b_j(t_1,t_2)}i + \sum_{j=1}^{L} e^{it_1^\top A \ut_j} \frac{\partial b_j(t_1,t_2)}{\partial t_1} .
    \end{align*}
    Setting $t_1=0$, it follows that
    \begin{align*}
        \sum_{j=1}^{L} Au_j \cdot b_j(0,t_2)  &= \sum_{j=1}^{L} A\ut_j  \cdot b_j(0,t_2),
    \end{align*}
    or equivalently,
    \begin{align}\label{eqn:JointDIP_identity}
        \sum_{j=1}^{L} Au_j \cdot \EEst{ e^{it_2^\top \phiinv(\Xsrcone)}}{\Ysrcone=j}  &= \sum_{j=1}^{L} A\ut_j  \cdot  \EEst{ e^{it_2^\top \phiinv(\Xsrcone)}}{\Ysrcone=j}.
    \end{align}
    For any vector $a\in\real^r$, let $t_2 = ka$ for $k\in\real$. In Eq.~\eqnref{JointDIP_identity}, take derivative with respect to $k$ up to $L-1$ times and set $k=0$, then we get
    \[\sum_{j=1}^{L} Au_j \cdot  \EEst{ \parenth{a^\top \phiinv(\Xsrcone)}^u}{\Ysrcone=j}  = \sum_{j=1}^{L} A\ut_j  \cdot \EEst{ \parenth{a^\top \phiinv(\Xsrcone)}^u}{\Ysrcone=j}, \]
    for $u=0,1,\ldots,L-1$,
    where we use the fact that $\phiinv$ is a conditionally invariant feature mapping. Writing the above linear system in a matrix form, we have
    \begin{align*}
        \underbrace{\begin{pmatrix}
            1 & 1 & \cdots & 1\\
            m_{1}^{1}(a) & m_{2}^{1}(a) & \cdots & m_L^1(a)\\
            \vdots & \vdots & \ddots & \vdots\\
            m_1^{L-1}(a) & m_2^{L-1}(a) & \cdots & m_L^{L-1}(a)
        \end{pmatrix}}_{=C_{\phiinv}(a)}
        \begin{pmatrix}
        ({Au_1} - {A\ut_1})^\top\\
        ({Au_2} - {A\ut_2})^\top\\
        \vdots\\
        ({Au_L} - {A\ut_L})^{\top}
        \end{pmatrix} = 0, \ \ \forall \ell\in\{1,2,\ldots,L\},
    \end{align*}
    where $m_j^l(a)$ = $\EEst{ \parenth{a^\top \phiinv(\Xsrcone)}^{l}}{\Ysrcone=j}$. By assumption that $C_{\phiinv}(a)$ is full rank, we get $Au_j = A\ut_j$ for all $j\in\{1,\ldots,L\}$, i.e.,
    \begin{align*}
        \Pcondsrcone{\phijdip(X)}{Y=y} = \Pcondtar{\phijdip(X)}{Y=y},  \ \ \forall y\in\{1,2,\ldots, L\},
    \end{align*}
	proving the result.

\begin{lemma}\label{lem:two_mixtures}
Consider the mixture distributions $P_1, P_2$ given by
\[P_1 = \sum_{j=1}^{L} w_j P(u_j)  \text{ and } P_2=\sum_{j=1}^{L}\wt_j P(\ut_j), \]
where $w_j, \wt_j$ are the mixing weights such that $\sum_j w_j = \sum_j \wt_j=1$ and $w_j,\wt_j\geq 0$, and $u_j,\ut_j\in\R^p$ are the fixed vectors for the $j$-th components, and $P(u_j), P(\ut_j)$ denote the distribution $P$ location shifted by $u_j,\ut_j$ respectively. If $P_1=P_2$, then there exists a permutation $\pi$ such that $u_j=\ut_{\pi j}$ for all $j$.
\end{lemma}

\begin{proof}
    Let $\nu=\sum_{j=1}^{L}w_j\delta_{u_j}$ and $\widetilde{\nu}=\sum_{j=1}^{L}\wt_j\delta_{\ut_j}$ be the mixing distributions of $P_1$ and $P_2$ respectively.
    Let $X\sim P_1$ and $\Xt\sim P_2$. Then we can write $X=U+Z$ and $\Xt=\widetilde{U}+\widetilde{Z}$ where
    \begin{equation*}
        \begin{cases}
            U\sim \nu, Z\sim P \textnormal{ independent of $U$};\\
            \widetilde{U}\sim \widetilde{\nu}, \widetilde{Z}\sim P \textnormal{ independent of $\widetilde{U}$}.
        \end{cases}
    \end{equation*}
    The characteristic functions of $X$ and $\Xt$ can be expressed as
    \[\varphi_X(t)=\varphi_U(t)\varphi_Z(t), \ \ \varphi_{\Xt}(t)=\varphi_{\widetilde{U}}(t)\varphi_{\widetilde{Z}}(t). \]
    Since $\varphi_X(t)=\varphi_{\Xt}(t)$ and $\varphi_Z(t)=\varphi_{\widetilde{Z}}(t)$, it follows that $\varphi_U(t) = \varphi_{\widetilde{U}}(t)$, or equivalently $\nu=\widetilde{\nu}$.

    Next let $S=\{u_1,\ldots,u_L\}$ and $\widetilde{S}=\{\ut_1,\ldots,\ut_L\}$ to denote the support sets of $\nu$ and $\widetilde{\nu}$. We can easily extend Lemma 2 in~\cite{wu2020optimal} to the high-dimensional setting to show that
    \[d_H(S,\widetilde{S}) \leq L\cdot W_1(\nu,\widetilde{\nu})  =0,\]
    where $d_H(S,\widetilde{S})$ is the Hausdorff distance between $S$ and $\widetilde{S}$, defined as
    \[ d_H(S,\widetilde{S}) = \max\left\{ \sup_{x\in S}\inf_{\widetilde{x}\in\widetilde{S}} \norm{x-\widetilde{x}}_1, \sup_{\widetilde{x}\in\widetilde{S}}\inf_{x\in S} \norm{x-\widetilde{x}}_1 \right\}, \]
    and where $W_1$ denotes the $1$-Wasserstein distance. By definition of $d_H(S,\widetilde{S})$, this proves that we can find a permutation $\pi$ such that $u_j=\ut_{\pi j}$ for all $j$, thus proving the lemma.
\end{proof}

\begin{lemma}\label{lem:linear_matching}
Suppose that $\{u_1,\ldots,u_L\}$ and $\{\ut_1,\ldots,\ut_L\}\subseteq\R^p$ are collections of vectors such that
\[\dim\left(\spn\{u_1-\ut_1,\ldots,u_L-\ut_L\}\right) = k <p. \]
Then for any $A\in\R^{q\times p}$ satisfying $Au_j=A\ut_j$ for all $j=1,\ldots,L$, we have $\rank(A)\leq p-k$. The equality is achieved if and only if the rows of $A$ span $\spn\{u_1-\ut_1,\ldots,u_L-\ut_L\}^\perp$.
\end{lemma}

\begin{proof}
    This lemma is a simple consequence of linear algebra. By the condition of the lemma, $A(u_j - \ut_j)=0$ for all $j=1,\ldots,L$. This implies that the row space of $A$ is orthogonal to the subspace spanned by $\{u_1-\ut_1,\ldots,u_L-\ut_L\}$, i.e.,
    \[\row{A}\perp \spn\{u_1-\ut_1,\ldots,u_L-\ut_L\}.\]
    In particular, $\dim(\row{A}) \leq p-k$, with equality achieved if and only if $\row{A}$ is the orthogonal complement of $\spn\{u_1-\ut_1,\ldots,u_L-\ut_L\}$. By the rank theorem, $\rank(A)=\dim(\row{A}) \leq p-k$, proving the lemma.
\end{proof}

\newpage
\section{Experiment details}\label{app:exp}
This section provides details on various DA algorithms, training procedures, network architectures, and hyperparameter selection strategies that we employed in our numerical experiments (\secref{exp}).
\subsection{DA algorithms in numerical experiments}\label{app:DA_alg}

\bgroup
\def\arraystretch{2}
\begin{table}[t]
    \centering
    \resizebox{\textwidth}{!}
    {\addtolength{\tabcolsep}{-0.4em}
        \begin{tabular}{l| c | l}
            \specialrule{.1em}{.05em}{.05em}
            DA method & $\whm$ Estimation & Minimization Objective \\\specialrule{.1em}{.05em}{.05em}
            Tar & - &  $ \riskhtar{g\circ \phi}$\\\hline
            ERM & - &  $\riskhsrcone{g\circ \phi}$ \\\hline
            ERM-Pool &  - &$\frac 1M\sum\limits_{m=1}^M\riskhsrc{g\circ\phi}$ \\\hline
            DIP & - &  $\riskhtar{g\circ \phi} + \lamdip\cdot\probdiv{\Phmarsrcone{\phi(X)}}{\Phmartar{\phi(X)}}$ \\\hline
            DIP-Pool & - &$\frac 1M\sum\limits_{m=1}^M\braces{\riskhsrc{g\circ\phi} + \lamdip\cdot\probdiv{\Phmarsrc{\phi(X)}}{\Phmartar{\phi(X)}}}$ \\\hline
            CIP & - & $\frac{1}{M}\sum\limits_{m=1}^{M}\riskhsrc{g\circ\phi} + \frac{\lamcip}{LM^2}\cdot\sum\limits_{y=1}^L\sum\limits_{m\neq m'}  \probdiv{\Phcondsrc{\phi(X)}{Y=y}}{\Phcondsrcmp{\phi(X)}{Y=y}}$ \\\hline
            IW-ERM & ERM-Pool & $\frac 1M\sum\limits_{m=1}^M\riskhsrcw{g\circ\phi}{\whm}$ \\\hline
            IW-CIP & CIP & $\frac{1}{M}\sum\limits_{m=1}^{M}\riskhsrcw{g\circ\phi}{\whm} + \frac{\lamcip}{LM^2}\cdot\sum\limits_{y=1}^L\sum\limits_{m\neq m'}  \probdiv{\Phcondsrc{\phi(X)}{Y=y}}{\Phcondsrcmp{\phi(X)}{Y=y}}$ \\\hline
            IW-DIP & CIP & $\riskhsrconew{g\circ \phi}{\whone} + \lamdip\cdot\probdiv{\Phmarsrcone{\whone(Y)\phi(X)}}{\Phmartar{\phi(X)}}$ \\\hline
            JointDIP & - & $\riskhsrcone{g\circ\phi} + \lamjdip\cdot\probdiv{\Phmarsrcone{(\phi,\phihcip )(X)}}{\Phmartar{(\phi,\phihcip)(X)}}$ \\\hline
            JointDIP-Pool & - & $\frac 1M\sum\limits_{m=1}^M \braces{\riskhsrc{g\circ\phi} + \lamjdip\cdot \probdiv{\Phmarsrc{(\phi,\phihcip)(X)}}{\Phmartar{(\phi,\phihcip)(X)}}}$ \\\hline
            IW-JointDIP & CIP & $\riskhsrconew{g\circ\phi}{\whone} + \lamjdip\cdot\probdiv{\Phmarsrcone{\whone(Y)\cdot(\phi,\phihcip)(X)}}{\Phmartar{(\phi,\phihcip)(X)}}$ \\\hline
            IRM & - &  $\frac1M\sum\limits_{m=1}^M\braces{\riskhsrc{g\circ \phi} + \lambda_{{\rm IRM}}\cdot \vecnorm{\nabla_{w\mid w=1.0}\riskhsrc{w\cdot g\circ \phi}}{2}^2}$ \\\hline
            V-REx & - &  $\frac1M\sum\limits_{m=1}^M\riskhsrc{g\circ \phi} + \lambda_{{\rm VREx}}\cdot \text{Var}\parenth{\riskhsrc{g\circ \phi}}$ \\\hline
            groupDRO & - &  $\sup_{\sum_{m=1}^Mq_m=1, q_m\geq 0}\sum_{m=1}^M q_m\riskhsrc{g\circ \phi}$
            \\\specialrule{.1em}{.05em}{.05em}
        \end{tabular}
    }
    \caption{DA algorithms used in numerical experiments.}
    \label{tab:DA_alg_table}
\end{table}
\egroup

\tabref{DA_alg_table} provides a summary of the algorithms that we used in our numerical experiments. Both mean distance and maximum mean discrepancy (MMD) were used in DIP and CIP, but for JointDIP, we only considered MMD distance, because joint matching with mean distance is essentially the same as the original DIP. In addition to the DA algorithms introduced in this work, we also include IRM~\citep{arjovsky2019invariant}, V-REx~\citep{krueger2021out}, and groupDRO~\citep{sagawa2019distributionally} in our experiments.

\begin{table}[t]
	\centering
	\resizebox{\textwidth}{!}
	{\addtolength{\tabcolsep}{-0.4em}
		\begin{tabular}{c| c c c c c c}
			\specialrule{.1em}{.05em}{.05em}
			Task & Epochs & Optimizer & Batch Size & NN Architecture & Feature Layer &$\probdiv{\cdot}{\cdot}$\tablefootnote{This column only applies to CIP and DIP. We always use MMD distance when implementing JointDIP.}\\\specialrule{.1em}{.05em}{.05em}
			SCM & 50 & Adam(lr=1e-2) & 100 & Linear Model & Last layer & mean\\\hline
			SCM\_binary\tablefootnote{This setting is used for the additional experiments in~\appref{add_exp}.} & 50 & Adam(lr=1e-2) & 100 & \makecell{Linear$(10, 10)$ or\\Linear$(10, 10)$\\ReLU\\Linear$(10, 2)$} & Last layer & mean\\\hline
			MNIST & 20 & Adam(lr=1e-3) & 256 & \makecell{Conv2d$(1,20,5,1)$\\ReLU\\MaxPool2d$(2)$\\Conv2d$(20,50,5,1)$\\ReLU\\MaxPool2d$(2)$\\Linear$(800,500)$\\ReLU\\Linear$(500,2)$}& Last layer & MMD\\\hline
			CelebA & 10 & Adam(lr=1e-3) & 64 & \makecell{Conv2d$(3,16,5,1)$\\ReLU\\MaxPool2d$(2)$\\Conv2d$(16,32,5,1,1)$\\ReLU\\MaxPool2d$(2)$\\Conv2d$(32,64,5,1)$\\ReLU\\MaxPool2d$(2)$\\Linear$(960,256)$\\ReLU\\Linear$(256,2)$}& Last layer & MMD\\\hline
			Camelyon17 & 5 & \makecell{SGD(lr=1e-3,\\weight\_decay=0.01,\\momentum=0.9)} & 32 & \makecell{DenseNet-121\\ $($pretrained$=$False$)$} & \makecell{Flattened\\CNN output} & MMD\\\hline
			DomainNet & 10 & Adam(lr=1e-3) & 16 & \makecell{ResNet-101\\ $($pretrained$=$True$)$} & Last layer & MMD
			\\\specialrule{.1em}{.05em}{.05em}
	\end{tabular}}
	\caption{Training details of each classification task.}
	\label{tab:training_setup}
\end{table}

\subsection{Training details and network architectures}\label{app:architec}
Except for Camelyon17 and DomainNet, the datasets we used in~\secref{exp} are either generated or perturbed synthetically. To avoid any potential bias that can result from repeatedly running DA algorithms on a single, fixed synthetic dataset, for each classification task, we use 10 different seeds to construct 10 different datasets, and each DA algorithm is run once per seed.

Table~\ref{tab:training_setup} shows details on the training parameters and model architectures used in all the experiments in~\secref{exp}. For multi-stage DA algorithms such as JointDIP (where CIP is run before solving JointDIP optimization problem), the number of epochs is the same in each training stage. As for the feature layer used for CIP and DIP matching penalty, we utilize the last layer of the neural network in experiments on SCM, MNIST, CelebA, and DomainNet; for experiments on Camelyon17, we use the flattened output from CNN in DenseNet-121 following the practice in~\cite{koh2021wilds}.

	\begin{table}[t]
		\centering
		\small
		\begin{tabular}{l| c c}
			\specialrule{.1em}{.05em}{.05em}
			Algorithm & Hyperparameter & Grid Search Space\\
			\specialrule{.1em}{.05em}{.05em}
			DIP (mean / MMD, Pool) & $\lambda_{\text{DIP}}^*$ & $10^{\braces{-2,-1,0,1,2}}$\\\hline
			CIP (mean / MMD) & $\lambda_{\text{CIP}}^*$ & $10^{\braces{-2,-1,0,1,2}}$\\\hline
			\multirow{2}{*}{IW-DIP (mean / MMD, Pool)} 
			& $\lambda_{\text{CIP}}$ & $10^{\braces{-2,-1,0,1,2}}$\\
			& $\lambda_{\text{IW-DIP}}$ & $10^{\braces{-2,-1,0,1,2}}$\\\hline
			\multirow{2}{*}{IW-CIP (mean / MMD)} 
			& $\lambda_{\text{CIP}}$ & $10^{\braces{-2,-1,0,1,2}}$\\
			& $\lambda_{\text{IW-CIP}}$ & $10^{\braces{-2,-1,0,1,2}}$\\\hline
			\multirow{2}{*}{JointDIP (Pool)} 
			& $\lambda_{\text{CIP}}^*$ & $10^{\braces{-2,-1,0,1,2}}$\\
			& $\lambda_{\text{JointDIP}}^*$ & $10^{\braces{-2,-1,0,1,2}}$\\\hline
			\multirow{2}{*}{IW-JointDIP} 
			& $\lambda_{\text{CIP}}$ & $10^{\braces{-2,-1,0,1,2}}$\\
			& $\lambda_{\text{JointDIP}}$ & $10^{\braces{-2,-1,0,1,2}}$\\\hline
			\multirow{2}{*}{IRM} 
			& $\lambda_{\text{IRM}}^{**}$ & $10^{\braces{-1,0,1,2,3,4}}$\\
			& iterations annealing$^{***}$ & $0, 10, 100, 1000, 3000$\\\hline
			\multirow{2}{*}{V-REx} 
			& $\lambda_{\text{VREx}}^{**}$ & $10^{\braces{-1,0,1,2,3,4}}$\\
			& iterations annealing$^{***}$ & $0, 10, 100, 1000, 3000$\\\hline
			groupDRO & eta & $10^{\braces{-2,-1,0,1}}$\\
			\specialrule{.1em}{.05em}{.05em}
		\end{tabular}
		\caption{Hyperparameter search space. In the experiments on Camelyon17, the search space of hyperparameters marked with $^*$ is defined by $10^{\braces{-3,-2,-1,0,1,2}}$. In the experiment on DomainNet, the search space of hyperparameters marked with $^*$ is defined by $10^{\braces{-2,-1,0,1,2,3}}$, for hyperparameters marked with $^{**}$ by $10^{\braces{-1,0,1}}$, and for hyperparameters marked with $^{***}$ by $\braces{0,10,100}$.}
		\label{tab:hyper}
	\end{table}

\subsection{Hyperparameter selection}\label{app:hyper}

In our experiments on SCM, MNIST, and CelebA, instead of searching for different hyperparameters for each seed, which would be time-consuming, we choose the same hyperparameters across all seeds. In Appendix~\ref{app:hyper2}, we compare between two hyperparameter selection strategies, i.e., using different hyperparameters versus using the same hyperparameters across seeds on SCMs, and find that these two strategies yield similar results. 

In order to select hyperparameters which allow for the highest average accuracy, we use a simple grid search. Three different ways of hyperparameter selection are presented in \cite{gulrajani2020search}: training-domain validation set, leave-one-out cross validation, and test-domain validation set (oracle). The training-domain validation set approach is not suitable for us as it assumes the target domain is very similar to the training domains, which does not apply to our case. The leave-one-out cross validation approach is cumbersome and would also suffer if the validation domain is significantly different from the target domain. The test-domain validation set approach has often been adopted in several  domain adaptation papers~\citep{ganin2016domain, arjovsky2019invariant, krueger2021out} when an appropriate validation set is unavailable. However, this approach contradicts the domain adaptation principle where most target labels should be unavailable. Therefore, we restrict access to just 10\% of labeled samples from the target domain in each seed, and check the accuracy of these labeled samples only at the final checkpoint during hyperparameter selection.

Note that Camelyon17 provides a validation set. Hence, we adhere to the standard submission guidelines from WILDS for this dataset, selecting hyperparameters based on accuracy scores on the validation set~\cite{koh2021wilds, sagawa2021extending}.

\begin{table}[t]
	\centering
	\small
	\begin{tabular}{l | ll | ll}
		\toprule
		{} & \multicolumn{2}{c|}{SCM \rom{1}} & \multicolumn{2}{c}{SCM \rom{2}} \\\hline
		Mean shift & \multicolumn{2}{c|}{Y} & \multicolumn{2}{c}{Y} \\
		CICs & \multicolumn{2}{c|}{N} & \multicolumn{2}{c}{Y} \\
		Label shift & \multicolumn{2}{c|}{N} & \multicolumn{2}{c}{Y} \\
		Label-flip & \multicolumn{2}{c|}{N} & \multicolumn{2}{c}{N} \\
		\hline
		Hyperparam & Different & Same & Different & Same \\
		\hline
		DIP-mean & \textbf{\textcolor{red}{88.2$\pm$1.9}} & \textbf{\textcolor{red}{87.6$\pm$1.5}} & 73.3$\pm$10.7 & 62.0$\pm$2.9 \\
		DIP-MMD & 86.2$\pm$2.2 & 86.6$\pm$2.2 & 72.0$\pm$16.2 & 59.9$\pm$3.3 \\
		DIP-Pool-mean & 86.2$\pm$2.8 & 86.4$\pm$2.2 & 70.8$\pm$10.1 & 60.1$\pm$3.1 \\
		DIP-Pool-MMD & 86.2$\pm$1.9 & 85.5$\pm$3.1 & 74.2$\pm$13.5 & 61.3$\pm$9.3 \\
		CIP-mean & 56.6$\pm$11.8 & 55.9$\pm$12.0 & 79.3$\pm$6.4 & 75.7$\pm$6.5 \\
		CIP-MMD & 63.2$\pm$14.1 & 60.4$\pm$14.8 & 84.1$\pm$9.6 & 73.5$\pm$7.0 \\
		IW-CIP-mean & 54.1$\pm$11.8 & 54.0$\pm$11.8 & 90.8$\pm$1.4 & 90.4$\pm$0.8 \\
		IW-CIP-MMD & 64.4$\pm$12.7 & 56.8$\pm$13.6 & 89.5$\pm$4.3 & 90.4$\pm$0.8 \\
		IW-DIP-mean & 54.3$\pm$11.7 & 54.3$\pm$11.7 & \textbf{\textcolor{red}{92.7$\pm$2.1}} & \textbf{\textcolor{red}{92.1$\pm$2.7}} \\
		IW-DIP-MMD & 75.2$\pm$16.2 & 68.0$\pm$19.4 & 91.5$\pm$2.9 & 90.1$\pm$1.1 \\
		JointDIP & 85.8$\pm$2.2 & 86.8$\pm$1.9 & 77.7$\pm$12.5 & 70.6$\pm$6.2 \\
		IW-JointDIP & 72.7$\pm$17.3 & 68.0$\pm$19.4 & 91.3$\pm$2.8 & 90.0$\pm$1.3 \\
		IRM & 66.1$\pm$12.1 & 56.7$\pm$10.4 & 86.7$\pm$5.4 & 71.9$\pm$17.9 \\
		V-REx & 63.0$\pm$13.1 & 55.6$\pm$11.5 & 72.7$\pm$26.3 & 62.9$\pm$25.7 \\
		groupDRO & 54.4$\pm$10.0 & 54.4$\pm$10.0 & 65.2$\pm$26.6 & 64.2$\pm$25.7 \\
		\bottomrule
	\end{tabular}
	\caption{Target accuracy in linear SCM \rom{1} and \rom{2} under different hyperparameter tuning strategies.}
	\label{tab:SEM_hyper_compare_1_2}
\end{table}

\subsection{Different hyperparameter vs. same hyperparameter across all seeds}\label{app:hyper2}
We carry out a comparison of two strategies for selecting hyperparameters on SCMs: (1) using the same hyperparameter across all seeds, and (2) using different hyperparameters for different seeds. See~\tabref{SEM_hyper_compare_1_2} and~\tabref{SEM_hyper_compare_3_4}. We observe that in most cases, the second approach tends to give a better performance. However, this difference typically doesn't affect the order of the top-performing algorithms; the best performing algorithm tends to remain the same regardless of the hyperparameter selection strategies. For the sake of simplicity and time efficiency, we therefore adopt the first approach in our final results.

\begin{table}[t]
	\centering
	\small
	\begin{tabular}{l | ll | ll}
		\toprule
		{} & \multicolumn{2}{c|}{SCM \rom{3}} & \multicolumn{2}{c}{SCM \rom{4}} \\\hline
		Mean shift & \multicolumn{2}{c|}{Y} & \multicolumn{2}{c}{Y} \\
		CICs & \multicolumn{2}{c|}{Y} & \multicolumn{2}{c}{Y} \\
		Label shift & \multicolumn{2}{c|}{N} & \multicolumn{2}{c}{Y} \\
		Label-flip & \multicolumn{2}{c|}{Y} & \multicolumn{2}{c}{Y} \\
		\hline
		Hyperparam & Different & Same & Different & Same \\
		\hline
		DIP-mean & 36.3$\pm$11.8 & 34.5$\pm$14.9 & 35.6$\pm$14.0 & 35.3$\pm$14.6 \\
		DIP-MMD & 77.5$\pm$7.9 & 63.3$\pm$31.7 & 74.0$\pm$4.2 & 60.2$\pm$30.1 \\
		DIP-Pool-mean & 81.9$\pm$1.8 & 82.0$\pm$1.1 & 81.9$\pm$3.8 & 82.3$\pm$3.8 \\
		DIP-Pool-MMD & 82.5$\pm$1.4 & 83.3$\pm$0.8 & 82.2$\pm$4.1 & 82.5$\pm$3.8 \\
		CIP-mean & 82.2$\pm$1.5 & 81.8$\pm$1.3 & 83.2$\pm$1.8 & 82.1$\pm$1.2 \\
		CIP-MMD & 81.7$\pm$2.7 & 81.7$\pm$2.6 & 83.6$\pm$2.8 & 82.7$\pm$1.6 \\
		IW-CIP-mean & 80.3$\pm$4.2 & 81.2$\pm$4.1 & 83.1$\pm$2.3 & 83.8$\pm$2.2 \\
		IW-CIP-MMD & 80.9$\pm$3.5 & 80.5$\pm$6.3 & 84.1$\pm$1.9 & 83.5$\pm$2.8 \\
		IW-DIP-mean & 37.6$\pm$12.9 & 37.2$\pm$14.9 & 66.6$\pm$5.2 & 64.2$\pm$7.3 \\
		IW-DIP-MMD & 70.9$\pm$13.8 & 65.1$\pm$23.6 & 81.5$\pm$2.1 & 80.1$\pm$2.6 \\
		JointDIP & \textbf{\textcolor{red}{85.3$\pm$2.8}} & \textbf{\textcolor{red}{85.4$\pm$2.1}} & 81.6$\pm$3.9 & 82.8$\pm$1.9 \\
		IW-JointDIP & 82.9$\pm$3.3 & 82.9$\pm$8.1 & \textbf{\textcolor{red}{85.4$\pm$1.9}} & \textbf{\textcolor{red}{85.1$\pm$3.7}} \\
		IRM & 82.0$\pm$1.9 & 80.2$\pm$1.8 & 84.0$\pm$2.2 & 83.7$\pm$3.3 \\
		V-REx & 80.8$\pm$7.6 & 80.4$\pm$7.6 & 84.0$\pm$4.0 & 83.8$\pm$3.7 \\
		groupDRO & 82.5$\pm$6.7 & 81.3$\pm$7.4 & 84.7$\pm$3.8 & 84.3$\pm$3.2 \\
		\bottomrule
	\end{tabular}
	\caption{Target accuracy in linear SCM \rom{3} and \rom{4} under different hyperparameter tuning strategies.}
	\label{tab:SEM_hyper_compare_3_4}
\end{table}

\section{Additional experiments}\label{app:add_exp}
In this section, we provide additional experiments on SCMs and MNIST to investigate the effect of the number of source domains on the generalization performance of CIP and its variants. Additionally, we compare the performance of our algorithms with simple heuristic alternatives through numerical experiments on MNIST and CelebA.

\subsection{Effect of number of source domains in linear SCMs}

To show that a sufficient number of source domains is necessary for finding CICs, we conduct experiments on linear SCMs to compare the results of utilizing fewer versus more source domains. Table~\ref{tab:SEM_different_M_split} reports the risk difference $\riskhtar{h}-\riskhbar{h}{\wh}$ between the source and target domains for four different CIP-based methods. We observe that the risk difference tends to be closer to 0 for larger values of $M$ in SCM \rom{2}, \rom{3}, and \rom{4}, whereas smaller values of $M$ lead to a larger risk difference. This finding suggests that CIP and IW-CIP may underperform when the number of source domains $M$ is small, because they can only find feature representations that are conditionally invariant across the source domains, but with limited domains, it becomes challenging to obtain feature representations that are conditionally invariant across both source and target domains. Note that in SCM \rom{1}, the number of domains does not make a significant difference, because CICs do not exist in SCM \rom{1}.

\begin{table}[t]
	\centering
	\small
	\begin{tabular}{l | cc | cc }
		\toprule
		{} & \multicolumn{2}{c|}{SCM \rom{1}} & \multicolumn{2}{c}{SCM \rom{2} ($\times10^{-2}$)} \\\hline
		DA Algorithm & $M=2$ & $M=3$ & $M=2$ & $M=11$ \\
		\hline
		CIP-mean & $1.7\pm2.0$ & $2.1\pm1.4$ & $11.3\pm57.9$ & $-0.1\pm0.9$ \\
		CIP-MMD  & $1.7\pm1.9$ & $1.3\pm1.1$ & $-5.9\pm7.6$  & $3.6\pm9.6$  \\
		IW-CIP-mean & $1.9\pm2.2$ & $2.0\pm0.8$ & $89.0\pm120.0$ & $-0.6\pm10.9$ \\
		IW-CIP-MMD  & $2.0\pm2.3$ & $1.7\pm1.0$ & $41.7\pm60.0$ & $-4.7\pm9.0$ \\
		\midrule
		{} & \multicolumn{2}{c|}{SCM \rom{3} ($\times10^{-2}$)} & \multicolumn{2}{c}{SCM \rom{4} ($\times10^{-2}$)} \\\hline
		DA Algorithm & $M=2$ & $M=11$ & $M=2$ & $M=11$ \\
		\hline
		CIP-mean & $22.2\pm26.3$ & $1.3\pm3.6$ & $14.2\pm25.3$ & $0.1\pm0.9$ \\
		CIP-MMD  & $6.2\pm7.1$ & $2.1\pm5.2$ & $2.0\pm4.7$ & $0.3\pm2.1$ \\
		IW-CIP-mean & $41.1\pm50.1$ & $2.6\pm6.9$ & $25.6\pm29.9$ & $5.1\pm6.7$ \\
		IW-CIP-MMD  & $30.2\pm31.6$ & $6.6\pm16.8$ & $11.4\pm20.8$ & $5.3\pm9.6$ \\
		\bottomrule
	\end{tabular}
	\caption{Risk difference $\riskhtar{h}-\riskhbar{h}{\wh}$ between source and target in linear SCMs. These differences are evaluated under varying number of source domains $M$.}
	\label{tab:SEM_different_M_split}
\end{table}

The minimum number of source domains necessary for finding CICs depends on the types of interventions and the hypothesis class in use. \propref{linear_sem_example} and~\lemref{linear_sem_example} show that the number of source domains needs to be greater than the dimension of the perturbations, which is the case for linear SCMs above. In this context, we also consider another SCM under binary interventions, as described below. The purpose is to demonstrate that the required number of source domains that enables successful finding of CICs via CIP can vary when the hypothesis class changes. 
\begin{itemize}
    \item SCM binary: The data generation model is
    \begin{align*}
        \Ysrc &= \text{Bernoulli}(0.5), \\
        \coordw{\Xsrc}{1:5} &= 0.2(\Ysrc-0.5) \cdot \mathds{1}_5 + 0.4 \cdot\Normal(0, \Ind_5),\\
        \coordw{\Xsrc}{6:10} &= 0.2(\Ysrc-0.5) \cdot \mathds{1}_5 + 0.4 \cdot\Normal(0, \Ind_5) +  \Asrc,
    \end{align*}
    where $\Asrc$ is randomly sampled from $\{a\in\real^5\mid a_i\in\{-1,1\}\}$ ($1\leq m\leq M$) without replacement in source domains, and the target domain suffers larger intervention with $\Atar=(2,2,2,2,2)^\top$.
\end{itemize}

\begin{table}[t]
	\centering
	\small
	\begin{tabular}{l | ccc}
		\toprule
		&   \multicolumn{3}{c}{CIP-mean}  \\\hline
		\#source domains & src\_acc & tar\_acc & risk\_diff \\\hline
		$M=2$ & $85.1\pm1.5$ & $50.6\pm1.6$ & $1.8\pm1.0$ \\
		$M=3$ & $84.7\pm2.6$ & $54.3\pm7.9$ & $9.1\pm6.3\times10^{-1}$ \\
		$M=4$ & $82.4\pm2.3$ & $62.3\pm11.2$ & $4.1\pm5.1\times10^{-1}$ \\
		$M=5$ & $81.9\pm2.0$ & $63.4\pm12.2$ & $2.9\pm4.3\times10^{-1}$ \\
		$M=6$ & $80.6\pm1.5$ & $72.2\pm10.8$ & $7.7\pm13.2\times10^{-2}$ \\
		$M=7$ & $80.0\pm0.5$ & $78.6\pm1.6$ & $6.7\pm9.9\times10^{-3}$ \\
		\bottomrule
	\end{tabular}
	
	\vspace{1em}
	
	\begin{tabular}{l | ccc}
		\toprule
		&	 \multicolumn{3}{c}{CIP-MMD} \\\hline
		\#source domains & src\_acc & tar\_acc & risk\_diff \\\hline
		$M=2$ & $84.5\pm1.0$ & $51.7\pm3.4$ & $2.4\pm1.4$ \\
		$M=3$ & $83.1\pm1.7$ & $58.4\pm9.9$ & $1.3\pm1.2$ \\
		$M=4$ & $81.9\pm1.5$ & $67.9\pm11.5$ & $5.2\pm7.0\times10^{-1}$ \\
		$M=5$ & $81.7\pm1.5$ & $69.8\pm11.4$ & $3.8\pm5.6\times10^{-1}$ \\
		$M=6$ & $81.3\pm0.9$ & $73.0\pm8.7$ & $2.5\pm4.7\times10^{-1}$ \\
		$M=7$ & $81.1\pm0.7$ & $73.9\pm5.0$ & $1.1\pm1.1\times10^{-1}$ \\
		\bottomrule
	\end{tabular}
	\caption{Performance of CIP under a linear model on SCM-binary. The number of source domains varies from $M=2$ to $M=7$.}
	\label{tab:SEM_binary_linear_different_M}
\end{table}

We use two different models in this experiment: a linear model and a two-layer neural network. For the linear model, as shown in~\propref{linear_sem_example} and~\lemref{linear_sem_example}, the number of source domains should be at least on the order of the dimension of intervention plus one, that is, $M\geq6$. We vary the number of source domains $M$ from $2$ to $7$ and evaluate the performance of CIP-mean and CIP-MMD in Table~\ref{tab:SEM_binary_linear_different_M} (there is no label shift between source and target domains). We find that when $M\geq6$ the target accuracy and risk are close to those in the source domains. We also observe that using more source domains than the theoretically required minimum number of source domains can further improve performance in practice.

For the two-layer neural network, we vary $M$ from $2^1$ to $2^5$. In Table~\ref{tab:SEM_binary_different_M}, we observe that when $M=32$, the target accuracy and target risk are close to the source accuracy and source risk. This suggests that if the hypothesis class is nonlinear and complex, CIP may need more domains to accurately identify CICs. In this case, $M=2^5=32$ source domains covering all possible binary interventions are required.


\begin{table}[t]
	\centering
	\small
	\begin{tabular}{l | ccc}
		\toprule
		&   \multicolumn{3}{c}{CIP-mean}  \\\hline
		\#src domains & src\_acc & tar\_acc & risk\_diff \\\hline
		$M=2$ & $85.3\pm1.7$ & $50.8\pm1.7$ & $2.8\pm1.8$ \\
		$M=4$ & $84.1\pm1.3$ & $55.5\pm8.7$ & $9.3\pm7.0\times10^{-1}$ \\
		$M=8$ & $82.5\pm1.4$ & $63.8\pm12.1$ & $3.1\pm4.2\times10^{-1}$ \\
		$M=16$ & $80.8\pm0.3$ & $76.4\pm4.3$ & $2.7\pm2.3\times10^{-2}$ \\
		$M=32$ & $80.5\pm0.4$ & $78.4\pm1.2$ & $1.6\pm1.1\times10^{-2}$ \\
		\bottomrule
	\end{tabular}
	
	\vspace{1em}
	
	\begin{tabular}{l | ccc}
		\toprule
		&	 \multicolumn{3}{c}{CIP-MMD} \\\hline
		\#src domains & src\_acc & tar\_acc & risk\_diff \\\hline
		$M=2$ & $86.5\pm1.9$ & $50.6\pm1.6$ & $2.7\pm1.1$ \\
		$M=4$ & $85.2\pm1.3$ & $53.5\pm8.4$ & $2.1\pm1.0$ \\
		$M=8$ & $83.3\pm1.4$ & $59.7\pm11.3$ & $9.5\pm7.2\times10^{-1}$ \\
		$M=16$ & $81.1\pm0.5$ & $72.1\pm7.4$ & $1.5\pm1.8\times10^{-1}$ \\
		$M=32$ & $80.6\pm0.3$ & $78.7\pm1.2$ & $2.6\pm1.6\times10^{-2}$ \\
		\bottomrule
	\end{tabular}
	\caption{Performance of CIP under a two-layer neural network on SCM-binary. The number of source domains varies from $M=2^1$ to $M=2^5$.}
	\label{tab:SEM_binary_different_M}
\end{table}

\subsection{Effect of number of source domains in MNIST}

Next we compare the performance of CIP-based methods on MNIST with varying number of source domains ($M=3$ versus $M=5$). For $M=3$, we use the middle three domains with rotation angles of $-15^\circ, 0^\circ$, and $15^\circ$ from~\secref{mnist_rotation_intervention}; for $M=5$, we use all source domains mentioned in our experiments from~\secref{mnist_rotation_intervention}. Table~\ref{tab:MNIST_different_M} shows the risk difference between source and target domains for four different CIP-based methods. Similar to our findings in the SCM experiments, the results demonstrate that using a higher number of source domains helps in identifying CICs and reducing the risk difference between source and target.

\begin{table}[t]
	\centering
	\small
	\begin{tabular}{l | cc | cc}
		\toprule
		{} & \multicolumn{2}{c|}{MNIST \rom{1} ($\times10^{-2}$)} & \multicolumn{2}{c}{MNIST \rom{2} ($\times10^{-4}$)} \\\hline
		DA Algorithm & $M=3$ & $M=5$ & $M=3$ & $M=5$ \\
		\hline
		CIP-mean      & $43.4\pm8.9$ & $16.0\pm1.9$ & $31.5\pm3.4$ & $17.1\pm5.1$ \\
		CIP-MMD       & $60.1\pm4.1$ & $17.9\pm4.9$ & $65.3\pm13.6$ & $15.6\pm3.2$ \\
		IW-CIP-mean   & $58.2\pm7.5$ & $15.7\pm1.4$ & $33.3\pm4.9$ & $11.5\pm2.0$ \\
		IW-CIP-MMD    & $62.9\pm5.6$ & $20.3\pm6.9$ & $62.5\pm13.3$ & $14.1\pm2.2$ \\
		\bottomrule
	\end{tabular}
	
	\vspace{1em}
	
	\begin{tabular}{l | cc | cc}
		\toprule
		{} & \multicolumn{2}{c|}{MNIST \rom{3} ($\times10^{-4}$)} & \multicolumn{2}{c}{MNIST \rom{4} ($\times10^{-4}$)} \\\hline
		DA Algorithm & $M=3$ & $M=5$ & $M=3$ & $M=5$ \\
		\hline
		CIP-mean      & $45.0\pm9.6$ & $26.8\pm3.8$ & $44.4\pm8.7$ & $33.7\pm5.7$ \\
		CIP-MMD       & $47.5\pm6.8$ & $35.2\pm4.1$ & $50.1\pm10.6$ & $36.0\pm6.3$ \\
		IW-CIP-mean   & $40.3\pm5.5$ & $27.5\pm4.4$ & $43.5\pm6.5$ & $27.6\pm2.9$ \\
		IW-CIP-MMD    & $42.7\pm5.9$ & $36.3\pm4.3$ & $46.5\pm5.6$ & $35.3\pm5.9$ \\
		\bottomrule
	\end{tabular}
	
	\caption{Risk difference $\riskhtar{h}-\riskhbar{h}{\wh}$ between source and target in MNIST experiments. These differences are evaluated under varying number of source domains $M$.}
	\label{tab:MNIST_different_M}
\end{table}

\subsection{Comparison with other heuristic baselines}

We introduced JointDIP, a method specifically designed to address the label-flipping issue observed with DIP when label-flipping features exist between source and target domains. Our theoretical results support this claim by showing that, under certain conditions, JointDIP avoids learning label-flipping features and achieves strong generalization to the target domain, where DIP may still struggle. The advantage of JointDIP is further confirmed through numerical experiments conducted in~\secref{exp}.

One may wonder if there are simpler alternatives to designing DA algorithms that prevent label-flipping issues while being easier to implement than JointDIP, which incorporates CIP features into the matching penalty. To explore this, we consider the following two alternative methods.
\begin{itemize}
	\item \textbf{CIP-FT-DIP:} This method first implements CIP to learn CICs and the corresponding classifier based on CICs. The learned features and classifier are then used as an initialization for DIP, followed by fine-tuning using the DIP objective.
	\item \textbf{CIP-Sum-DIP:} This method combines the CIP and DIP matching penalties into a single objective by optimizing their sum. By doing so, it learns features that simultaneously satisfy the marginal invariance of covariates between source and target domains and the conditional distributional invariance across multiple source domains. Unlike CIP-FT-DIP, this approach integrates both objectives into one optimization problem instead of a two-step algorithm.
\end{itemize}

At first glance, both CIP-FT-DIP and CIP-Sum-DIP might appear capable of addressing label-flipping issues in the presence of label-flipping features. However, the following simple example illustrates that these algorithms can still learn label-flipping features. Consider a scenario where the source and target data are generated according to the following SCMs:
\begin{align*}
	\begin{cases}
		\Ysrcone \sim \text{Rad}(\{-1,+1\}),\\
		\Xsrcone_{[1]} = \Ysrcone + \epssrcone_{[1]},\\
		\Xsrcone_{[2]} =10 \Ysrcone + \epssrcone_{[2]},\\
		\Xsrcone_{[3]} = \Ysrcone + \epssrcone_{[3]},\\
	\end{cases}
	\begin{cases}
	\Ysrctwo \sim \text{Rad}(\{-1,+1\}),\\
	\Xsrctwo_{[1]} = \Ysrctwo + \epssrctwo_{[1]},\\
	\Xsrctwo_{[2]} =10 \Ysrctwo + 2+ \epssrctwo_{[2]},\\
	\Xsrctwo_{[3]} = -2\Ysrctwo + \epssrctwo_{[3]},\\
\end{cases}
\begin{cases}
	\Ytar \sim \text{Rad}(\{-1,+1\}),\\
	\Xtar_{[1]} = \Ytar + \epstar_{[1]},\\
	\Xtar_{[2]} =-10 \Ytar + \epstar_{[2]},\\
	\Xtar_{[3]} = \Ytar + \epstar_{[3]},\\
\end{cases}	
\end{align*}
where $\text{Rad}(\{-1,+1\})$ denotes the Rademacher distribution defined on $\{-1,+1\}$ with equal probability, and $\epssrcone,\epssrctwo,\epstar\sim \mathcal{N}(0,\mathds{1}_{3})$. In this example, $X_{[1]}$ is a CIC that CIP can correctly identify. $X_{[2]}$ represents a label-flipping feature between the first source and target domains, while $X_{[3]}$ is a feature that helps prediction between the first source and the target domain. Here the first source domain is used for DIP.

Assume that we use the feature selector as the feature mapping. For CIP-FT-DIP, if we initialize with CIP by selecting feature $X_{[1]}$ and then fine-tune with the DIP objective, $X_{[2]}$, which satisfies the DIP penalty for marginal invariance and is more predictive of $Y$ compared to $X_{[1]}$ and $X_{[3]}$, will quickly be learned during fine-tuning. As a result, CIP-FT-DIP ends up learning $X_{[2]}$, despite being initialized with $X_{[1]}$, thereby inheriting the label-flipping issue of DIP. For CIP-Sum-DIP, enforcing features to simultaneously satisfy both CIP and DIP penalties restricts the selection to $X_{[1]}$, leading to results similar to CIP. However, this approach ignores the potential benefits of incorporating $X_{[3]}$, limiting its effectiveness in improving target performance. In contrast, JointDIP, when applied to the first source and target domains, can learn both $X_{[1]}$ and $X_{[3]}$, allowing to avoid the label-flipping issue associated with $X_{[2]}$ while leveraging $X_{[3]}$ to enhance target performance. As a result, JointDIP outperforms both CIP-FT-DIP and CIP-Sum-DIP in this scenario.

We further conduct experiments on MNIST and CelebA to numerically compare the two methods, CIP-FT-DIP and CIP-Sum-DIP, against other DA methods such as DIP, CIP, and JointDIP. For CIP-FT-DIP, we first run CIP to learn the weights, which are then used to initialize the architecture before fine-tuning with the DIP objective. For CIP-Sum-DIP, the DIP objective is modified by adding the CIP matching penalty, combining the two objectives into a single optimization process. The training details, including architectures, batch sizes, and other settings, follow the same configurations as in our main experiments (see~\appref{architec}). For hyperparameter selection of $\lamcip$ and $\lamdip$, we vary over $\{0.01,0.1,1\}$ for both parameters and select the best-performing one using the protocol described in~\appref{hyper}.

The results on MNIST are presented in~\tabref{MNIST_results_additional}. In this experiment, MNIST \rom{2} and \rom{4} involve label shift interventions, and we do not implement the importance-weighting based label shift correction (IW) step for CIP-FT-DIP and CIP-Sum-DIP. To ensure a fair comparison, we also exclude IW-DIP, IW-CIP, and IW-JointDIP. We observe that CIP-FT-DIP performs comparably to CIP across various settings but shows slightly worse performance than DIP. In cases where label-flipping features are present, such as in MNIST \rom{3} and \rom{4}, the performance drops compared to settings without label-flipping features. This behavior is similar to that of DIP, indicating that CIP-FT-DIP is not as effective as DIP in preventing the learning of label-flipping features. CIP-Sum-DIP exhibits similar behaviors but performs significantly worse overall compared to other DA methods. A potential explanation is that the CIP and DIP matching penalties may have interactions that make the method more sensitive to the choice of hyperparameter settings. As a result, CIP-Sum-DIP likely requires more careful hyperparameter tuning than other DA methods that rely on a single matching penalty.

\begin{table}[t]
	\centering
	\small
	\begin{tabular}{l | l | l | l | l}
		\toprule
		{}  & \multicolumn{1}{c|}{MNIST \rom{1}} & \multicolumn{1}{c|}{MNIST \rom{2}} & \multicolumn{1}{c|}{MNIST \rom{3}} & \multicolumn{1}{c}{MNIST \rom{4}} \\
		\hline
		Rotation shift        & Y & Y & Y & Y \\
		Label shift           & N & Y & N & Y \\
		Label-flipping features & N & N & Y & Y \\
		\hline
		DA Algorithm & tar\_acc & tar\_acc & tar\_acc & tar\_acc \\
		\hline
		DIP*         & 97.5$\pm$0.3 & 96.7$\pm$0.5 & 91.0$\pm$1.9 & 90.1$\pm$1.6 \\
		CIP*         & 94.9$\pm$0.9 & 94.9$\pm$2.1 & 89.1$\pm$1.7 & 90.0$\pm$1.3 \\
		JointDIP*    & 97.3$\pm$0.3 & 96.9$\pm$0.6 & 93.5$\pm$1.2 & 91.5$\pm$1.1 \\
		CIP-FT-DIP   & 95.1$\pm$1.4 & 92.6$\pm$0.8 & 91.0$\pm$1.4 & 89.3$\pm$1.5 \\
		CIP-Sum-DIP  & 91.6$\pm$0.9 & 89.3$\pm$1.6 & 80.8$\pm$3.1 & 79.5$\pm$3.2 \\
		\bottomrule
	\end{tabular}
	\caption{Target accuracy on rotated MNIST for CIP-FT-DIP and CIP-Sum-DIP. For comparison with other DA algorithms, we include results from~\tabref{MNIST_results_1_2},~\tabref{MNIST_results_3_4}, marked with *. The algorithms are run on 10 different datasets, each generated using 10 different random seeds.}
	\label{tab:MNIST_results_additional}
\end{table}

The results on CelebA are presented in~\tabref{CelebA_mechan_additional}. In this setting, there is no label shift, but the features are manipulated such that the Color\_Balance feature maintains a stable correlation with the label between the last source and target domains, while the Mouth\_Slightly\_Open feature serves as a label-flipping feature. CIP-FT-DIP and CIP-Sum-DIP show comparable or slightly worse performance than DIP and CIP on CelebA \rom{1} and CelebA \rom{2}. Interestingly, both methods outperform all other DA methods on CelebA \rom{3} by a significant margin. On CelebA \rom{3}, the correlation between the Mouth\_Slightly\_Open feature and the label is high, allowing DIP to rely on this feature for predictions, which leads to poor prediction performance. JointDIP demonstrates higher accuracy, highlighting its reliability compared to DIP. Moreover, the performance of CIP-FT-DIP and CIP-Sum-DIP in this setting indicates that combining CIP with DIP can further enhance the reliability of algorithms in mitigating the influence of label-flipping features in this example.


\begin{table}[t]
	\centering
	\small
	\begin{tabular}{l | c | c | c}
		\toprule
		& CelebA \rom{1} & CelebA \rom{2} & CelebA \rom{3} \\
		\hline
		DA Algorithm & tar\_acc & tar\_acc & tar\_acc \\
		\hline
		DIP*          & 82.0$\pm$2.2   & 90.6$\pm$1.0   & 82.0$\pm$2.2 \\
		CIP*          & 76.6$\pm$1.4   & 83.7$\pm$1.4   & 89.7$\pm$0.7 \\
		JointDIP*     & 87.3$\pm$0.6   & 90.4$\pm$1.2   & 88.7$\pm$1.0 \\
		CIP-FT-DIP    & 81.5$\pm$2.8   & 79.7$\pm$2.6   & 92.3$\pm$0.4 \\
		CIP-Sum-DIP   & 77.6$\pm$2.1   & 79.5$\pm$1.8   & 92.9$\pm$0.3 \\
		\bottomrule
	\end{tabular}
	\caption{Target accuracy on CelebA \rom{1}, \rom{2}, \rom{3} for CIP-FT-DIP and CIP-Sum-DIP. For comparison with other DA algorithms, we include results from~\tabref{CelebA_mechan}, marked with *. For each setting of the CelebA, DA algorithms are run on 10 different datasets, each generated using 10 different random seeds.}
	\label{tab:CelebA_mechan_additional}
\end{table}

\bibliography{DA.bib}

\begin{thebibliography}{98}
\providecommand{\natexlab}[1]{#1}
\providecommand{\url}[1]{\texttt{#1}}
\expandafter\ifx\csname urlstyle\endcsname\relax
  \providecommand{\doi}[1]{doi: #1}\else
  \providecommand{\doi}{doi: \begingroup \urlstyle{rm}\Url}\fi

\bibitem[Amini and Gallinari(2003)]{amini2003semi}
M.~Amini and P.~Gallinari.
\newblock Semi-supervised learning with an explicit label-error model for
  misclassified data.
\newblock In \emph{International Joint Conference on Artificial Intelligence
  (IJCAI)}, pages 555--560, 2003.

\bibitem[Arjovsky et~al.(2019)Arjovsky, Bottou, Gulrajani, and
  Lopez-Paz]{arjovsky2019invariant}
M.~Arjovsky, L.~Bottou, I.~Gulrajani, and D.~Lopez-Paz.
\newblock Invariant risk minimization.
\newblock \emph{arXiv preprint arXiv:1907.02893}, 2019.

\bibitem[Azizzadenesheli et~al.(2019)Azizzadenesheli, Liu, Yang, and
  Anandkumar]{azizzadenesheli2019regularized}
K.~Azizzadenesheli, A.~Liu, F.~Yang, and A.~Anandkumar.
\newblock Regularized learning for domain adaptation under label shifts.
\newblock In \emph{International Conference on Learning Representations
  (ICLR)}, 2019.

\bibitem[Baktashmotlagh et~al.(2013)Baktashmotlagh, Harandi, Lovell, and
  Salzmann]{baktashmotlagh2013unsupervised}
M.~Baktashmotlagh, M.~T. Harandi, B.~C. Lovell, and M.~Salzmann.
\newblock Unsupervised domain adaptation by domain invariant projection.
\newblock In \emph{IEEE International Conference on Computer Vision (ICCV)},
  pages 769--776, 2013.

\bibitem[Ben-David et~al.(2006)Ben-David, Blitzer, Crammer, and
  Pereira]{ben2006analysis}
S.~Ben-David, J.~Blitzer, K.~Crammer, and F.~Pereira.
\newblock Analysis of representations for domain adaptation.
\newblock \emph{Advances in Neural Information Processing Systems (NeurIPS)},
  19, 2006.

\bibitem[Ben-David et~al.(2010)Ben-David, Blitzer, Crammer, Kulesza, Pereira,
  and Vaughan]{ben2010theory}
S.~Ben-David, J.~Blitzer, K.~Crammer, A.~Kulesza, F.~Pereira, and J.~W.
  Vaughan.
\newblock A theory of learning from different domains.
\newblock \emph{Machine Learning}, 79\penalty0 (1--2):\penalty0 151--175, 2010.

\bibitem[Ben-Tal et~al.(2009)Ben-Tal, El~Ghaoui, and Nemirovski]{ben2009robust}
A.~Ben-Tal, L.~El~Ghaoui, and A.~Nemirovski.
\newblock \emph{Robust Optimization}, volume~28.
\newblock Princeton University Press, 2009.

\bibitem[Ben-Tal et~al.(2013)Ben-Tal, Den~Hertog, De~Waegenaere, Melenberg, and
  Rennen]{ben2013robust}
A.~Ben-Tal, D.~Den~Hertog, A.~De~Waegenaere, B.~Melenberg, and G.~Rennen.
\newblock Robust solutions of optimization problems affected by uncertain
  probabilities.
\newblock \emph{Management Science}, 59\penalty0 (2):\penalty0 341--357, 2013.

\bibitem[B{\"u}hlmann(2020)]{buhlmann2020invariance}
P.~B{\"u}hlmann.
\newblock Invariance, causality and robustness.
\newblock \emph{Statistical Science}, 35\penalty0 (3):\penalty0 404--426, 2020.

\bibitem[Cadrin-Ch{\^e}nevert(2022)]{cadrin2022moving}
A.~Cadrin-Ch{\^e}nevert.
\newblock Moving from imagenet to radimagenet for improved transfer learning
  and generalizability.
\newblock \emph{Radiology: Artificial Intelligence}, 4\penalty0 (5):\penalty0
  e220126, 2022.

\bibitem[Chapelle et~al.(2009)Chapelle, Schölkopf, and Zien]{chapelle2009semi}
O.~Chapelle, B.~Schölkopf, and A.~Zien.
\newblock Semi-supervised learning.
\newblock \emph{IEEE Transactions on Neural Networks}, 20\penalty0
  (3):\penalty0 542--542, 2009.

\bibitem[Chen et~al.(2022)Chen, Zaharia, and Zou]{chen2022estimating}
L.~Chen, M.~Zaharia, and J.~Y. Zou.
\newblock Estimating and explaining model performance when both covariates and
  labels shift.
\newblock \emph{Advances in Neural Information Processing Systems (NeurIPS)},
  35:\penalty0 11467--11479, 2022.

\bibitem[Chen and B{\"u}hlmann(2021)]{chen2021domain}
Y.~Chen and P.~B{\"u}hlmann.
\newblock Domain adaptation under structural causal models.
\newblock \emph{Journal of Machine Learning Research}, 22\penalty0
  (261):\penalty0 1--80, 2021.

\bibitem[Chung and Lu(2006)]{chung2006concentration}
F.~Chung and L.~Lu.
\newblock Concentration inequalities and martingale inequalities: A survey.
\newblock \emph{Internet Mathematics}, 3\penalty0 (1):\penalty0 79--127, 2006.

\bibitem[Cortes and Mohri(2011)]{cortes2011domain}
C.~Cortes and M.~Mohri.
\newblock Domain adaptation in regression.
\newblock In \emph{International Conference on Algorithmic Learning Theory
  (ALT)}, pages 308--323. Springer, 2011.

\bibitem[Cortes and Mohri(2014)]{cortes2014domain}
C.~Cortes and M.~Mohri.
\newblock Domain adaptation and sample bias correction theory and algorithm for
  regression.
\newblock \emph{Theoretical Computer Science}, 519:\penalty0 103--126, 2014.

\bibitem[Cortes et~al.(2010)Cortes, Mansour, and Mohri]{cortes2010learning}
C.~Cortes, Y.~Mansour, and M.~Mohri.
\newblock Learning bounds for importance weighting.
\newblock \emph{Advances in Neural Information Processing Systems (NeurIPS)},
  23, 2010.

\bibitem[Courty et~al.(2016)Courty, Flamary, Tuia, and
  Rakotomamonjy]{courty2016optimal}
N.~Courty, R.~Flamary, D.~Tuia, and A.~Rakotomamonjy.
\newblock Optimal transport for domain adaptation.
\newblock \emph{IEEE Transactions on Pattern Analysis and Machine
  Intelligence}, 39\penalty0 (9):\penalty0 1853--1865, 2016.

\bibitem[Courty et~al.(2017)Courty, Flamary, Habrard, and
  Rakotomamonjy]{courty2017joint}
N.~Courty, R.~Flamary, A.~Habrard, and A.~Rakotomamonjy.
\newblock Joint distribution optimal transportation for domain adaptation.
\newblock \emph{Advances in Neural Information Processing Systems (NeurIPS)},
  30, 2017.

\bibitem[Dempster et~al.(1977)Dempster, Laird, and Rubin]{dempster1977maximum}
A.~P. Dempster, N.~M. Laird, and D.~B. Rubin.
\newblock Maximum likelihood from incomplete data via the {EM} algorithm.
\newblock \emph{Journal of the Royal Statistical Society: Series B
  (Methodological)}, 39\penalty0 (1):\penalty0 1--22, 1977.

\bibitem[Deng et~al.(2009)Deng, Dong, Socher, Li, Li, and
  Fei-Fei]{deng2009imagenet}
J.~Deng, W.~Dong, R.~Socher, L.-J. Li, K.~Li, and L.~Fei-Fei.
\newblock Imagenet: A large-scale hierarchical image database.
\newblock In \emph{IEEE Conference on Computer Vision and Pattern Recognition
  (CVPR)}, pages 248--255. IEEE, 2009.

\bibitem[Duchi et~al.(2021)Duchi, Glynn, and Namkoong]{duchi2021statistics}
J.~C. Duchi, P.~W. Glynn, and H.~Namkoong.
\newblock Statistics of robust optimization: A generalized empirical likelihood
  approach.
\newblock \emph{Mathematics of Operations Research}, 46\penalty0 (3):\penalty0
  946--969, 2021.

\bibitem[Finn et~al.(2017)Finn, Abbeel, and Levine]{finn2017model}
C.~Finn, P.~Abbeel, and S.~Levine.
\newblock Model-agnostic meta-learning for fast adaptation of deep networks.
\newblock In \emph{International Conference on Machine Learning (ICML)}, pages
  1126--1135. PMLR, 2017.

\bibitem[Ganin et~al.(2016)Ganin, Ustinova, Ajakan, Germain, Larochelle,
  Laviolette, Marchand, and Lempitsky]{ganin2016domain}
Y.~Ganin, E.~Ustinova, H.~Ajakan, P.~Germain, H.~Larochelle, F.~Laviolette,
  M.~Marchand, and V.~Lempitsky.
\newblock Domain-adversarial training of neural networks.
\newblock \emph{The Journal of Machine Learning Research}, 17\penalty0
  (1):\penalty0 2096--2030, 2016.

\bibitem[Garg et~al.(2020)Garg, Wu, Balakrishnan, and Lipton]{garg2020unified}
S.~Garg, Y.~Wu, S.~Balakrishnan, and Z.~Lipton.
\newblock A unified view of label shift estimation.
\newblock \emph{Advances in Neural Information Processing Systems (NeurIPS)},
  33:\penalty0 3290--3300, 2020.

\bibitem[Gong et~al.(2012)Gong, Shi, Sha, and Grauman]{gong2012geodesic}
B.~Gong, Y.~Shi, F.~Sha, and K.~Grauman.
\newblock Geodesic flow kernel for unsupervised domain adaptation.
\newblock In \emph{IEEE Conference on Computer Vision and Pattern Recognition
  (CVPR)}, pages 2066--2073. IEEE, 2012.

\bibitem[Gong et~al.(2016)Gong, Zhang, Liu, Tao, Glymour, and
  Sch{\"o}lkopf]{gong2016domain}
M.~Gong, K.~Zhang, T.~Liu, D.~Tao, C.~Glymour, and B.~Sch{\"o}lkopf.
\newblock Domain adaptation with conditional transferable components.
\newblock In \emph{International Conference on Machine Learning (ICML)}, pages
  2839--2848. PMLR, 2016.

\bibitem[Gopalan et~al.(2011)Gopalan, Li, and Chellappa]{gopalan2011domain}
R.~Gopalan, R.~Li, and R.~Chellappa.
\newblock Domain adaptation for object recognition: An unsupervised approach.
\newblock In \emph{IEEE International Conference on Computer Vision (ICCV)},
  pages 999--1006. IEEE, 2011.

\bibitem[Gretton et~al.(2012)Gretton, Borgwardt, Rasch, Sch{\"o}lkopf, and
  Smola]{gretton2012kernel}
A.~Gretton, K.~M. Borgwardt, M.~J. Rasch, B.~Sch{\"o}lkopf, and A.~Smola.
\newblock A kernel two-sample test.
\newblock \emph{The Journal of Machine Learning Research}, 13\penalty0
  (1):\penalty0 723--773, 2012.

\bibitem[Gulrajani and Lopez-Paz(2021)]{gulrajani2020search}
I.~Gulrajani and D.~Lopez-Paz.
\newblock In search of lost domain generalization.
\newblock In \emph{International Conference on Learning Representations
  (ICLR)}, 2021.

\bibitem[He et~al.(2016)He, Zhang, Ren, and Sun]{he2016deep}
K.~He, X.~Zhang, S.~Ren, and J.~Sun.
\newblock Deep residual learning for image recognition.
\newblock In \emph{IEEE Conference on Computer Vision and Pattern Recognition
  (CVPR)}, pages 770--778, 2016.

\bibitem[He et~al.(2019)He, Girshick, and Doll{\'a}r]{he2019rethinking}
K.~He, R.~Girshick, and P.~Doll{\'a}r.
\newblock Rethinking imagenet pre-training.
\newblock In \emph{IEEE International Conference on Computer Vision (ICCV)},
  pages 4918--4927, 2019.

\bibitem[Heinze-Deml and Meinshausen(2021)]{heinze2021conditional}
C.~Heinze-Deml and N.~Meinshausen.
\newblock Conditional variance penalties and domain shift robustness.
\newblock \emph{Machine Learning}, 110\penalty0 (2):\penalty0 303--348, 2021.

\bibitem[Hoffman et~al.(2018{\natexlab{a}})Hoffman, Mohri, and
  Zhang]{hoffman2018algorithms}
J.~Hoffman, M.~Mohri, and N.~Zhang.
\newblock Algorithms and theory for multiple-source adaptation.
\newblock \emph{Advances in Neural Information Processing Systems (NeurIPS)},
  31, 2018{\natexlab{a}}.

\bibitem[Hoffman et~al.(2018{\natexlab{b}})Hoffman, Tzeng, Park, Zhu, Isola,
  Saenko, Efros, and Darrell]{hoffman2018cycada}
J.~Hoffman, E.~Tzeng, T.~Park, J.-Y. Zhu, P.~Isola, K.~Saenko, A.~Efros, and
  T.~Darrell.
\newblock Cycada: Cycle-consistent adversarial domain adaptation.
\newblock In \emph{International Conference on Machine Learning (ICML)}, pages
  1989--1998. PMLR, 2018{\natexlab{b}}.

\bibitem[Hu et~al.(2018)Hu, Niu, Sato, and Sugiyama]{hu2018does}
W.~Hu, G.~Niu, I.~Sato, and M.~Sugiyama.
\newblock Does distributionally robust supervised learning give robust
  classifiers?
\newblock In \emph{International Conference on Machine Learning (ICML)}, pages
  2029--2037. PMLR, 2018.

\bibitem[Huang et~al.(2017)Huang, Liu, Van Der~Maaten, and
  Weinberger]{huang2017densely}
G.~Huang, Z.~Liu, L.~Van Der~Maaten, and K.~Q. Weinberger.
\newblock Densely connected convolutional networks.
\newblock In \emph{IEEE Conference on Computer Vision and Pattern Recognition
  (CVPR)}, pages 4700--4708, 2017.

\bibitem[Jiang et~al.(2020)Jiang, Lao, Matwin, and Havaei]{jiang2020implicit}
X.~Jiang, Q.~Lao, S.~Matwin, and M.~Havaei.
\newblock Implicit class-conditioned domain alignment for unsupervised domain
  adaptation.
\newblock In \emph{International Conference on Machine Learning (ICML)}, pages
  4816--4827. PMLR, 2020.

\bibitem[Jiang and Veitch(2022)]{jiang2022invariant}
Y.~Jiang and V.~Veitch.
\newblock Invariant and transportable representations for anti-causal domain
  shifts.
\newblock \emph{Advances in Neural Information Processing Systems},
  35:\penalty0 20782--20794, 2022.

\bibitem[Johansson et~al.(2019)Johansson, Sontag, and
  Ranganath]{johansson2019support}
F.~D. Johansson, D.~Sontag, and R.~Ranganath.
\newblock Support and invertibility in domain-invariant representations.
\newblock In \emph{International Conference on Artificial Intelligence and
  Statistics (AISTATS)}, pages 527--536. PMLR, 2019.

\bibitem[Kamath et~al.(2021)Kamath, Tangella, Sutherland, and
  Srebro]{kamath2021does}
P.~Kamath, A.~Tangella, D.~Sutherland, and N.~Srebro.
\newblock Does invariant risk minimization capture invariance?
\newblock In \emph{International Conference on Artificial Intelligence and
  Statistics (AISTATS)}, pages 4069--4077. PMLR, 2021.

\bibitem[Koh et~al.(2021)Koh, Sagawa, Marklund, Xie, Zhang, Balsubramani, Hu,
  Yasunaga, Phillips, Gao, et~al.]{koh2021wilds}
P.~W. Koh, S.~Sagawa, H.~Marklund, S.~M. Xie, M.~Zhang, A.~Balsubramani, W.~Hu,
  M.~Yasunaga, R.~L. Phillips, I.~Gao, et~al.
\newblock Wilds: A benchmark of in-the-wild distribution shifts.
\newblock In \emph{International Conference on Machine Learning (ICML)}, pages
  5637--5664. PMLR, 2021.

\bibitem[Komura and Ishikawa(2018)]{komura2018machine}
D.~Komura and S.~Ishikawa.
\newblock Machine learning methods for histopathological image analysis.
\newblock \emph{Computational and Structural Biotechnology Journal},
  16:\penalty0 34--42, 2018.

\bibitem[Koyama and Yamaguchi(2020)]{koyama2020invariance}
M.~Koyama and S.~Yamaguchi.
\newblock When is invariance useful in an out-of-distribution generalization
  problem?
\newblock \emph{arXiv preprint arXiv:2008.01883}, 2020.

\bibitem[Krueger et~al.(2021)Krueger, Caballero, Jacobsen, Zhang, Binas, Zhang,
  Le~Priol, and Courville]{krueger2021out}
D.~Krueger, E.~Caballero, J.-H. Jacobsen, A.~Zhang, J.~Binas, D.~Zhang,
  R.~Le~Priol, and A.~Courville.
\newblock Out-of-distribution generalization via risk extrapolation ({RE}x).
\newblock In \emph{International Conference on Machine Learning (ICML)}, pages
  5815--5826. PMLR, 2021.

\bibitem[Kumar et~al.(2020)Kumar, Ma, and Liang]{kumar2020understanding}
A.~Kumar, T.~Ma, and P.~Liang.
\newblock Understanding self-training for gradual domain adaptation.
\newblock In \emph{International Conference on Machine Learning (ICML)}, pages
  5468--5479. PMLR, 2020.

\bibitem[LeCun(1998)]{lecun1998mnist}
Y.~LeCun.
\newblock The mnist database of handwritten digits.
\newblock \emph{http://yann.lecun.com/exdb/mnist/}, 1998.

\bibitem[LeCun et~al.(1998)LeCun, Bottou, Bengio, and
  Haffner]{lecun1998gradient}
Y.~LeCun, L.~Bottou, Y.~Bengio, and P.~Haffner.
\newblock Gradient-based learning applied to document recognition.
\newblock \emph{Proceedings of the IEEE}, 86\penalty0 (11):\penalty0
  2278--2324, 1998.

\bibitem[Li et~al.(2018{\natexlab{a}})Li, Gong, Tian, Liu, and
  Tao]{li2018domain}
Y.~Li, M.~Gong, X.~Tian, T.~Liu, and D.~Tao.
\newblock Domain generalization via conditional invariant representations.
\newblock In \emph{AAAI Conference on Artificial Intelligence (AAAI)},
  volume~32, 2018{\natexlab{a}}.

\bibitem[Li et~al.(2018{\natexlab{b}})Li, Tian, Gong, Liu, Liu, Zhang, and
  Tao]{li2018deep}
Y.~Li, X.~Tian, M.~Gong, Y.~Liu, T.~Liu, K.~Zhang, and D.~Tao.
\newblock Deep domain generalization via conditional invariant adversarial
  networks.
\newblock In \emph{European Conference on Computer Vision (ECCV)}, pages
  624--639, 2018{\natexlab{b}}.

\bibitem[Lipton et~al.(2018)Lipton, Wang, and Smola]{lipton2018detecting}
Z.~Lipton, Y.-X. Wang, and A.~Smola.
\newblock Detecting and correcting for label shift with black box predictors.
\newblock In \emph{International Conference on Machine Learning (ICML)}, pages
  3122--3130. PMLR, 2018.

\bibitem[Liu et~al.(2015)Liu, Luo, Wang, and Tang]{liu2015deep}
Z.~Liu, P.~Luo, X.~Wang, and X.~Tang.
\newblock Deep learning face attributes in the wild.
\newblock In \emph{IEEE International Conference on Computer Vision (ICCV)},
  pages 3730--3738, 2015.

\bibitem[Long et~al.(2017)Long, Zhu, Wang, and Jordan]{long2017deep}
M.~Long, H.~Zhu, J.~Wang, and M.~I. Jordan.
\newblock Deep transfer learning with joint adaptation networks.
\newblock In \emph{International Conference on Machine Learning (ICML)}, pages
  2208--2217. PMLR, 2017.

\bibitem[Long et~al.(2018)Long, Cao, Wang, and Jordan]{long2018conditional}
M.~Long, Z.~Cao, J.~Wang, and M.~I. Jordan.
\newblock Conditional adversarial domain adaptation.
\newblock \emph{Advances in Neural Information Processing Systems (NeurIPS)},
  31, 2018.

\bibitem[Magliacane et~al.(2018)Magliacane, Van~Ommen, Claassen, Bongers,
  Versteeg, and Mooij]{magliacane2018domain}
S.~Magliacane, T.~Van~Ommen, T.~Claassen, S.~Bongers, P.~Versteeg, and J.~M.
  Mooij.
\newblock Domain adaptation by using causal inference to predict invariant
  conditional distributions.
\newblock \emph{Advances in Neural Information Processing Systems (NeurIPS)},
  31, 2018.

\bibitem[Mansour et~al.(2009)Mansour, Mohri, and
  Rostamizadeh]{mansour2009domain}
Y.~Mansour, M.~Mohri, and A.~Rostamizadeh.
\newblock Domain adaptation: Learning bounds and algorithms.
\newblock In \emph{Conference on Learning Theory (COLT)}, 2009.

\bibitem[Mao et~al.(2017)Mao, Li, Xie, Lau, Wang, and Smolley]{mao2017least}
X.~Mao, Q.~Li, H.~Xie, R.~Y.~K. Lau, Z.~Wang, and S.~P. Smolley.
\newblock Least squares generative adversarial networks.
\newblock In \emph{IEEE International Conference on Computer Vision (ICCV)},
  pages 2794--2802, 2017.

\bibitem[Meinshausen(2018)]{meinshausen2018causality}
N.~Meinshausen.
\newblock Causality from a distributional robustness point of view.
\newblock In \emph{IEEE Data Science Workshop (DSW)}, pages 6--10. IEEE, 2018.

\bibitem[M{\"u}ller(1997)]{muller1997integral}
A.~M{\"u}ller.
\newblock Integral probability metrics and their generating classes of
  functions.
\newblock \emph{Advances in Applied Probability}, 29\penalty0 (2):\penalty0
  429--443, 1997.

\bibitem[Oren et~al.(2019)Oren, Sagawa, Hashimoto, and
  Liang]{oren2019distributionally}
Y.~Oren, S.~Sagawa, T.~B. Hashimoto, and P.~Liang.
\newblock Distributionally robust language modeling.
\newblock In \emph{Conference on Empirical Methods in Natural Language
  Processing and the 9th International Joint Conference on Natural Language
  Processing (EMNLP-IJCNLP)}, pages 4227--4237, 2019.

\bibitem[Pan et~al.(2008)Pan, Kwok, and Yang]{pan2008transfer}
S.~J. Pan, J.~T. Kwok, and Q.~Yang.
\newblock Transfer learning via dimensionality reduction.
\newblock In \emph{AAAI Conference on Artificial Intelligence (AAAI)},
  volume~8, pages 677--682, 2008.

\bibitem[Pan et~al.(2010)Pan, Tsang, Kwok, and Yang]{pan2010domain}
S.~J. Pan, I.~W. Tsang, J.~T. Kwok, and Q.~Yang.
\newblock Domain adaptation via transfer component analysis.
\newblock \emph{IEEE Transactions on Neural Networks}, 22\penalty0
  (2):\penalty0 199--210, 2010.

\bibitem[Pearl(2009)]{pearl2009causality}
J.~Pearl.
\newblock \emph{Causality}.
\newblock Cambridge University Press, 2009.

\bibitem[Peng et~al.(2019)Peng, Bai, Xia, Huang, Saenko, and
  Wang]{peng2019moment}
X.~Peng, Q.~Bai, X.~Xia, Z.~Huang, K.~Saenko, and B.~Wang.
\newblock Moment matching for multi-source domain adaptation.
\newblock In \emph{IEEE/CVF International Conference on Computer Vision
  (ICCV)}, pages 1406--1415, 2019.

\bibitem[Peters et~al.(2016)Peters, Bühlmann, and
  Meinshausen]{peters2016causal}
J.~Peters, P.~Bühlmann, and N.~Meinshausen.
\newblock Causal inference by using invariant prediction: Identification and
  confidence intervals.
\newblock \emph{Journal of the Royal Statistical Society: Series B (Statistical
  Methodology)}, 78\penalty0 (5):\penalty0 947--1012, 2016.

\bibitem[Quinonero-Candela et~al.(2008)Quinonero-Candela, Sugiyama,
  Schwaighofer, and Lawrence]{quinonero2008dataset}
J.~Quinonero-Candela, M.~Sugiyama, A.~Schwaighofer, and N.~D. Lawrence.
\newblock \emph{Dataset Shift in Machine Learning}.
\newblock MIT Press, 2008.

\bibitem[Redko et~al.(2020)Redko, Morvant, Habrard, Sebban, and
  Bennani]{redko2020survey}
I.~Redko, E.~Morvant, A.~Habrard, M.~Sebban, and Y.~Bennani.
\newblock A survey on domain adaptation theory: Learning bounds and theoretical
  guarantees.
\newblock \emph{arXiv preprint arXiv:2004.11829}, 2020.

\bibitem[Rosenfeld et~al.(2021)Rosenfeld, Ravikumar, and
  Risteski]{rosenfeld2020risks}
E.~Rosenfeld, P.~Ravikumar, and A.~Risteski.
\newblock The risks of invariant risk minimization.
\newblock In \emph{International Conference on Learning Representations
  (ICLR)}, 2021.

\bibitem[Rothenh{\"a}usler et~al.(2021)Rothenh{\"a}usler, Meinshausen,
  B{\"u}hlmann, and Peters]{rothenhausler2018anchor}
D.~Rothenh{\"a}usler, N.~Meinshausen, P.~B{\"u}hlmann, and J.~Peters.
\newblock Anchor regression: Heterogeneous data meet causality.
\newblock \emph{Journal of the Royal Statistical Society Series B: Statistical
  Methodology}, 83\penalty0 (2):\penalty0 215--246, 2021.

\bibitem[Sagawa et~al.(2020)Sagawa, Koh, Hashimoto, and
  Liang]{sagawa2019distributionally}
S.~Sagawa, P.~W. Koh, T.~B. Hashimoto, and P.~Liang.
\newblock Distributionally robust neural networks for group shifts: On the
  importance of regularization for worst-case generalization.
\newblock In \emph{International Conference on Learning Representations
  (ICLR)}, 2020.

\bibitem[Sagawa et~al.(2021)Sagawa, Koh, Lee, Gao, Xie, Shen, Kumar, Hu,
  Yasunaga, Marklund, et~al.]{sagawa2021extending}
S.~Sagawa, P.~W. Koh, T.~Lee, I.~Gao, S.~M. Xie, K.~Shen, A.~Kumar, W.~Hu,
  M.~Yasunaga, H.~Marklund, et~al.
\newblock Extending the wilds benchmark for unsupervised adaptation.
\newblock \emph{arXiv preprint arXiv:2112.05090}, 2021.

\bibitem[Sch{\"o}lkopf et~al.(2012)Sch{\"o}lkopf, Janzing, Peters, Sgouritsa,
  Zhang, and Mooij]{scholkopf2012causal}
B.~Sch{\"o}lkopf, D.~Janzing, J.~Peters, E.~Sgouritsa, K.~Zhang, and J.~Mooij.
\newblock On causal and anticausal learning.
\newblock In \emph{International Conference on Machine Learning (ICML)}, pages
  459--466, 2012.

\bibitem[Shi et~al.(2022)Shi, Seely, Torr, Siddharth, Hannun, Usunier, and
  Synnaeve]{shi2021gradient}
Y.~Shi, J.~Seely, P.~H.~S. Torr, N.~Siddharth, A.~Hannun, N.~Usunier, and
  G.~Synnaeve.
\newblock Gradient matching for domain generalization.
\newblock In \emph{International Conference on Learning Representations
  (ICLR)}, 2022.

\bibitem[Shimodaira(2000)]{shimodaira2000improving}
H.~Shimodaira.
\newblock Improving predictive inference under covariate shift by weighting the
  log-likelihood function.
\newblock \emph{Journal of Statistical Planning and Inference}, 90\penalty0
  (2):\penalty0 227--244, 2000.

\bibitem[Simard et~al.(1998)Simard, LeCun, Denker, and
  Victorri]{simard1998transformation}
P.~Y. Simard, Y.~A. LeCun, J.~S. Denker, and B.~Victorri.
\newblock Transformation invariance in pattern recognition: Tangent distance
  and tangent propagation.
\newblock In \emph{Neural Networks: Tricks of the Trade}, pages 239--274.
  Springer, 1998.

\bibitem[Sugiyama and Kawanabe(2012)]{sugiyama2012machine}
M.~Sugiyama and M.~Kawanabe.
\newblock \emph{Machine Learning in Non-Stationary Environments: Introduction
  to Covariate Shift Adaptation}.
\newblock MIT Press, 2012.

\bibitem[Sugiyama et~al.(2007)Sugiyama, Krauledat, and
  M{\"u}ller]{sugiyama2007covariate}
M.~Sugiyama, M.~Krauledat, and K.-R. M{\"u}ller.
\newblock Covariate shift adaptation by importance weighted cross validation.
\newblock \emph{Journal of Machine Learning Research}, 8\penalty0 (5), 2007.

\bibitem[Sun and Saenko(2016)]{sun2016deep}
B.~Sun and K.~Saenko.
\newblock Deep coral: Correlation alignment for deep domain adaptation.
\newblock In \emph{European Conference on Computer Vision (ECCV)}, pages
  443--450. Springer, 2016.

\bibitem[Tachet~des Combes et~al.(2020)Tachet~des Combes, Zhao, Wang, and
  Gordon]{tachet2020domain}
R.~Tachet~des Combes, H.~Zhao, Y.-X. Wang, and G.~J. Gordon.
\newblock Domain adaptation with conditional distribution matching and
  generalized label shift.
\newblock \emph{Advances in Neural Information Processing Systems},
  33:\penalty0 19276--19289, 2020.

\bibitem[Thrun and Pratt(2012)]{thrun2012learning}
S.~Thrun and L.~Pratt.
\newblock \emph{Learning to Learn}.
\newblock Springer Science \& Business Media, 2012.

\bibitem[Tropp(2015)]{tropp2015introduction}
J.~A. Tropp.
\newblock An introduction to matrix concentration inequalities.
\newblock \emph{Foundations and Trends in Machine Learning}, 8\penalty0
  (1--2):\penalty0 1--230, 2015.

\bibitem[Tzeng et~al.(2017)Tzeng, Hoffman, Saenko, and
  Darrell]{tzeng2017adversarial}
E.~Tzeng, J.~Hoffman, K.~Saenko, and T.~Darrell.
\newblock Adversarial discriminative domain adaptation.
\newblock In \emph{IEEE Conference on Computer Vision and Pattern Recognition
  (CVPR)}, pages 7167--7176, 2017.

\bibitem[Veitch et~al.(2021)Veitch, D'Amour, Yadlowsky, and
  Eisenstein]{veitch2021counterfactual}
V.~Veitch, A.~D'Amour, S.~Yadlowsky, and J.~Eisenstein.
\newblock Counterfactual invariance to spurious correlations: Why and how to
  pass stress tests.
\newblock In \emph{Conference on Neural Information Processing Systems
  (NeurIPS)}, pages 16196--16208, 2021.

\bibitem[Veta et~al.(2016)Veta, Van~Diest, Jiwa, Al-Janabi, and
  Pluim]{veta2016mitosis}
M.~Veta, P.~J. Van~Diest, M.~Jiwa, S.~Al-Janabi, and J.~P.~W. Pluim.
\newblock Mitosis counting in breast cancer: Object-level interobserver
  agreement and comparison to an automatic method.
\newblock \emph{PLOS ONE}, 11\penalty0 (8):\penalty0 e0161286, 2016.

\bibitem[Wainwright(2019)]{wainwright2019high}
M.~J. Wainwright.
\newblock \emph{High-Dimensional Statistics: A Non-Asymptotic Viewpoint},
  volume~48.
\newblock Cambridge University Press, 2019.

\bibitem[Wang et~al.(2022)Wang, Lan, Liu, Ouyang, Qin, Lu, Chen, Zeng, and
  Philip]{wang2022generalizing}
J.~Wang, C.~Lan, C.~Liu, Y.~Ouyang, T.~Qin, W.~Lu, Y.~Chen, W.~Zeng, and S.~Y.
  Philip.
\newblock Generalizing to unseen domains: A survey on domain generalization.
\newblock \emph{IEEE Transactions on Knowledge and Data Engineering},
  35\penalty0 (8):\penalty0 8052--8072, 2022.

\bibitem[Wang and Veitch(2022)]{wang2022unified}
Z.~Wang and V.~Veitch.
\newblock A unified causal view of domain invariant representation learning.
\newblock \emph{arXiv preprint arXiv:2208.06987}, 2022.

\bibitem[Wei et~al.(2021)Wei, Shen, Chen, and Ma]{wei2020theoretical}
C.~Wei, K.~Shen, Y.~Chen, and T.~Ma.
\newblock Theoretical analysis of self-training with deep networks on unlabeled
  data.
\newblock In \emph{International Conference on Learning Representations
  (ICLR)}, 2021.

\bibitem[Wu et~al.(2021)Wu, Guo, Su, and Weinberger]{wu2021online}
R.~Wu, C.~Guo, Y.~Su, and K.~Q. Weinberger.
\newblock Online adaptation to label distribution shift.
\newblock \emph{Advances in Neural Information Processing Systems},
  34:\penalty0 11340--11351, 2021.

\bibitem[Wu et~al.(2020)Wu, Guo, Chen, Liang, Jha, and
  Chalasani]{wu2020representation}
X.~Wu, Y.~Guo, J.~Chen, Y.~Liang, S.~Jha, and P.~Chalasani.
\newblock Representation {B}ayesian risk decompositions and multi-source domain
  adaptation.
\newblock \emph{arXiv preprint arXiv:2004.10390}, 2020.

\bibitem[Wu and Yang(2020)]{wu2020optimal}
Y.~Wu and P.~Yang.
\newblock Optimal estimation of {G}aussian mixtures via denoised method of
  moments.
\newblock \emph{The Annals of Statistics}, 48\penalty0 (4):\penalty0
  1981--2007, 2020.

\bibitem[Wu et~al.(2019)Wu, Winston, Kaushik, and Lipton]{wu2019domain}
Y.~Wu, E.~Winston, D.~Kaushik, and Z.~Lipton.
\newblock Domain adaptation with asymmetrically-relaxed distribution alignment.
\newblock In \emph{International Conference on Machine Learning (ICML)}, pages
  6872--6881. PMLR, 2019.

\bibitem[Xie et~al.(2020)Xie, Ye, Chen, Liu, Sun, and Li]{xie2020risk}
C.~Xie, H.~Ye, F.~Chen, Y.~Liu, R.~Sun, and Z.~Li.
\newblock Risk variance penalization.
\newblock \emph{arXiv preprint arXiv:2006.07544}, 2020.

\bibitem[Yao et~al.(2022)Yao, Wang, Li, Zhang, Liang, Zou, and
  Finn]{yao2022improving}
H.~Yao, Y.~Wang, S.~Li, L.~Zhang, W.~Liang, J.~Zou, and C.~Finn.
\newblock Improving out-of-distribution robustness via selective augmentation.
\newblock In \emph{International Conference on Machine Learning (ICML)}, pages
  25407--25437. PMLR, 2022.

\bibitem[Yu and Kumbier(2020)]{yu2020veridical}
B.~Yu and K.~Kumbier.
\newblock Veridical data science.
\newblock \emph{Proceedings of the National Academy of Sciences}, 117\penalty0
  (8):\penalty0 3920--3929, 2020.

\bibitem[Zech et~al.(2018)Zech, Badgeley, Liu, Costa, Titano, and
  Oermann]{zech2018variable}
J.~R. Zech, M.~A. Badgeley, M.~Liu, A.~B. Costa, J.~J. Titano, and E.~K.
  Oermann.
\newblock Variable generalization performance of a deep learning model to
  detect pneumonia in chest radiographs: A cross-sectional study.
\newblock \emph{PLOS Medicine}, 15\penalty0 (11):\penalty0 e1002683, 2018.

\bibitem[Zhang et~al.(2018)Zhang, Cisse, Dauphin, and
  Lopez-Paz]{zhang2017mixup}
H.~Zhang, M.~Cisse, Y.~N. Dauphin, and D.~Lopez-Paz.
\newblock Mixup: Beyond empirical risk minimization.
\newblock In \emph{International Conference on Learning Representations
  (ICLR)}, 2018.

\bibitem[Zhao et~al.(2019)Zhao, Des~Combes, Zhang, and
  Gordon]{zhao2019learning}
H.~Zhao, R.~T. Des~Combes, K.~Zhang, and G.~Gordon.
\newblock On learning invariant representations for domain adaptation.
\newblock In \emph{International Conference on Machine Learning (ICML)}, pages
  7523--7532. PMLR, 2019.

\end{thebibliography}


\end{document}